\documentclass[twoside,12pt]{article}

\RequirePackage{amsthm,amsmath,amsfonts,amssymb}
\RequirePackage[numbers]{natbib}
 
\usepackage{fullpage} 
\usepackage{tikz} 

\usepackage{cancel}
\usepackage{enumitem} 
\usepackage[ruled,vlined]{algorithm2e}
\usepackage{tikz,tkz-berge}
\usepackage{amsmath}
\usepackage{pgfplots}
\usepackage{subcaption}
\usepackage{comment}
\usepackage{float}
\usepackage{soul}
\newtheorem{theorem}{Theorem}
\newtheorem{assumption}[theorem]{Assumption}
\newtheorem{remark}[theorem]{Remark}
\newtheorem{lemma}[theorem]{Lemma}

\newtheorem{corollary}[theorem]{Corollary}
\newtheorem{proposition}[theorem]{Proposition}

\newcommand{\real}{{\mathbb{R}}}

\newcommand{\by}{{\mathbf{y}}}

\newcommand{\bg}{{\mathbf{g}}}
\newcommand{\bu}{\boldsymbol{u}}
\newcommand{\bW}{{\mathbf{W}}}
\newcommand{\bG}{{\mathbf{G}}}
\newcommand{\bX}{{\mathbf{X}}}

\newcommand{\rate}{\lambda}
\newcommand{\B}{B}
\newcommand{\btheta}{{\boldsymbol{\theta}}}
\newcommand{\bTheta}{{\boldsymbol{\Theta}}}

\newcommand{\htheta}{\widehat{\btheta}}
\newcommand{\hTheta}{{\widehat{\bTheta}}}

\newcommand{\cL}{{\mathcal{L}}}

\newcommand{\cG}{{\mathcal{G}}}
\newcommand{\cV}{{\mathcal{V}}}
\newcommand{\cE}{{\mathcal{E}}}
\newcommand{\cN}{{\mathcal{N}}}

\newcommand{\tO}{{\widetilde{O}}}

\newcommand{\argmin}{\mathop{\rm argmin}}

\newcommand{\norm}[1]{\left \| #1 \right\|}

\newcommand{\inmat}[2]{\langle \kern-0.3ex \langle #1, #2 \rangle \kern-0.3ex \rangle}
\newcommand{\vertiii}[1]{{\left\vert\kern-0.3ex\left\vert\kern-0.3ex\left\vert #1 
		\right\vert\kern-0.3ex\right\vert\kern-0.3ex\right\vert}}

\makeatletter
\newtheorem{rep@theorem}{\rep@title}
\newcommand{\newreptheorem}[2]{%
\newenvironment{rep#1}[1]{%
 \def\rep@title{#2 \ref{##1}}%
 \begin{rep@theorem}}%
 {\end{rep@theorem}}}
\makeatother

\newreptheorem{theorem}{Theorem}
\newreptheorem{corollary}{Corollary}
\newreptheorem{proposition}{Proposition}

\allowdisplaybreaks

\makeatletter
\renewcommand\appendix{\par
  \setcounter{section}{0}%
  \setcounter{subsection}{0}%
  \gdef\thesection{Appendix \@Alph\c@section}}
\newcommand\supplement{\par
  \setcounter{section}{0}%
  \setcounter{subsection}{0}%
  \gdef\thesection{S.\@arabic\c@section}}
\makeatother

\begin{document}

\title{\bf   Decentralized Sparse Linear Regression via Gradient-Tracking\vspace{1cm}}

\author{  Marie Maros$^{1}$, Gesualdo Scutari$^{2}$, Ying Sun$^{3}$, Guang Cheng$^{4}$    \bigskip  \\
	  $^{1}$ Texas A\&M University, {\it mmaros@tamu.edu}\smallskip\\   $^{2}$ Purdue University, {\it   gscutari@purdue.edu}\smallskip\\
	  $^{3}$ The Pennsylvania State University, {\it ysun@psu.edu}\smallskip \\
  $^{4}$ The University of California, Los Angeles, \it{guangcheng@ucla.edu.} \bigskip \\
        The order of the first three authors is alphabetic.}
\date{First submission: Jan. 2022. Final revised version: Dec. 2024.\footnote{This work was conducted when Sun, Maros, and Cheng were at Purdue University. \newline The work of the first three authors have been supported by   the Office of Naval Research (ONR Grant N. N000142112673 and ONR Grant N. N000142412751).}
}

\maketitle

\begin{abstract}  
We study   sparse linear regression over a network of agents, modeled as an undirected  graph without a center node. The estimation of the $s$-sparse parameter is  formulated as a constrained LASSO problem wherein each agent   owns a subset of the $N$ total   observations. We analyze the convergence rate and statistical guarantees of a distributed projected gradient tracking-based algorithm under high-dimensional scaling, allowing the ambient  dimension $d$ to grow with (and possibly exceed) the sample size $N$. Our theory shows that, under standard   notions of restricted strong convexity and smoothness of the average loss functions,  suitable conditions on the network connectivity  and algorithm tuning, the distributed algorithm converges   globally at a {\it  linear} rate to an estimate that is within the centralized {\it statistical precision} of the model, $O(s\log d/N)$. When $s\log d/N=o(1)$, a condition necessary for statistical consistency, an $\varepsilon$-optimal solution is attained after  $\mathcal{O}(\kappa \log (1/\varepsilon))$ gradient computations  and $O (\kappa/(1-\rho) \log (1/\varepsilon))$  communication rounds,
where $\kappa$ is the restricted condition number of the loss function and $\rho$ measures the network connectivity. 
The computation cost matches that of  the centralized projected gradient algorithm despite  having data distributed; whereas the communication rounds reduce as the network connectivity improves.
Overall, our study   reveals  interesting connections between statistical efficiency, network connectivity \& topology, and  convergence rate in  the high dimensional setting. 
\end{abstract}

\bigskip   

\section{Introduction}\label{sec:Intro} Datasets with massive sample size and high-dimensionality are ubiquitous in modern science and engineering; examples include genomics and  biomedical data, social media, financial time-series, and e-commerce data, just to name a few. The  archetypal scenario is the one wherein data  are  produced, collected, or stored at different times and locations. 
The sheer volume and   spatial/temporal disparity of  data render centralized processing and storage on standalone machines prohibitively expensive or outright impossible.  
This has motivated in recent years a surge of interest in developing distributed methods that enable analytics and computations across multiple machines (henceforth referred to   as agents), connected by a communication network.  {}{Broadly speaking,   two architectures  have captured most of the interest of the research  community, namely: (i)  the {\it master-worker} topology    and (ii) a general, connected, multiagent architecture (a.k.a. {\it mesh network}). They are depicted in Fig.~\ref{fig1:topology} and briefly commented next.} \\\indent  \begin{figure}
    \centering
    \includegraphics[scale=0.4]{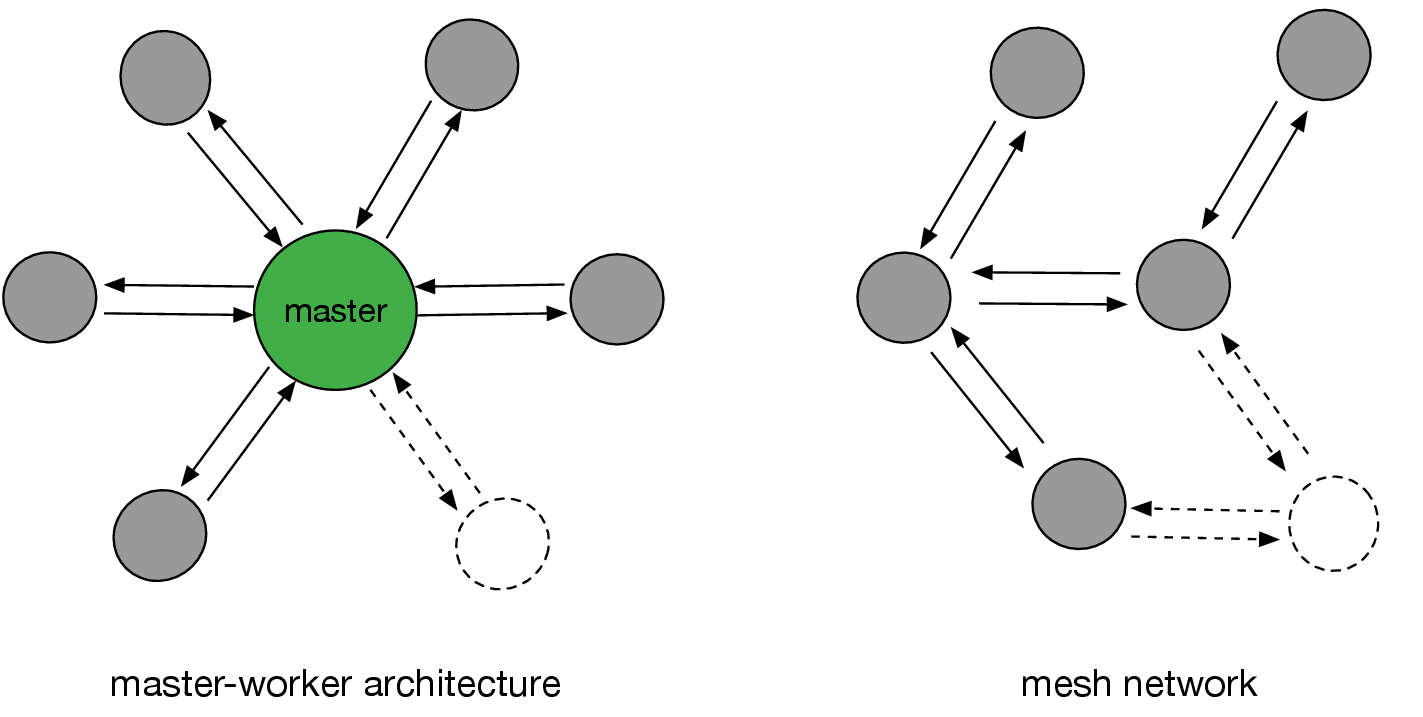}
    \caption{{}{Two network architectures: master-worker    and mesh networks.}}
    \label{fig1:topology}\vspace{-0.4cm}
\end{figure} 
{}{In a {\it master-worker}  architecture, there is typically one or more master nodes connected to all other nodes, termed workers. The workers store part of the data, execute  (in parallel) intermediate   computations, and communicate (iteratively) suitable outcomes to the master node, which maintains and updates   the authoritative copy of the optimization variables, producing eventually the final  output. These federated architectures have been adopted in several applications to parallelize and decompose a variety of learning and optimization tasks; see, e.g.,  \citep{smith_fedLen_SPMag20, pmlr-v54-mcmahan17a, boyd2011distributed}. The statistical community is best acquainted with   
{\it divide-and-conquer} (D\&C) methods whereby the master node combines  the  estimators produced by the workers using their local datasets. The idea has been widely applied in statistical estimation and inference  such as M-estimation~\citep{zhang2013communication,chen2014split,shi2018massive}, sparse high dimensional models~\cite{battey2018distributed,lee2017communication},  PCA~\cite{fan2019distributed}, matrix factorization~\cite{mackey2011divide}, quantile regression~\cite{volgushev2019distributed,chen2019quantile} and non- and semi-parametric methods including~\citep{neiswanger2014asymptotically,shang2017computational,zhao2016partially,zhang2013divide,kleiner2014scalable}, just to name a few. }\\\indent
{}{However, such architectures may be impractical or undesirable in various scenarios. For example, in large-scale systems, the master node can become a systemic bottleneck: (i) its failure can disrupt the entire network, (ii) the node's communication resources may become insufficient as the network size grows   (owing to restricted bandwidth), and (iii) the implementation of some algorithms may require  prohibitive computation  efforts on standalone machines \citep{hendrikx2020}, particularly when processing high-dimensional and large datasets.
 Furthermore, there are scenarios where establishing a star topology (or a spanning tree) is simply not feasible or too resource-expensive; this is the case, e.g.,   of   wireless networks with low-power devices, which can   communicate only with  nodes in their physical proximity. 
To circumventing these impediments, a natural approach is to eliminate master nodes, transitioning to   peer-to-peer architectures   wherein each node is connected exclusively to a subset of other nodes. These are the systems considered in this work. } 

 {
\subsection{Notation}
Throughout the paper vectors  are boldfaced and matrices are boldfaced and capitalized. We denote by $\mathbf{1}$ the vector of all ones, and by $\mathbf{J}=\mathbf{1}\mathbf{1}^\top$ the projection onto the consensus subspace; dimensions of these quantities will be clear from the context. Let $\|\cdot\|_{p}$ be denote$\ell_p$ norm, for $p \geq 1.$ When $p$ is not specified. With a slight abuse of notation, we also denote by $\|\cdot\|_p$ the operator norm when applied to matrices,  induced by the vector  $\ell_p$ norm;   $\|\cdot\|_F$ denotes the Frobenius norm.}

{}{\subsection{Sparse linear regression over mesh networks}}
We study  sparse linear regression problem over a mesh network of $m$  agents, modeled as a general connected, undirected graph. Each agent $i$  takes $n$ linear measures of an $s$-sparse signal $\btheta^* \in \real^d$:  \vspace{-0.2cm}
\begin{equation}\label{eq:I/O}
	\by_i = \bX_i \btheta^* + \mathbf{n}_i,
\end{equation}
where $\by_i\in \mathbb{R}^n$ is the vector of measurements,   $\bX_i \in \real^{n \times d}$ is the  design matrix, and  $\mathbf{n}_i \in \real^n$ is the observation noise. For simplicity we assume that all agents acquire the same number $n$ of measurements, accumulating to a total of   $N=m\cdot n$ over the network. Our focus  is on  the  high-dimensional setting where both  the ambient dimension $d$ and  the   sample size $N$ grow, with $d$ faster than $N$.

A standard approach to estimate $\btheta^*$ from   $\{(\by_i,\bX_i)\}_{i=1}^m$ is solving the LASSO problem, whose constrained form reads\vspace{-0.2cm} 
\begin{equation}\label{p:regularized_ERM_constraint}
	\widehat \btheta \in \argmin_{\| \btheta\|_1 \leq r} \, \left\{ \cL(\btheta) \triangleq \frac{1}{m} \sum_{i=1}^{m} 	\cL_i(\btheta)\right\}, \qquad \text{where} \qquad \cL_i(\btheta) \triangleq \frac{1}{2n} \| \by_i - \bX_i \btheta\|^2.
\end{equation}
In~\eqref{p:regularized_ERM_constraint}, each $\cL_i$ is the least squares loss, local to agent $i$, and the $\ell_1$-norm constraint (with $r$ known to all the agents) aims at promoting sparsity on the solution  $\widehat \btheta$.

As an instance of the  empirical risk minimization problem with nonsmooth regularization, there is a large body of distributed first order algorithms applicable for~\eqref{p:regularized_ERM_constraint}. Among these methods, the gradient tracking based ones have been demonstrated successful in improving the algorithmic efficiency both empirically and theoretically~\citep{Xu-TAC:hs,NEXT16,nedich2016achieving}. As such, this technique has been widely applied to design distributed algorithms in various problem setups. 
Based on decentralizing the proximal gradient algorithm (PGD), the DGT algorithm for problem~\eqref{p:regularized_ERM_constraint}  decouples the optimization by letting each agent $i$ locally maintain an estimator  $\btheta_i \in \mathbb{R}^d$ of the common variable $\btheta$. At every iteration $t$, each agent $i\in [m]$ performs in parallel the following update:   \begin{subequations}\label{alg:DGT}
 \begin{align}
        \text{comm. step:} \quad & \btheta_i^{t}   = \sum_{j =1}^m w_{ij} \,\btheta_j^{t-\frac{1}{2}}, \quad
		\bg_i^{t}  =  \sum_{j = 1}^m w_{ij} \,\Big(\bg_j^{t-1} + \nabla \cL_j (\btheta_j^t) -  \nabla \cL_j (\btheta_j^{t-1}) \Big), \label{eq:tracking}\\
    \text{comp. step:} \quad & \btheta_i^{t + \frac{1}{2}}  = \prod_{\|\btheta_i\|_1 \leq r}~ \left( \btheta_i^t - \gamma^{-1} \bg_i^t\right).\label{eq:loc_opt}
 \end{align}
 \end{subequations}
 In the computation step,   each agent performs a proximal (projected) gradient update, where $\bg_i^t$ plays the role of estimating the centralized gradient $\nabla \cL (\btheta_i^t)$; $\prod_{\|\btheta_i\|_1 \leq r}$ denots the Euclidean projection onto the $\ell_1$ ball $\{\btheta_i\in \mathbb{R}^d\,:\,\|\btheta_i\|_1 \leq r\}$. 
 In  the communication step,  each agent updates  the local variable  $\btheta_i^{t-\frac{1}{2}}$  and gradient estimator 
 $\bg_i^{t-1}$ based on local averaging. The update of $\bg_i^{t}$ (gradient tracking step)  can be  explained as follows:  
At each iteration $t$,  agent $i$   subtracts the outdated gradient $\nabla \cL_i (\btheta_i^{t-1})$ at the previous iteration and adds the new one $\nabla \cL_i (\btheta_i^{t})$;  this refreshes the ``memory'' of $\bg_i$ and ensures that  the  sum of the $\bg_i^t$'s is equal to the sum gradient $\sum_{i = 1}^{m} \nabla \cL_i (\btheta_i^t)$. Then, agents perform a local averaging  by computing a weighted sum using $w_{ij}$, enforcing thus consensus among the $\bg_i$'s. If both $\btheta_i$'s and $\bg_i$'s are asymptotically consensual, then\vspace{-0.2cm}
\begin{align}
	\bg_i^t \to \frac{1}{m} \sum_{j = 1}^{m} \bg_j^t = \frac{1}{m}\sum_{j=1}^m \nabla\cL_j(\btheta_j^t) \to \frac{1}{m}\sum_{j=1}^m \nabla\cL_j(\btheta_i^t) \quad \text{as} \quad  t \to \infty.
\end{align}
That is, $\bg_i^t$ converges to $\nabla \cL (\btheta_i^t)$, as desired.

{}{Despite the vast literature on decentralized algorithms, 
 existing results of DGT are of pure optimization type;  hence they  lack   guarantees in the high-dimensional setting, even for the sparse  linear regression problem \eqref{p:regularized_ERM_constraint}. In contrast, the PGD algorithm as its centralized counterpart has been proven to converge at a fast linear rate to  statistically optimal solutions \citep{agarwal2012fast}. For a detailed review of the literature, we refer readers to Section~\ref{sec_related works}. Given this disparity, it naturally raises the question of whether DGT  can enjoy the same fast rate and statistical optimality of PGD in the decentralized setting, and how the network configuration affects its performance. This paper addresses this open question. }\medskip 

 \subsection{Major contributions}\label{sec_major_contributions}
%
We provide the  statistical and convergence analysis of the DGT algorithm for solving the sparse linear regression problem~\eqref{p:regularized_ERM_constraint}.  Our major contributions are the  following. 

\begin{description}
\item[\bf (i)] {\bf  Statistical-computational guarantees:}  
By exploiting the  Restricted Strong Convexity (RSC) and Smoothness (RSM) properties of $\cL$~\citep{agarwal2012fast},  we establish regularity conditions under which the sequence generated by DGT provably convergences at a \emph{linear rate} to an estimate  whose distance to $\htheta$ is \emph{smaller than the  statistical precision $\| \htheta - \btheta^*\|$}.     Specifically, to enter an $\varepsilon$-neighborhood of such precision,  it takes \begin{equation}\label{eq:it_complexity}
 	O\left( \kappa  \log (1/\varepsilon)\right)
 \end{equation}
iterations (gradient evaluations) as long as  $\rho \leq \texttt{poly}(m,\kappa)^{-1}$, where $\kappa$ is the restricted condition number of $\cL$ and $\rho$ measures the network connectivity.
The condition on $\rho$, when not met by the given network, can be enforced via   multiple rounds of communication per  iteration, resulting in   
 \begin{equation}\label{eq:comm_complexity}
O\left(\frac{\kappa}{1 - \rho}  \log(m\kappa) \log (1/\varepsilon)\right) \end{equation} overall   communication rounds.  
As an implication, our result  shows statistically optimal estimators can be computed at a linear rate by the DGT algorithm in the limit $N,d \to \infty$ and $s \log d/N = o(1)$.  Furthermore,    \eqref{eq:it_complexity}  matches   the iteration complexity    of the {\it centralized} proximal gradient descent (PGD)~\citep{agarwal2012fast}.  

\item[\bf (ii)] {\bf  Sample/network scaling: } The communication complexity (\ref{eq:comm_complexity}) remains unchanged, for    $s\log d/N$ being  constant and for a fixed network.   
We also study the scalability of (\ref{eq:comm_complexity}) with respect to the network size $m$,  revealing tradeoffs between communication rounds and   communication costs for different graph topologies, as summarized in Table~\ref{tab:comm-cost}.  This sheds light on which network architectures are more favorable for high-dimensional estimation. 
For instance, among the commonly used network architectures,   DGT running on  the Erd\H{o}s-R\'{e}nyi graph with edge connecting probability $p = \log m/m$ achieves the lowest communication cost at the busiest node, and ties with  the implementation on the star topology  in the amount of total channel use.
{ 
\item[\bf(iii)]{\bf Technical novelties:} Our  analysis represents a shift from both centralized approaches in high-dimension (e.g., \citep{agarwal2012fast}) and the     prevailing techniques used for studying gradient-tracking algorithms in decentralized optimization  (e.g., \citep{nedich2016achieving,sun2019distributed,xu2021distributed}), across several  aspects. }

{  {\bf  (a)} {\it Inexact descent under RSC/RSM:} Linear convergence of gradient-tracking methods has   been proved  under  the assumption of {\it global } strong convexity of $\cL$ and global smoothness  of $\cL_i$'s--conditions that fall short in the high-dimensional setting where $d>>N$. In contrast, our analysis employs the more  nuanced  RSC and  RSM conditions--better suited for high-dimensional data and known to hold with high probability across various data generation models \citep{Wainwright_2019}.  Our first result is an inexact descent property of the estimation error (detailed in Proposition~\ref{eq:inexact_opt}), achieving linear shrinkage subject to consensus and gradient-tracking errors.  This can be view as a generalization of  \citep{agarwal2012fast}  to more complex scenarios posed by the mentioned errors (not present in   \citep{agarwal2012fast}) and decentralized contexts.  }

{   \textbf{(b)} {\it A new analysis of gradient tracking errors in high-dimensions:} Traditional  analyses of gradient tracking errors leverage the Lipschitz continuity of each $\nabla \cL_i$ to establish {\it uniform} bounds 
 on the  tracking errors   $\mathbf{g}_i^t-\nabla \cL(\btheta_i^t)$, in terms of consensus and optimization errors (see Sec.~\ref{sec:track_err}). However, this approach proves inadequate in high-dimensional settings, where in the    asymptotics $d/N\to \infty$,    the Lipschitz constants  of $\nabla \cL_i$'s  scale as $\mathcal{O}(d)$, for various predictor models \citep{Wainwright_2019},  thereby diverging as $d\uparrow \infty$.  Fundamentally, this issue arises because the gradients   $\nabla \cL_i(\btheta_i)$  remain dense vectors even when all $\btheta_i$'s are   sparse. This would translate in convergence rates scaling unfavorably with the ambient dimension.  
   Our new analysis posits that, under the RSC and RSM conditions,  it is sufficient to manage the tracking errors only along {\it sparse directions}. {\it Within these directions},   $\mathbf{g}_i^t$ is proved to serves as a sufficiently accurate estimate   of $\nabla \cL(\btheta_i^t)$ (see Proposition~\ref{prop:track_error}). This crucial adjustment enables us to achieve convergence rates for DGT that are   {\it independent of  the ambient dimensions}. }

{   \textbf{(c)} {\it Combining error dynamics via the $z$-transform:} The refined treatment of gradient-tracking errors, as discussed in (b), significantly complicates the convergence analysis of DGT. The error dynamics at any given iteration now depend on the consensus/tracking errors and optimization residuals from {\it all past iterates}, starting from initialization. This extended dependency chain precludes the use of simpler, traditional  analytical techniques,  such as those in \citep{Xu-TAC:hs,sun2019distributed,nedich2016achieving} based on the small gain theorem,   typically applied only to dynamics with {\it single-hop} time dependencies.  To effectively manage this historical information, we employ   the  
   $z$-transform for finite-length sequences  and   obtain a sufficient condition for the linear convergence of DGT.  This approach represents a technique of independent interest, offering potential applications to other problem  of decentralized optimization. } 

\end{description}
 
\begin{figure}[H]
 \SetVertexNoLabel
\setlength{\tabcolsep}{5ex}
{\Huge
 \resizebox{\columnwidth}{!}{
\begin{tabular}{ccccc}
\begin{tikzpicture}
		\grPath[prefix=h,RA=2,RS=2]{4} 
\end{tikzpicture} &
	\begin{tikzpicture}
	\begin{scope}
		\grLadder[RA=2,RS=2,prefix = g, prefixx = h]{4}%
	\end{scope}
	\begin{scope}[yshift = -4cm]
		\grLadder[RA=2,RS=2]{4}%
	\end{scope}
	\EdgeMod{b}{g}{4}{0}
\end{tikzpicture} &
\begin{tikzpicture}
	\grStar[RA=2.5]{8}%
\end{tikzpicture} & 
\begin{tikzpicture}
	\grComplete[RA=2/sin(60)]{8}
\end{tikzpicture} &
\begin{tikzpicture}
	\begin{scope}[xshift=12cm]
		\grEmptyCycle[prefix=a,RA=2/sin(60)]{8}
	\end{scope}
	\Edges[style={dashed,lightgray}](a0,a1,a4,a7,a3,a2,a6,a5)
	\Edges[style={dashed,lightgray}](a7,a2,a5)
\end{tikzpicture} \\[5ex]
line graph & 2-d grid & star & complete graph & Erd\H{o}s-R\'{e}nyi 
\end{tabular}
}
}
\caption{Some commonly used graphs $\cG$ for the communication network. Dashed line in the Erd\H{o}s-R\'{e}nyi graph stands for the existence of an edge with some probability. }\vspace{-0.2cm}\label{fig:graph-top}
\end{figure}
 \begin{table}[h!]
 \resizebox{\columnwidth}{!}{
\begin{tabular}{cccccccc}
                   & line  & $2$-d grid    & star  & complete  & push-pull (star)& Erd\H{o}s-R\'{e}nyi & Erd\H{o}s-R\'{e}nyi \\[1ex]\hline\\[-2ex]
$(1 - \rho)^{-1}$    & $O(m^2)$   & $O(m \log m)$ & $O(m^2)$   & $1$   & $1$         & $O(1)$  [$p = \log m/m$]                           & $O(1)$  [$p = O(1)$]                         \\
max degree & $2$        & $4$           & $m$        & $m$   & $m$         & $O(\log m)$                                            & $O(m)$                                                \\
channel use/comm.  & $O(m)$     & $O(m)$        & $O(m)$     & $O(m^2)$  & $O(m)$     & $O(m \log m)$                                         & $O(m^2)$                                              \\
max channel use    & $\tO(m^2)$ & $\tO(m)$      & $\tO(m^3)$ & $\tO(m)$     & $\tO(m)$ & $\boldsymbol{\tO(1)}$                                            & $\tO(m)$                                        \\
total channel use  & $\tO(m^3)$ & $\tO(m^2)$    & $\tO(m^3)$ & $\tO(m^2)$     & $\boldsymbol{\tO(m)}$& $\boldsymbol{\tO(m)}$                                              & $\tO(m^{2})$  \\[1ex]\hline     \vspace{0.1cm}                           
\end{tabular}
}
\caption{Scalability of the communication complexity of DGT with the network size $m$, for  the  network topology listed in Fig.~\ref{fig:graph-top}.  {\bf Max degree}:   maximum degree of the nodes; {\bf channel use/comm.}: total channel use per communication round; {\bf max channel use}:  channel use of the busiest node to reach $\varepsilon$ optimization precision; {\bf total channel use}: total  channel uses to reach $\varepsilon$ optimization precision. $\tO$ hides the logarithmic factors in  communication complexity. Push-pull(star) refers to the centralized PGD algorithm implemented   over star-networks, we refer the readers to the description around~\eqref{alg:PGD} for more details.
}\label{tab:comm-cost}\vspace{-0.2cm}
\end{table}

\subsection{Related works}~\label{sec_related works}
\noindent{}{\noindent\textbf{Distributed methods over master-worker systems.} The last decade has witnessed a significant surge in research on statistical methods within master-worker architectures, as highlighted in Sec.~1. Much of this research has focused on D\&C  strategies \citep{zhang2013communication,chen2014split,shi2018massive,battey2018distributed,lee2017communication,fan2019distributed,mackey2011divide,volgushev2019distributed,chen2019quantile,neiswanger2014asymptotically,shang2017computational,zhao2016partially,zhang2013divide,kleiner2014scalable}. These methods typically involve a single round of communication from clients to the server and achieve statistical optimality, but {\it only with a limited number of client nodes}, restricting their use to small networks.  
  To address these limitations, iterative approaches have been developed that allow for larger networks by removing the cap on the number of workers; examples include  \citep{jordan2018communication,CEASE-JASA,wang2017efficient}. At the cost of multiple communication rounds, they effectively remove restrictions on the number of workers,  at the expense of requiring a {\it minimum sample size} at each client to ensure the statistical optimality of the final estimator at the server.}\\
 \indent  {}{All these algorithms heavily rely on the presence of   a server node that, at each iteration,  aggregates local gradients from clients into a centralized gradient and distributes updates. This setup ensures that the iterates are generated using {\it exact} gradient information, mimicking a fully centralized setting. Without such a centralized node, this mechanism vanishes, rendering these algorithms inapplicable on mesh networks. Moreover, simply replacing the exact gradient with approximations obtained by averaging neighboring local gradients is inadequate. This approach fails to achieve statistically optimal estimates or guarantee convergence, as it does not align with the assumptions and analytical frameworks of centralized methods, which do not account for gradient inaccuracies.}
  \\[1ex]
\noindent\textbf{Distributed methods over  mesh networks.} {}{Early decentralized methods over mesh networks are  distributed (sub-)gradient descent (DGD) type methods including~\citep{Nedic2009,Nedic2010,chen2012fast,Chen-Sayed,nedic2014distributed}. These algorithms are not directly applicable to (\ref{p:regularized_ERM_constraint}) because they cannot handle constraints. Subsequent   schemes handling constraints include  \citep{Sayed,shi2015proximal,NEXT16,sun2019distributed,ScutariSunMathProg,xu2021distributed}. }

Based on the order of local computation and communication, the DGD algorithm has the following two variations: 
\begin{align}
    \btheta_i^{t+1} & = \prod_{\|\btheta_i\|_1 \leq r}~ \left( \sum_{j = 1}^m w_{ij}\btheta_j^t - \gamma^{-1} \nabla \cL_i (\btheta_i^t) \right), \tag*{\text{DGD-CTA}}\\
    \btheta_i^{t+1} & = \prod_{\|\btheta_i\|_1 \leq r}~ \left( \sum_{j = 1}^m w_{ij} \Big(\btheta_j^t - \gamma^{-1} \nabla \cL_j (\btheta_j^t) \Big) \right). \tag*{\text{DGD-ATC}}
\end{align}
In DGD-CTA,   each agent  at iteration $t$ first averages the neighboring variables through a consensus step and then performs a gradient step using its local gradient. Conversely, in the  DGD-ATC, each agent first executes the local gradient descent step before carrying out the consensus averaging. It is well established that DGD algorithms, even when applied to smooth and strongly convex loss functions  $\cL$,   converge at a sublinear rate. To overcome this drawback, the  DGT [cf. eq.~\eqref{alg:DGT}]~\citep{Xu-TAC:hs,NEXT16,sun2019distributed,nedich2016achieving}  and various primal-dual algorithms~\citep{shi2015proximal,yuan2018exact,xu2021distributed} utilize an improved estimate of the centralized gradient at each agent's location. Consequently, these algorithms achieve a linear convergence rate for loss functions 
$\cL$ that are both strongly convex and smooth. For a detailed overview of these methods, please see   ~\citep{nedic10distributed_tutorial,Nedic_Olshevsky_Rabbat2018}.

{}{Despite intensive study, existing convergence results for these algorithms, derived purely from an optimization perspective, lead to pessimistic conclusions when applied to the high-dimensional problem  \eqref{p:regularized_ERM_constraint}. Specifically, since 
$N<d$, the average loss function 
$\cL$
  lacks strong convexity, resulting in sublinear convergence rates for any {\it fixed} ambient dimension $d$.  Moreover, when $d$ increases at a rate faster than $N$, which is the typical regime in high-dimension \citep{Wainwright_2019}, 
  the Lipschitz constant of 
$\nabla \cL$  will  diverge as well, under standard choices of the  design matrices $\{\bX\}_{i=1}^m$. This requires progressively smaller step sizes, ultimately hindering the convergence of the algorithms.} \\[1ex]
\noindent\textbf{First order algorithms  in the high-dimensional setting.} {     It is well-recognized that  convergence analysis and parameter tuning  of {\it centralized} first-order methods for (\ref{p:regularized_ERM_constraint})  must   exploit the landscape properties  of $\cL$ as determined by the underlying  statistical model. Notably, linear  convergence to   statistical optimal estimates  is established in~\citep{agarwal2012fast,wang2014optimal,lin2014adaptive,xiao2013proximal} for the proximal gradient algorithm applied to~\eqref{p:regularized_ERM_constraint}, under the  RSC and RSM properties of $\cL$, which hold with high probability across various random data generation  models and noise distribution \citep{Wainwright_2019}. These conditions, coupled with a proper algorithm tuning,  ensure that the algorithmic trajectory remains within a confined region where   $\cL$ is approximately  strongly convex and smooth, thereby securing linear convergence up to statistical optimality.} {  An intriguing question is whether {\it distributed} first-order algorithms can similarly leverage this landscape to achieve rapid convergence toward statistically optimal estimates, particularly when local agents lack sufficient data samples to ensure statistical consistency locally. }

{  In contrast to their centralized counterparts, the performance of distributed algorithms in high-dimension is less clear. While research is still evolving, existing studies (concurrent to this work \citep{NetLASSOarXiv}) suggest that not all categories of distributed algorithms consistently benefit from these favorable landscape properties, indicating a more complex interaction among algorithmic design, distributed nature of data, and network topology, to be revealed. }   For instance,  the analysis of a DGD-CTA type algorithm for the Lagrangian formulation of the LASSO problem,  developed in the companion paper~\citep{Ji-DGD21}, proves that the iterates generated by the algorithm  enter an $\varepsilon$-neighborhood of a statistically optimal estimate of $\btheta^\ast$ after  \begin{equation}\label{eq:DGD-comm-complexity}
\mathcal{O} \left( \frac{\kappa}{1-\rho} \cdot d\cdot m \log m \cdot \log \frac{1}{\varepsilon} \right)\end{equation} number of  iterations, where $\rho\in [0,1)$ is a measure of the connectivity of the network (the smaller $\rho$, the more connected the graph, see Assumption~\ref{assump:weight} in Sec.~\ref{sec:Main-results} for a formal definition); and  $\kappa$  is  the restricted condition number of  $\cL$.  This result shows that DGD-CTA converges at a linear rate, but  the number of iteration towards statistically  optimal estimates grows as $\mathcal{O}(d)$, even in the asymptotic regime   $N,d\uparrow \infty$ where  the statistical error $\mathcal O((s\log d)/N)$ remains unchanged. In contrast,   the DGD-ATC algorithm achieve statistical optimal estimates at a converges at rate ~\citep{ji2023distributed} 
\begin{equation}\label{eq:DGD-ATC}
    \widetilde{\mathcal{O}} \left( \frac{\kappa }{1 - \rho} \cdot \log (dm) \cdot \log \frac{1}{\varepsilon}\right),
\end{equation}
which has the more favorable $\mathcal{O}(\log d)$ scalability with $d$,  compared to the CTA version.

The key difference between the two DGD variants lies in  their local optimization approaches.  In the CTA version, each agent uses its local gradient evaluated at a point that is \emph{inconsistent} with the iterate obtained after the communication step, 
necessitating a step size of  $\mathcal{O}(d^{-1})$ regardless of   network connectivity. Conversely, by swapping the order of the communication and optimization steps, the ATC version allows for a {\it constant} step size in well-connected networks, enhancing scalability with the ambient dimension. Building on the insights  from these analyses,  in a paper \citep{maros2022dgd} subsequent to this work, we introduce a modified version of DGD  (termed DGD$^2$) that  achieves statistical optimality at comparable communication and computational cost of DGT.  

\subsection{Paper organization}
The remainder of the paper is structured as follows: Sec.~\ref{sec:Main-results} outlines the primary assumptions and details the statistical-computational guarantees--the proofs  are reported  in Sec.~\ref{sec:convergence}, with  intermediate results available in the appendix.   Numerical simulations corroborating our theoretical findings are discussed in Sec.~\ref{sec:numerical_result}.  Finally, Sec.~\ref{sec:conclusions} 
 draws some conclusions.   \vspace{-0.2cm}

\section{Convergence Analysis}\label{sec:Main-results}

We start by introducing the assumptions related to the LASSO problem~\eqref{p:regularized_ERM_constraint} and  the network topology, used  for analyzing the convergence of DGT, given in~\eqref{alg:DGT}.

\subsection{Problem Setup and Main Assumptions}
\begin{assumption}[Global RSC/RSM]\label{assump:G_RSM}
Given $\bX = [\bX_1^\top, \ldots, \bX_m^\top]^\top \in \mathbb{R}^{N \times d}$, with $N = n \cdot m$,  the following conditions   hold:\vspace{-0.2cm}
\begin{align}
	\text{(Global RSC/RSM) }\qquad
\begin{split}\label{eq:G_RE}
	\frac{\| \bX \bu\|^2}{N} & \geq \mu_{\Sigma} \| \bu\|^2 - \tau_\mu \|\bu\|_1^2,\\
	\frac{\| \bX \bu\|^2}{N} & \leq L_\Sigma \| \bu\|^2 + \tau_g \|\bu\|_1^2, 
\end{split}  \quad \forall \ \bu \in \real^d,
\end{align}
where $(\mu_{\Sigma}, \tau_\mu)$ and $(L_\Sigma, \tau_g)$ are positive RSC/RSM parameters.
\end{assumption}

{  Assumption~\ref{assump:G_RSM} is a standard requirement for ensuring linear convergence of the  PGD algorithm \citep{agarwal2012fast} to statistically optimal estimates, even in centralized settings.} Additionally, we  impose the following local restricted smoothness condition on the agents' losses, which is instrumental to ensure that the gradient tracking  procedure estimates the centralized gradient accurately.

\begin{assumption}[Local RSM]\label{assump:L_RSS}
For each local design matrix $\bX_i$, the following   holds:
\begin{align}
	\frac{\| \bX_i \bu\|^2}{n} & \leq \ell_\Sigma \| \bu\|^2 + \tau_\ell \|\bu\|_1^2, \quad \forall \ \bu \in \real^d,
\end{align}
where $(\ell_{\Sigma}, \tau_\ell)$ are positive RSM parameters.
\end{assumption}


The mesh network is modeled as an undirected graph $\cG \triangleq (\cV,\cE)$, with   $\cV=\{1,\ldots, m\}$ being the set of $m$  nodes and  $\cE\subseteq \cV\times \cV$ being the set of edges corresponding to the communication links:   $(i,j) \in \cE$ iff there exists a communication link between agent $i$ and $j$. We denote by  $\cN_i=\{j\in \cV\,:\, (i,j)\in \cE\}\cup \{i\}$ the set of  neighbors of agent $i$ (including the agent itself). We make the following standard Assumption~\ref{assump:net1} on the graph connectivity, which is necessary to achieve consensus over the network.    

\begin{assumption}[Connected network]\label{assump:net1}
The graph $\mathcal G$ is connected.
\end{assumption}

To achieve consensus and ensure implementation over the communication network  $\mathcal{G}$,   standard conditions in the literature require the following constraints on the weights   $w_{ij}$ used in the DGT algorithm~\eqref{alg:DGT}, where   $\mathcal{P}_K$  denotes the set of polynomials with degree no larger than $K$. 
 
\begin{assumption}\label{assump:weight} [On the gossip matrices]  The matrix $\mathbf{W}=(w_{ij})_{i,j=1}^m$  satisfies the following:  

\noindent\textbf{(a)}   $\mathbf{W}=P^K(\overline{\mathbf{W}})$, where $P_K\in \mathcal{P}_K$ with $P_K(1)=1$, and $\overline{\mathbf{W}}\triangleq \big(\bar{w}_{ij}\big)_{i,j=1}^m$    
 has a sparsity pattern compliant with $\mathcal{G}$, that is \vspace{-0.2cm}
	\begin{enumerate}[label=\roman*)]
		\item $\bar{w}_{ii} > 0$, for all $i = 1, \ldots, m$;\vspace{-0.1cm}
		\item $\bar{w}_{ij} > 0$, if $(i,j) \in \mathcal{E}$; and $\bar{w}_{ij}=0$ otherwise.\vspace{-0.2cm}
	\end{enumerate}
	Furthermore, $\overline{\mathbf{W}}$ is doubly stochastic, that is, $\mathbf{1}^\top \overline{\mathbf{W}} = \mathbf{1}^\top$ and $\overline{\mathbf{W}}  \mathbf{1} = \mathbf{1}$. \smallskip 
	
	\noindent\textbf{(b)} Let $\rho\triangleq \|\mathbf{W}-\mathbf{J}\|_2$; 
    then, it holds   $\|(\bW - \mathbf{J})^{t}\|_\infty \leq c_m \rho^t,$ for all $t=1,2,\ldots,$ and  some $c_m \geq 1$. 
\end{assumption}

 \begin{remark}
 	 Assumption~\ref{assump:net1} and Assumption~\ref{assump:weight}(a) imply $\rho<1$ and $\|(\mathbf{W}-\mathbf{J})^t\|_2=\rho^t$. Using the norm bound  $\|(\bW - \mathbf{J})^{t}\|_\infty \leq \sqrt{m}\,\|(\mathbf{W}-\mathbf{J})^t\|_2$, we infer that  Assumption~\ref{assump:weight}(b) holds with   $c_m\leq \sqrt{m}$. 
 \end{remark}

 Assumption~\ref{assump:weight}(a) requires   $\overline{\bf W}$ being a consensus-forcing matrix with sparsity pattern matching the topolgy of communication network graph $\mathcal{G}$. When implementing the DGT algorithm~\eqref{alg:DGT} with $w_{ij} = \overline{w}_{ij}$, step~\eqref{eq:tracking} can be executed via one-hop communication with each agent's immediate neighbors. 
Several rules of choosing $\overline{\mathbf{W}}$ have been proposed in the literature satisfying Assumption \ref{assump:weight}, such as  the Laplacian,  the Metropolis-Hasting, and the maximum-degree weights rules; see, e.g.,     \citep{Nedic_Olshevsky_Rabbat2018}  and references therein. 
When the connectivity of $\mathcal{G}$ is relatively weak, it may be beneficial to run multiple communication rounds to accelerate information propagation.  This is realized by letting ${\bW} = P_K (\overline{\bW})$.
When $K>1$,   $K$  rounds  of communications per iteration $t$ are employed. 
For instance, by letting $\bW = \overline{\bW}^K$, the connectivity parameter can be improved from $\bar{\rho} = \|\overline{\mathbf{W}}-\mathbf{J}\|_2$ to $\|\overline{\mathbf{W}}^K-\mathbf{J}\|_2 = \bar{\rho}^K$.
Faster information  mixing   can be  obtained using suitably designed polynomials $P_K(\overline{\mathbf W})$, such as Chebyshev \citep{auzinger2011iterative,scaman2017optimal} or orthogonal (a.k.a. Jacobi) \citep{Berthier2020} polynomials.  

It is worth noting that the DGT algorithm contains as a special instance, the proximal gradient descent  when the graph $\mathcal G$ is a  star-network.  In this setting, by electing the first node as the master node (star center), the centralized PGD iterate 
	\begin{align}\label{alg:PGD}
		\btheta^{t+1} = \prod_{\|\btheta\|_1 \leq r}~ \left( \btheta^t - \gamma^{-1}\nabla \cL (\btheta^t)\right)
	\end{align}
can be naturally implemented on the star-network. Specifically,  at each iteration $t$, the master node broadcasts $\btheta^t$.  Then each client node $j$ evaluates $\nabla \cL_j (\btheta^t)$ and sends it to the master. The master node averages the gradients  $\nabla \cL(\btheta^t)=\frac{1}{m}\sum_{j=1}^m  \nabla \cL_j (\btheta^t)$ based on which the optimization step~\eqref{alg:PGD} is computed. The new value $\btheta^{t+1}$ is broadcast back to the clients.
It is not difficult to check that the above procedure can be encompassed by the DGT algorithm~\eqref{alg:DGT} using the mixing matrix  $\bW = \bW' \bW'^{\top}$, with $\bW' = 1/\sqrt{m}[\mathbf{1}_{m}, \mathbf{0}_{m \times m-1}]$.
Note that the resulting $\bW$ is the same as that of a complete graph, but implemented in a more communication efficient way.

We are ready to  present the main convergence properties of DGT and their implications.  
Our statements consists of two parts. In Sec.~\ref{sec:summary_results_RSC/RSM},  we    establish conditions  under which the network average optimization error $ \frac{1}{m} \sum_{i = 1}^m \| \btheta_i^t - \htheta\|^2$ shrinks  linearly up to a tolerance   $o(\|\htheta - \btheta^*\|^2)$, and provide an explicit expression of the rate. In Sec.~\ref{sec:summary_results_high_prob}, we show that for  the statistical model subsumed in~\eqref{eq:I/O} with random design matrix, this result  holds with high probability. We also discuss the impact of the network topology on the algorithm's performance. The  proofs of the theorems are deferred to  Sec.~\ref{sec:convergence}. 

\subsection{Convergence under RSC/RSM}\label{sec:summary_results_RSC/RSM}

Under Assumption~\ref{assump:G_RSM}-\ref{assump:weight}, define 
\begin{equation}\label{eq:Delta_stat_final}
\Delta_{\rm stat} \triangleq \frac{\rho}{2 L_\Sigma}\left(   \frac{ \ell_{\Sigma}}{\mu_{\Sigma}} \frac{5 C_2^2   c_m^2 }{(1 - \rho)^2}  \right) \tau_\ell \,\nu^2  + \frac{C_1 (\tau_\mu + \tau_g )}{L_\Sigma}   \,  \nu^2, \quad \text{with}\quad \nu \triangleq 2\, \|\htheta - \btheta^* \|_1 + 2 \sqrt{s} \| \htheta - \btheta^*\|,
\end{equation}
 and   
\begin{align}\label{eq:rate_expression_final}
       \rate \triangleq \frac{1 - (2\kappa)^{-1} + C_1 s (\tau_\mu + \tau_g)/L_\Sigma}{1 - 2 C_1 s \tau_g /L_\Sigma},\quad \text{with}\quad \kappa \triangleq \frac{L_\Sigma}{\mu_\Sigma},
\end{align}
where   $C_1 >0$, $C_2 >1$ are universal constants.

The convergence rate of DGT is given by the following theorem. 

\begin{theorem}\label{thm:convergence_deterministic_1}
Given the 
 linear model~\eqref{eq:I/O} with $\btheta^*$ being   $s$-sparse, consider the LASSO problem (\ref{p:regularized_ERM_constraint}) over  a network $\cG$ satisfying Assumption~\ref{assump:net1}. Suppose that the global     $\cL$ and   local   $\cL_i$ losses  satisfy Assumption~\ref{assump:G_RSM}  and Assumption~\ref{assump:L_RSS}, respectively. Let $\{(\btheta_i^t)_{i=1}^m\}$ be the sequence generated by the DGT algorithm with step size $\gamma$ chosen as
 \begin{align}\label{eq:gamma_expression_final-rept}
     	\gamma = L_\Sigma  +  4 \frac{  \ell_{\Sigma}^2 }{\mu_{\Sigma}}   \frac{c_m^2 \sqrt{\rho}}{ (1 - \rho)^4},
 \end{align}
 and the weight matrix $\mathbf{W}$ satisfying Assumption~\ref{assump:weight}. Further, assume that  \begin{equation}\label{cond_mu_sigma} 
 \mu_\Sigma > 36 C_1 s (\tau_\mu + \tau_g)
 \end{equation}  and
\begin{align}\label{eq:rho_final_condition}
     \rho   \leq \left\{2  \left(75 c_m^2 C_2^2 \frac{  \ell_{\Sigma}^2 }{\mu_{\Sigma}^2}  + \frac{ \ell_{\Sigma}}{\mu_{\Sigma}^2 } \cdot 6 C_2^2   c_m^2    s \tau_\ell   \right) \right\}^{-2} = \texttt{poly} \left(m, \frac{s \tau_\ell}{\mu_\Sigma}, \frac{\ell_\Sigma}{\mu_\Sigma} \right)^{-1}.
\end{align} 
Then, for any optimal solution $\htheta$ of the LASSO problem satisfying $\|\htheta\|_1 = r$, we have 
\begin{equation}\label{eq:linear_rate_final}
    \frac{1}{m} \sum_{i = 1}^m \| \btheta_i^t - \htheta\|^2 \leq \B \cdot \rate^t  + \frac{\Delta_{\rm stat}}{1 - \rate}, \quad \forall t=1, 2, \ldots, 
\end{equation} 
for $\rate \in (0,1)$, where 
$B$ is a constant depending on the initialization [c.f.~\eqref{eq:B_final}].
\end{theorem}


 The residual error $\Delta_{\rm stat}$ depends  on the parameters of the RSC/RSM conditions, the network connectivity $\rho$,  and the statistical error $\|\htheta-\btheta^\ast\|^2$.   The next corollary  shows that these parameters can be chosen such that  $\Delta_{\rm stat}=o\big(\|\htheta - \btheta^*\|^2\big)$. 

\begin{corollary}\label{cor:convergence_deterministic}
Instate assumptions of Theorem~\ref{thm:convergence_deterministic_1}; and suppose  $r \leq \|\btheta^*\|_1$ and 
\begin{align}\label{eq:sample_condition_cor_20}
    s (\tau_\mu + \tau_g) = o(1)  \quad \text{and} \quad s \cdot \rho c_m^2 \frac{\ell_\Sigma}{\mu_\Sigma} \tau_\ell = o(1).
\end{align}
Then, we have \vspace{-0.3cm}
\begin{align}
    \frac{1}{m} \sum_{i = 1}^m \| \btheta_i^t  - \htheta\|^2 \leq \rate^t \cdot B + o \left( \|\htheta - \btheta^* \|^2 \right).
\end{align}
\end{corollary}

\noindent 

\noindent $\bullet$ \textbf{Iteration complexity.} 
The contraction factor $\rate$ determining the linear  decay of the optimization error depends  on the restricted condition number $\kappa$  and the RSC/RSM parameters. The lack of strong convexity and smoothness deteriorates the convergence rate by introducing the terms related to the tolerance parameters $\tau_\mu, \tau_g$.
Nevertheless, it can be verified that under   condition~\eqref{cond_mu_sigma}, $\lambda \leq 1-(4\kappa)^{-1}$, and thus
such degradation  is limited. {  This
  matches the convergence rate of the centralized PGD up to a constant factor and    improves on  existing analyses of distributed algorithms   whose convergence to a solution of~\eqref{p:regularized_ERM_constraint} is certified only at {\it sublinear} rate (see Sec.~\ref{sec_related works}).}

 When special settings  are considered, the convergence rate reduces to well-known expressions.   In particular, for  strongly convex and smooth losses, i.e., $\tau_\mu=\tau_g=\tau_\ell = 0$, we have   $\Delta_{\rm stat} = 0$, implying that   DGT converges to an $\varepsilon$-solution of (\ref{p:regularized_ERM_constraint}) in $O(\kappa \log (1/\varepsilon))$ iterations. This recovers the result in \citep{sun2019distributed}.   When the network is fully connected, the first term in $\Delta_{\rm stat}$ vanishes (as $\rho=0$), implying that   DGT converges linearly to a neighborhood of $\htheta$ of size $O((\tau_\nu + \tau_g) \nu^2 /L_\Sigma))$. This  rate matches  (up to constant factors) that  of the centralized PGD under RSC/RSM  established  in~\citep{agarwal2012fast}. \\[1ex]
\noindent $\bullet$ \textbf{Interplay between optimization and communication.}
Convergence of DGT is established, in particular,  under  condition \eqref{eq:rho_final_condition}. This  reveals an interesting interplay between  communication cost and  hardness of the optimization problem. The quantity   $\ell_\Sigma/\mu_\Sigma$ can be viewed as the   restricted condition number of the local  losses $\cL_i$ (cf.~Assumption~\ref{assump:L_RSS});   larger $\ell_\Sigma/\mu_\Sigma$ and $s \tau_\ell/\mu_\Sigma$ correspond to  more ill conditioned functions  $\cL_i$. Therefore, the more ill conditioned $\cL_i,$  the   more connected  the network must be 
to achieve  convergence the rate of centralized PGD. When the network topology is given, the associated $\rho$ might not satisfy \eqref{eq:rho_final_condition}. If so, one can still enforce the desired threshold on $\rho$ by employing multiple rounds of communications. This will be discussed in detail in the next subsection.

\subsection{Guarantees for sparse linear regression under random designs}\label{sec:summary_results_high_prob}
We present some implications  of Theorem~\ref{thm:convergence_deterministic_1} and Corollary \ref{cor:convergence_deterministic} for the  model~\eqref{eq:I/O}, under the following assumption on the design matrix and noise. 

\begin{assumption}\label{assump:rand-design}
Consider the linear   model~\eqref{eq:I/O}. The design matrices $\{\bX_i\}_{i = 1}^m$ are i.i.d. random matrices drawn from the $\Sigma$-Gaussian ensemble; and the  elements of $\{\mathbf{n}_i\}_{i=1}^m$ are i.i.d. $\sigma^2$-sub-Gaussian random variables,  independent of the $\bX_i$'s. Define $$\mu_{\Sigma}\triangleq \frac{1}{2} \sigma_{\min} (\Sigma),\quad L_{\Sigma} \triangleq 2 \sigma_{\max} (\Sigma),\quad text{and}\quad  \zeta\triangleq \max_{j} \Sigma_{jj}.$$
\end{assumption}




\begin{theorem}\label{theorem:main}
Given the 
 linear model~\eqref{eq:I/O}, with $\btheta^*$  being  $s$-sparse, consider the LASSO problem (\ref{p:regularized_ERM_constraint}) over  a network $\cG$ satisfying Assumption~\ref{assump:net1}, and  design matrices $\{\bX_i\}_{i = 1}^m$ and the measurement noise satisfying Assumption~\ref{assump:rand-design}. Suppose  \begin{equation}\label{eq:cond_rho_hig-prob-setting}
 	\|\htheta\|_1 = r \leq \|\btheta^*\|_1,\quad s \log d/N < c_5 \cdot \frac{\mu_\Sigma}{\zeta},\quad \text{and}\quad \rho < (c_6 m^8 \kappa^4)^{-1} =  \texttt{poly}(m,\kappa)^{-1},
 \end{equation}
  for some constants $c_5, c_6>0$.    
  Let $\{(\btheta_i^t)_{i=1}^m\}$ be the sequence generated by the DGT algorithm with step size $\gamma$ chosen according to~\eqref{eq:gamma_expression_final-rept} and the weight matrix $\mathbf{W}$ satisfying Assumption~\ref{assump:weight}.  Then, the following holds:
 \begin{align}~\label{eq:convergece_whp}
    \frac{1}{m} \sum_{i = 1}^m \| \btheta_i^t - \htheta\|^2 \leq B\cdot \left(\frac{1 - (2\kappa)^{-1} + C' s \log d /N}{1 - C' s \log d/N}\right)^t   + c_7\left( \frac{\zeta}{L_\Sigma}\frac{s \log d}{N}  \right) \|\htheta - \btheta^*\|^2,
\end{align}
with probability at least  $1 - c_8 \exp (-c_9 \log d)$, for some constants $C', c_7, c_8, c_9>0$. The expression of  $B>0$ is given in \eqref{eq:stat_bound_B}.
\end{theorem}

Theorem~\ref{theorem:main} reveals several interesting properties of DGT, as discussed next.\smallskip 
 
\noindent $\bullet$ \textbf{Scalability with respect to the problem dimension.} For a fixed network with connectivity $\rho$ satisfying \eqref{eq:cond_rho_hig-prob-setting}, the dependency of the convergence rate on the ambient dimension $d$, the total sample size $N$, and sparsity level $s$    is only through  the ratio $s \log d /N$.  This implies that  the convergence rate   \eqref{eq:convergece_whp} is invariant under the asymptotics $s, N, d\to \infty$ and  $s \log d /N=C$ ($C$ is an universal constant)--resulting in the global  sample scaling  $N=O(s\log d)$.

 \noindent $\bullet$ \textbf{Near optimal sample complexity.} When $s \log d/N = o(1)$, the residual error  in~\eqref{eq:convergece_whp} is of  smaller order than the statistical precision $\|\htheta - \btheta^*\|^2$. Therefore,   DGT takes $O(\kappa \log (1/\varepsilon))$ iterations to reach an $\varepsilon$-neighborhood of a statistically optimal solution; this rate is of the same order of the centralized PGD under RSC/RSM~\citep{agarwal2012fast}.  Note also that    the condition $s \log d/N = o(1)$  (almost) matches  the   optimal sample complexity $N = \Omega (s \log (d/s))$ for  model~\eqref{eq:I/O} \citep{Wainwright_2019}, which  is necessary for any centralized method to consistently estimate the $d$-dimensional,  $s$-sparse $\btheta^*$ from $N$ samples.  

 Since   $N=n\cdot m$,  the condition $s \log d/N = o(1)$  
reveals an interesting interplay between  statistical error, convergence rate,  and network connectivity/communication cost, which is peculiar  of the distributed setting:  when agents do not have   enough local samples $n$ for statistical consistency but the overall sample size  $N$ across the network suffices--a situation that happens, e.g., when $m$ is large (large-scale network)--DGT still achieves centralized statistical accuracy at linear rate, provided that  the network is ``sufficiently'' connected, i.e., $\rho \leq \texttt{poly}(m,\kappa)^{-1}$. Therefore, the insufficiency of the local  sample size is compensated by   higher communication costs. The impact of network parameters and topology on the convergence of DGT is further elaborated next.

\noindent $\bullet$ \textbf{Logarithmic scalabilty of the communication cost with the network size.} When the graph    $\cG$ is not part of the design but given a-priori along with $\overline{\bW}$ (satisfying Assumption~\ref{assump:weight}), the condition $\rho \leq \texttt{poly}(m,\kappa)^{-1}$ can be  satisfied by running multiple rounds of consensus steps. More precisely, since $\bar{\rho} = \|\overline{\bW} - \mathbf{J}\| < 1$ (Assumption~\ref{assump:weight}), one can construct $\bW = \overline{\bW}^K$ such that $\rho = \bar{\rho}^K \leq \texttt{poly}(m,\kappa)^{-1}$. This corresponds to running $K$ consensus steps over the graph $\cG$ using   each time the   matrix $ \overline{\bW}$ in~\eqref{eq:tracking}. Note that for fixed $m$ (graph), such $K$ is always finite. In fact, one can see that taking 
 $K = O\big(\log (m\, \kappa)\cdot (1 - \rho)^{-1}\big)$   fulfills the requirement. If Chebyshev acceleration is used to employ multiple rounds of communications, coupled  with a symmetric matrix $\overline{\bW}$ (satisfying Assumption~\eqref{assump:weight}), one can show that the number of communication steps $K$  reduces to $O\big(\log (m\, \kappa)\cdot (1 - \rho)^{-1/2}\big)$. 
 Combining this fact with the iteration complexity \eqref{eq:convergece_whp},  one can conclude that, in the setting of Theorem~\ref{theorem:main},   DGT based on $\bW = \overline{\bW}^K$ drives $\frac{1}{m} \sum_{i = 1}^m \| \btheta_i^t - \htheta\|^2 $ within an $\varepsilon$-neighborhood of a statistical optimal solution in  the following number of communication steps\vspace{-0.2cm}
   \begin{equation}
     \label{eq:comcost}
     O\left(\frac{\kappa}{1 - \rho}  \log(m\kappa) \log (1/\varepsilon)\right).\end{equation} 
{       This approach significantly improves the communication complexity compared to DGD-like algorithms, such as those discussed in \citep{Ji-DGD21,ji2023distributed}. Unlike \eqref{eq:DGD-comm-complexity} and \eqref{eq:DGD-ATC},  (\ref{eq:comcost}) does not depend on the ambient dimension 
$d$; therefore, DGT avoids the speed-accuracy trade-off inherent in DGD algorithms. This marks the first instance of a distributed algorithm demonstrating such a desirable property in high-dimensional settings. } %

 \noindent $\bullet$ \textbf{Impact of the network topology.} The network topology affects the convergence rate of DGT through the quantities   $m$ and $\rho$. A larger $m$ or $\rho$ increases communication complexity, as indicated by \eqref{eq:comcost}. In addition, for networks of the same size $m$, the efficiency of information propagation depends on the graph's topology, which in \eqref{eq:comcost} is captured by  the dependency of $\rho$ on $m$~\citep{Nedic_Olshevsky_Rabbat2018}. Table~\ref{tab:comm-cost} provides, for some commonly used network topologies (see Fig.~\ref{fig:graph-top} for some examples), the dependency of $\rho$ on $m$, when     the lazy Metropolis rule is adopted;   for other rules, see \citep{charron2020}.  
While    $(1-\rho)^{-1}$    affects the total number of communications, it  does not capture the overall cost of communications, which also depends on the edge-density of the graph. Hence, counting as one channel use per communication, with each edge shared between two nodes, Table~\ref{tab:comm-cost}  reports also  the following quantities, capturing different  communications costs: the total channel use per communication round,  the  channel use of the busiest node, and the total   channel uses   to reach an $\varepsilon$ neighborhood  of a statistically optimal solution.  
 
From the first row of the table and \eqref{eq:comcost},  we infer that DGT  running  over a graph with stronger connectivity (smaller $\rho$) requires less communication rounds to converge. On the other hand,   as a   trade-off,   better connectivity generally results from   denser graphs, and thus yields higher communications costs (heavier channel uses), as reported in the table.  
We deduce that    DGT running on  the Erd\H{o}s-R\'{e}nyi graph with edge connecting probability $p = \log m/m$ achieves the lowest communication cost at the busiest node, and ties with  the implementation on the star (``push-pull'')  in the amount of total channel use\footnote{We note that the performance on ``push-pull'' (star) are better than those on the Erd\H{o}s-R\'{e}nyi when  the log factor in $\tO$ is explicitly accounted.}. For the path graph and 2-d grid graph, even though the channel use per communication round is low,   DGT takes more iterations to converge, and  overall a higher total communication cost. Comparing the two instances of the   Erd\H{o}s-R\'{e}nyi graphs (the last two columns in the table), one may prefer a   sparser graph at the price of more  communication rounds to a denser graph and less   communication rounds. Finally,   even though the Erd\H{o}s-R\'{e}nyi is relatively efficient for DGT, in practice one may not be able to construct such a topology, e.g.,  due to geographic constraints. In such scenarios, a path/grid graph is still a valuable choice. \vspace{-.3cm}

\section{Proof of Main Results}\label{sec:convergence}  
\vspace{-.1cm}
\subsection{Preliminaries}\vspace{-.1cm}

We begin  rewriting the iterates \eqref{eq:tracking}-\eqref{eq:loc_opt} in matrix form and introducing some notation, which is convenient for our developments.  
Given $(\btheta_i^t)_{i=1}^m$,  $(\bg_i^t)_{i=1}^m$,   and $(\nabla \cL (\btheta_i^t))_{i=1}^m$,  we introduce the following associated matrices: {}{
\begin{align}
\begin{split}
 & \bTheta^t   \triangleq  [\btheta_1^t , \ldots , \btheta_m^t]^\top,
 \quad  \Delta \bTheta^t   = [\Delta \btheta_1^t , \ldots , \Delta \btheta_m^t]^\top  \triangleq  \bTheta^{t+\frac{1}{2}} - \bTheta^{t},\\
&  \bG^t   \triangleq [\bg_1^t ,\ldots , \bg_m^t]^\top, \quad\text{and}\quad 
\nabla \cL(\bTheta^t)   \triangleq [ \nabla \cL_1 (\btheta_1^t), \ldots,  \nabla \cL_m (\btheta_m^t)]^\top.
\end{split}\label{eq:matrix_variables}
\end{align}}  
\noindent Furthermore, denote $\widehat{\bTheta} = [\underbrace{\widehat{\btheta}, \ldots, \widehat{\btheta}}_{m \text{ times}}]^\top$. 
Using \eqref{eq:matrix_variables}, \eqref{eq:tracking} can be rewritten in  compact form:
\begin{subequations}
\label{eq:matrix_alg}
\begin{align}
& \bTheta^t = \bW \left(\bTheta^{t-1} + \Delta \bTheta^{t-1}\right) \label{eq:update_theta_matrix}, \\ 
& \bG^t = \bW \left( \bG^{t-1} + \nabla \cL(\bTheta^t) - \nabla \cL(\bTheta^{t-1}) \right).\label{eq:update_y_matrix}
\end{align}
\end{subequations}

Introducing the average quantities 
\begin{align}\label{eq:average_theta_g}
& \bar{\btheta}^t \triangleq \frac{1}{m} \sum_{i=1}^m \btheta_i^t  
\quad \text{and}\quad \bar{\bg}^t \triangleq \frac{1}{m}\sum_{i=1}^m \bg_i^t, 
\end{align} 
we define the consensus errors on the $\btheta$, $\bg$-vectors as 
\begin{equation}\label{eq:theta_perp_scalar}
	\btheta_{i,\perp}^t\triangleq \btheta_i^t - \bar{\btheta}^t\quad \text{and}\quad \bg_{i,\perp}^t\triangleq \bg_i^t - \bar{\bg}^t,\quad i=1,\ldots, m.
\end{equation}
Note that in our setting, $\btheta_{i,\perp}^0=0$, for all $i=1,\ldots, m$. The matrix counterparts of (\ref{eq:theta_perp_scalar}) are
\begin{align}\label{eq:theta_perp}
& \bTheta_{\perp}^t \triangleq  [\btheta_{1,\perp}^t , \ldots , \btheta_{m,\perp}^t]^\top =(\mathbf{I}-\mathbf{J})\bTheta^t
\quad\text{and}\quad    \bG_{\perp}^t \triangleq  [\bg_{1,\perp}^t , \ldots , \bg_{m,\perp}^t]^\top= (\mathbf{I} - \mathbf{J})\bG^t,
\end{align}
where  $\mathbf{J}\triangleq \mathbf{1}\mathbf{1}^\top/m$. 

Using     (\ref{eq:matrix_alg}) and the fact that $(\mathbf{I} - \mathbf{J})\bW = \bW - \mathbf{J}$, the dynamic of (\ref{eq:theta_perp}) reads  \begin{subequations}
\begin{align}  
& \bTheta_{\perp}^{t} = (\bW - \mathbf{J})(\bTheta^{t-1}_{\perp} + \Delta \bTheta^{t-1})\label{eq:update_theta_perp}\\
& \bG^{t}_{\perp} = (\bW - \mathbf{J})(\bG_{\perp}^{t-1} + \nabla \cL(\bTheta^t) - \cL (\bTheta^{t-1})). \label{eq:update_err_y} 
\end{align}  \end{subequations}

The $\ell_1$-norm constraint plays a key role in enforcing  sparsity of feasible points. 
Specifically,  for any $\btheta\in \mathbb{R}^d$, with  $\| \btheta \|_1 \leq r$,  $\btheta - \htheta$ is   approximately $s$-sparse, as quantified  next.
\begin{lemma}[\cite{agarwal2012fast}]\label{lem:norm_bound} 
	Let $\htheta$ be any optimal solution of Problem~\eqref{p:regularized_ERM_constraint} such that $\|\htheta\|_1 = r$. Then for any $\|\btheta\|_1 \leq r$, there holds
	\begin{equation}\label{eq:norm_bound}
		\|\btheta - \htheta\|_1 \leq 2 \sqrt{s}\|\btheta - \htheta \| + \underbrace{ 2 \|\Delta^*\|_1 + 2 \sqrt{s} \| \Delta^*\|}_{\nu  },\vspace{-0.1cm}
	\end{equation}
	where $\Delta^* \triangleq \htheta - \btheta^*$ is the statistical error.
\end{lemma}

 \subsection{Analysis of the error dynamics} 

{  Due to the double stochasticity of $\bW$,   the iterates $\btheta_i^{t + \frac{1}{2}}$ and $\btheta_i^t$ generated by DGT must be feasible. As a result of Lemma~\ref{lem:norm_bound}, the directions $\btheta_i^{t + \frac{1}{2}} - \widehat{\btheta}$  and $\btheta_i^{t} - \widehat{\btheta}$ satisfy the approximate sparsity property defined  by~\eqref{eq:norm_bound}.  This adherence is a fundamental property leveraged in our proof to establish the desired convergence results, marking a  notable shift from traditional proof techniques found in the literature on decentralized optimization.} Based on this key property, we structure our analysis into the following steps.
\begin{itemize}
      \item[\bf (i)] {\bf Inexact descent property (Sec.~\ref{sec:inexact_SCA})}: {  We demonstrate that the    optimization residual  $\| \bTheta^{t} - \widehat{\bTheta} \|^2$ decreases linearly, up to a tracking error term arising from the discrepancy  between $\nabla \cL(\btheta_i^t)$ and its estimator $\bg_i^t$, and a tolerance term related to $\nu$ due to the approximate  sparsity of the direction $\btheta_i^{t} - \widehat{\btheta}$.}
    \item[\bf (ii)] {\bf Bounding the tracking error (Sec.~\ref{sec:track_err}):} {   We establish that  the tracking error can be effectively  bounded by   an average of all historical optimization residuals, weighted by the network connectivity  $\rho$ and a $\nu$-dependent tolerance term, akin to the one described in (i).  Aiming at exploiting the  local restricted smoothness   of $\cL_i$ and the approximate sparsity of the direction $\btheta_i^{t} - \widehat{\btheta}$, our approach introduces a novel line of analysis that diverges from traditional proof techniques  used to study convergence of gradient-tracking algorithms. 
    }
    \item[\bf (iii)] {\bf Combining error dynamics via the $z$-transform (Sec.~\ref{sec:small-gain}):} {  We combine the two error dynamics in (i) and (ii) employing the z-transform for finite-length sequences, to effectively deal with {\it   historical} optimization residuals, and   obtain a sufficient condition for the algorithm's convergence. Our method deviates from conventional approaches commonly found in the literature on gradient-tracking methods, such as \citep{Xu-TAC:hs,sun2019distributed,nedich2016achieving}, which often rely on the small gain theorem to analyze dynamics with only single-hop time dependencies. By employing the 
z-transform, we broaden the analytical framework to incorporate  more complex time dependencies.  }
\end{itemize}

\subsubsection{Inexact  descent}~\label{sec:inexact_SCA}
We begin with the analysis of the optimization step~\eqref{eq:loc_opt}. Since each $\bg_i$ aims to estimate the local gradient $\nabla \cL(\btheta_i)$, step~\eqref{eq:loc_opt} can be regarded as an inexact  version of the  proximal gradient update:
\begin{align}~\label{eq:exact_prox_grad}
	\btheta_i^{t + \frac{1}{2}} & =  \argmin_{\|\btheta\|_1 \leq r}~\left\{\nabla \cL(\btheta_i^t)^\top (\btheta_i - \btheta_i^t) + \frac{\gamma }{2} \| \btheta_i - \btheta_i^t\|^2 \right\}.
\end{align}
Therefore, we base our proof on the analysis of~\eqref{eq:exact_prox_grad} while taking into account the gradient estimation error $\nabla \cL(\btheta_i^t) - \bg_i^t$, due to the use in (\ref{eq:exact_prox_grad}) of the gradient tracking vector $\bg_i^t$ rather than  $\nabla \cL(\btheta_i^t)$. The following proposition establishes the one-step descent of the optimization error, up to additive errors.

\begin{proposition}\label{eq:inexact_opt} Consider the LASSO problem (\ref{p:regularized_ERM_constraint}) over the network $\mathcal{G}$, under Assumptions \ref{assump:net1}  and \ref{assump:G_RSM}; and  $\widehat{\btheta}$  such that $\|\widehat{\btheta}\|=r$. Let    $\{\btheta_i^t\}$ be the sequence generated by Algorithm~\ref{alg:DGT} under Assumption~\ref{assump:weight}. Then, the following holds: 
	 \begin{multline}\label{eq:descent_obj}
	\left(1 -   \frac{C_1 s \tau_g}{\gamma}  \right) \| \bTheta^{t+\frac{1}{2}} - \widehat{\bTheta} \|^2_F
	\leq   \left(1 - \frac{ \mu_\Sigma}{\gamma} + \frac{C_1 s (\tau_\mu + \tau_g)}{\gamma}    \right) \| \bTheta^{t} - \widehat{\bTheta} \|^2_F\\
	- \left(1 - \frac{L_\Sigma}{\gamma}  \right) \|\Delta \bTheta^t\|^2_F +  \frac{2}{\gamma}\,\underbrace{\sum_{i=1}^m(\nabla \cL(\btheta_i^t )- \bg_i^t )^\top ({\btheta}_i^{t+\frac{1}{2}} - \widehat{\btheta})}_{\triangleq \delta^t}  + \frac{C_1 (\tau_\mu + \tau_g )}{\gamma}     m \nu^2,
\end{multline}
 	for some universal constant $C_1 > 0$.
\end{proposition}\begin{proof}
	See~\ref{pf:inexact_opt}.\vspace{-0.2cm}
\end{proof}

{}{Note that  in a fully connected network, $\nabla \cL(\btheta_i^t )- \bg_i^t = \mathbf{0}$ and $\btheta_i^t = \btheta_j^t$, for all $i,j =1,\ldots, m$, and $t=1,\ldots, $. Setting $\gamma = L_\Sigma$,  \eqref{eq:descent_obj} recovers the convergence result of  the centralized PGD algorithm, as in~\citep[Thm. 1]{agarwal2012fast}.}

{}{The subsequent analysis focuses on the dynamics of the gradient tracking error $\delta^t$. }


\subsubsection{  Bounding the tracking error in high dimension}\label{sec:track_err}
%

{}{To study the evolution of the tracking error  $\delta^t$,
we   resort to the dynamics of $\bg_i^t$. Recalling the definition of   $\bar{\bg}^t$ [cf. (\ref{eq:average_theta_g})] and the tracking dynamics  (\ref{eq:tracking}), we have $$\bar{\bg}^t =  \bar{\bg}^{t-1} + \frac{1}{m}\sum_{i=1}^m \left( \nabla \cL_i(\btheta_i^t) - \nabla \cL_i(\btheta_i^{t-1})\right) \quad \text{and hence} \quad \bar{\bg}^t = \frac{1}{m}\sum_{i=1}^m \nabla \cL_i(\btheta_i^t),$$
where we used the initialization $\bar{\bg}^0 = \frac{1}{m}\sum_{i=1}^m \left( \nabla \cL_i(\btheta_i^0) \right).$ }

{}{This average preserving property  of the sum-gradient  suggests the following decomposition of  the tracking error $\delta^t$:
\begin{align}
\delta^t 
= & \sum_{i=1}^m \left(\nabla \cL(\btheta_i^t ) - \frac{1}{m} \sum_{j=1}^m \nabla \cL_j (\btheta_j^t) +\bar{\bg}^t- \bg_i^t \right)^\top ({\btheta}_i^{t +\frac{1}{2}} - \widehat{\btheta})\label{eq:bound_delta_1}\\
= & \sum_{i=1}^m \left( \frac{1}{m} \sum_{j=1}^m \nabla \cL_j(\btheta_i^t ) - \frac{1}{m} \sum_{j=1}^m \nabla \cL_j (\btheta_j^t) \right)^\top({\btheta}_i^{t + \frac{1}{2}} - \widehat{\btheta})  + \sum_{i=1}^m \left(\bar{\bg}^t- \bg_i^t \right)^\top ({\btheta}_i^{t +\frac{1}{2}} - \widehat{\btheta}).\notag
\end{align} \noindent \textbf{Limitation of current proof techniques:} The subsequent analysis, aiming at bounding the  gradient-tracking error $\delta^t$,    necessities   a departure from traditional techniques  developed in the literature of   gradient-tracking algorithms, such as   \citep{Xu-TAC:hs,sun2019distributed,nedich2016achieving}.  Traditional approaches often  handle the   error   $\delta^t$ in \eqref{eq:bound_delta_1} using the Cauchy-Schwarz inequality, which   yields the following type of bound: \begin{equation}\label{eq:delta_bound_lip}\delta^t \leq C^\prime  \cdot \ell_{\max}  \cdot \|\bTheta^t_{\perp}\|^2_F + C^{\prime\prime}  \|{\bTheta}^{t+\frac{1}{2}} - \widehat{\bTheta}\|^2_F+ C^{\prime\prime\prime}  \|{\bG}^t_{\perp}\|_F^2,\end{equation}  where   $\ell_{\max}=\max_{i} \big\|({1}/{n}) \bX_i^\top\bX_i\big\|_2$ is the largest  Lipschitz constants of the local gradients  $\nabla \cL_i$. The analysis would proceed absorbing the second term $\|{\bTheta}^{t+\frac{1}{2}} - \widehat{\bTheta}\|^2_F$  in \eqref{eq:delta_bound_lip} into   the LHS of (\ref{eq:descent_obj}), and   using the dynamics \eqref{eq:update_theta_perp}-\eqref{eq:update_err_y} along with the  Lipschitz continuity of $\nabla \cL_i$'s  to obtain the following bounds on the    remaining  consensus and tracking errors  $\|\bTheta^t_{\perp}\|^2_F $ and $\|\bG^t_{\perp}\|^2_F$ in \eqref{eq:delta_bound_lip} \citep{sun2019distributed}:       
      \begin{subequations}\label{eq:cons-tracking-classical}\begin{align}
     \label{eq:one-step-consensus-error}
      \|\bTheta^{t}_{\perp}\|_F &\leq  \rho \,\|\bTheta^{t-1}_{\perp}\|_F + \rho \,\|\Delta\bTheta^{t-1}\|_F,   \\
      \|\bG^{t}_{\perp}\|_F & \leq \rho\, \|\bG^{t-1}_{\perp}\|_F + \rho\, \ell_{\max}\,\|\bTheta^{t-1}_{\perp}\|_F + \rho\, \ell_{\max}\,\| \Delta \bTheta^{t-1}\|_F .\label{eq:one-step-track-error}
  \end{align}\end{subequations}   
 The rest of the proof in \citep{sun2019distributed} would chain   the inequalities  (\ref{eq:descent_obj}), \eqref{eq:delta_bound_lip}, and \eqref{eq:cons-tracking-classical}   using small-gain arguments,  yielding linear convergence rates, under a suitable choice of the stepsize $\gamma$ (and the network connectivity $\rho$), {\it provided that   strong convexity of $\cL$ is assumed}--see \citep[Propositions 3.6 \& 3.8]{sun2019distributed}.} 

  {}{This line of analysis fails in high dimensional scenarios, for the following reasons. {\bf (i)} 
  The centralized objective functions $\cL$ is RSC and not strongly convex globally, and the local objective functions $\cL_i$ are only convex. {\bf (ii)}   In the    typical asymptotics of high-dimensional problems--$d/N\to \infty$ and $s \log d /N = O(1)$--the Lipschitz constants ${\ell}_{\max}$ of $\nabla \cL_i'$s  scale as $\mathcal{O}(d)$, for various predictor $x_i'$ models \citep{Wainwright_2019},  thereby diverging as $d\uparrow \infty$. Even postulating to be able to relax the strong convexity requirement of $\cL$ to the RSC, establishing linear convergence along the path of the classical analysis   will require   the network connectivity $\rho$ (or the  step-size $\gamma$)  to counteract the unfavorable  $ {1}/{d}$ scaling of ${\ell}_{\max}$,  ensuring $\rho\, {\ell}_{\max}=\mathcal{O}(1)$,    if optimal centralized statistical precision is to be achieved. This would translate in the unsatisfactory scaling of  the total number of communications as $\log d$ (or $d$). Notice that, this issue remains   even under more restrictive assumptions on the sample size, such as $s\log d/n = \mathcal{O}(1).$ At the high level, this  arises because the gradients $\nabla \cL_i(\btheta_i)$'s remain non-sparse vectors even when $\btheta_i$'s are all sparse.}\smallskip 

 \noindent {  \textbf{A new line of analysis:}} 
{}{The challenges outlined above  underscore the need of a new mechanism to control the  gradient error $\delta^t$ in high-dimensional decentralized settings. The key observation is that the undesirable scaling of the tracking error $\delta^t$ with the ambient dimension $d$, steaming from  $\ell_{\max}$, is unavoidable if one pursues the statistically agnostic bound of  $\|\bG^{t}_{\perp}\|_F$ as given in  \eqref{eq:one-step-track-error}, which {\it globally} relates the gradient tracking error $\|\bG^{t}_{\perp}\|_F$ to the consensus error.   Our proposed approach capitalizes  on the   observation that $\delta^t$ primarily depends on the component of the gradient error projected along the directions $\btheta_i^{t+ \frac{1}{2}} - \widehat{\btheta},$ which are {\it approximately sparse} vectors. This contrasts with \eqref{eq:one-step-track-error}, which  considers the entirety of the gradient tracking error. More precisely, our proof leverages the fact that, for any set of row-wise feasible points   $\bTheta_1,$ $\bTheta_2,$ $\bTheta_3,$ and $\bTheta_4$, the inner products  \vspace{-.1cm}
\begin{subequations}
\label{eq:inner}
\begin{align}
    &(\bTheta_1 - \bTheta_2)^{\top}(\nabla \cL_j (\bTheta_3) - \nabla \cL_j(\bTheta_4)), \, \forall j \in [m]\\
    & (\bTheta_1 - \bTheta_2)^{\top}(\nabla \cL (\bTheta_3) - \nabla \cL(\bTheta_4))
\end{align}
\end{subequations}
can be controlled by quantities that do scale logarithmically with the ambient dimension at most. Therefore, in place of breaking the inner products using the Cauchy-Schwartz inequality, yielding to one-hop bounds   as   \eqref{eq:one-step-track-error},   our approach necessitates working directly with inner product relationships. The final result is a bound of $\delta^t$ that depends on the entire algorithm trajectory from the initialization. This is formalized in the next proposition.  }

\begin{proposition}\label{prop:track_error}
	Under the setting of Proposition~\ref{eq:inexact_opt} and the local restricted smoothness Assumption~\ref{assump:L_RSS}, the gradient tracking error $\delta^t$ ($t = 1,\ldots, $) can be bounded as
	\begin{align}\label{eq:tracking_err}
 		\delta^t \leq & \ \frac{C_2}{2} \left\{ \frac{\rho \cdot c_m \epsilon}{1 - \rho} \left( \ell_\Sigma +  s \tau_\ell  \right) \| {\bTheta}^{t + \frac{1}{2}} - \widehat{\bTheta} \|_F^2  
 		 +  \frac{ c_m \ell_\Sigma}{ \epsilon}\sum_{s=0}^{t-1} \rho^{t-s}  \| \Delta \bTheta^s\|_F^2 
 		+ \frac{ c_m \ell_\Sigma}{ \epsilon}\sum_{s=0}^{t-1} \rho^{t-s}  \|  \bTheta_\bot^{s}\|_F^2\right.\notag\\
 		& +c_m  s\tau_\ell \epsilon^{-1} \sum_{s=0}^{t-1}  \rho^{t-s}  \| \bTheta^{s+\frac{1}{2}} - \hTheta\|_F^2
 		+ c_m  s \tau_\ell \epsilon^{-1} \sum_{s=0}^{t-1}  \rho^{t-s}  \| \bTheta^{s - \frac{1}{2}} - \hTheta\|_F^2 \notag\\
 		&  \left.+ m c_m \rho^t \left( A_3 \cdot  \epsilon^{-1}  + s c_g^2 \ell_{\Sigma}^{-1} \epsilon^{-1} + c_g \nu\right) + \frac{\rho \cdot m c_m }{1 - \rho}  \tau_\ell \nu^2 (\epsilon + \epsilon^{-1})  \right\}
	\end{align}
for some $C_2 > 0$; where $\epsilon > 0$, $\bTheta^{-\frac{1}{2}} \triangleq \bTheta^0$ and
	\begin{equation}
		A_3  \triangleq   \ell_\Sigma \| \btheta^* \|^2 + \tau_\ell\|\btheta^* \|_1^2  , \quad
c_g  \triangleq \|\nabla \cL_j (\btheta^*)\|_\infty + \| \nabla \cL (\btheta^*) \|_\infty.
	\end{equation}
\end{proposition}
\begin{proof}
See~\ref{sec:pf_track_error}.
\end{proof}

Proposition~\ref{prop:track_error} reveals that the gradient tracking error $\delta^t$ can be bounded by the discounted cumulative sum of historical optimization errors $\{\|\bTheta^{s - 1/2} - \hTheta\|\}_{s=0}^{t+1}$, historical step-length  $\{\| \Delta \bTheta^s\|\}_{s = 0}^{t-1}$, and consensus errors $\{ \|\bTheta_\bot^s\| \}_{s = 0}^{t-1}$.
Furthermore, one can check that   $\delta^t$ decreases  if   $\rho$, $m$, or the right-hand-side tolerance parameter $\tau_\ell$ decrease, which is a consequence of  a    more connected network, and larger local sample size $n$, respectively.

	We can now combine   the statements obtained in Propositions \ref{eq:inexact_opt} and \ref{prop:track_error}, and write 
	\begin{align}\label{eq:twoinone}
&	\left(1 -   \frac{1}{\gamma} \left(C_1 s \tau_g  + \frac{ C_2 \rho \cdot c_m \epsilon}{1 - \rho} \left( \ell_\Sigma +  s \tau_\ell  \right) \right)\right) \| \bTheta^{t+\frac{1}{2}} - \widehat{\bTheta} \|^2_F \notag
\\
\leq   & \left(1 - \frac{ \mu_\Sigma}{\gamma} + \frac{C_1 s (\tau_\mu + \tau_g)}{\gamma}    \right) \| \bTheta^{t} - \widehat{\bTheta} \|^2_F - \left(1 - \frac{L_\Sigma}{\gamma}  \right) \|\Delta \bTheta^t\|^2_F   + \frac{C_1 (\tau_\mu + \tau_g )}{\gamma}     m \nu^2 \notag
\\
& 
 +  \frac{C_2}{\gamma}\frac{ c_m \ell_\Sigma}{ \epsilon}\sum_{s=0}^{t-1} \rho^{t-s}  \| \Delta \bTheta^s\|_F^2 
 + \frac{ C_2}{\gamma} \frac{ c_m \ell_\Sigma}{ \epsilon}\sum_{s=0}^{t-1} \rho^{t-s}  \|  \bTheta_\bot^{s}\|_F^2 \notag\\
& +\frac{ C_2}{\gamma} c_m  s\tau_\ell \epsilon^{-1} \sum_{s=0}^{t-1}  \rho^{t-s}  \| \bTheta^{s+\frac{1}{2}} - \hTheta\|_F^2
+ \frac{ C_2}{\gamma} c_m  s \tau_\ell \epsilon^{-1} \sum_{s=0}^{t-1}  \rho^{t-s}  \| \bTheta^{s - \frac{1}{2}} - \hTheta\|_F^2 \notag\\
& + \frac{ C_2}{\gamma} m c_m \rho^t \left( A_3 \cdot  \epsilon^{-1}  + s\, c_g^2 \ell_\Sigma^{-1} \,\epsilon^{-1}+ c_g \nu\right) + \frac{ C_2}{\gamma}\frac{\rho \cdot m c_m }{1 - \rho}  \tau_\ell \nu^2 (\epsilon + \epsilon^{-1}).
\end{align}

	Our last step bounds the  consensus error $\|\bTheta_\bot^t\|$  using $\{\| \Delta \bTheta^s\|\}_{s = 0}^{t-1}$, as given next. {}{The bound deviates from \eqref{eq:one-step-consensus-error}, as  it is a function of the entirety  sequence $\{\|\Delta \bTheta^s\|\}_{s \geq 0}.$ This will allow us to exploit the term $-\|\Delta \bTheta^t\|^2$ in \eqref{eq:twoinone} to control  consensus errors.}
 
\begin{proposition}\label{prop:consenus_err}
	The consensus error   $\bTheta_\bot^t$  can be  bounded as
	\begin{align}\label{eq:cons_error_dym}
		\|\bTheta_\bot^{t}\|_F^2 \leq 2 \rho^{2t} \underbrace{\| \bTheta_\bot^{0}\|_F^2}_{=0} +  \frac{2\rho}{1-\rho} \cdot \sum_{s=0}^{t-1} \rho^{t-s} \|\Delta \bTheta^s \|_F^2.
	\end{align}
\end{proposition}
\begin{proof}
The proof is standard and is provided in~\ref{sec:pf_consensus_err} for completeness.	
\end{proof}
 \vspace{-0.3cm}
 
\subsubsection{Combining the error dynamics}\label{sec:small-gain}
To establish linear convergence of DGT from (\ref{eq:twoinone})-(\ref{eq:cons_error_dym}), we will employ the $z$-transform for finite length sequences to  deal with the convolution-sum terms.   
{}{Notice that  small gain arguments typically used in \citep{nedich2016achieving, sun2019distributed} are not directly applicable to chain the inequalities (\ref{eq:twoinone})-(\ref{eq:cons_error_dym}) in our scenario. This is because, unlike  \citep{nedich2016achieving, sun2019distributed},  each sequence in the above inequalities is influenced by the complete historical dependencies of the others, rather than a simpler single-hop dependence.  
This motivates the introduction of the $z-$transform to handle an arbitrarily long memory.}

 We begin briefly reviewing the relevant properties of the $z$-transform and then apply it to the system (\ref{eq:twoinone})-(\ref{eq:cons_error_dym}). 

\subsubsection*{$\bullet$ The $z$-transform  and properties}\label{sec:z-transform}

	Let $\{a(t)\}$ be a sequence of nonnegative numbers. Define the length $K$   $z$-transform of  $\{a(t)\}$ as
\begin{align}
A^K(z) \triangleq \sum_{t=1}^K a(t)z^{-t},\quad   z\in \mathbb{R}_{++}.
\end{align}

 Some properties of the  $z$-transform instrumental to our developments are summarized next, whose proofs are provided in~\ref{app:z-transform}. 
\begin{lemma}~\label{lem:z_trans}
	Let $A^K(z)$ be the length $K$ $z$-transform of a nonnegative sequence  $\{a(t)\}$, and  $\rho\in (0,1)$. Then, the following hold:
	\begin{itemize}
		\item[\bf (i)] $\sum_{t=1}^K a(t+1) z^{-t} \geq z A^K(z) -  a(1),\quad \forall z\in \mathbb{R}_{++};$
		\item[\bf (ii)] $\sum_{t=1}^K\left( \sum_{s=0}^{t-1} \rho^{t-s} a(s) \right) z^{-t} \leq \dfrac{\rho}{z - \rho} \cdot \left( A^K (z) + a(0) \right), \quad \forall z \in (\rho,1);$
		\item[\bf (iii)] $\sum_{t=1}^K\left( \sum_{s=0}^{t-1} \rho^{t-s} a(s+ 1) \right) z^{-t} \leq z \cdot \dfrac{\rho}{z - \rho} A^K(z), \quad \forall z \in (\rho,1).$
	\end{itemize}
\end{lemma}


\begin{lemma}\label{lem:linear_rate}
	Suppose a nonnegative sequence $\{a(t)\}$ satisfies 
	\begin{align}
	\label{eq:condition}
	\sum_{t=1}^K a(t) z^{-t} \leq B + c \cdot \sum_{t=1}^K z^{-t},\quad \text{for all }\quad  K \geq 1\quad\text{and}\quad z \in (\bar{z},1),
	\end{align}
and some $B,c>0$. 	Then 
	\begin{equation}
	a(t) \leq B \cdot \bar{z}^t + \frac{c}{1-\bar{z}}.
	\end{equation}
	When $B < + \infty$,   $\{a(t)\}$ converges linearly at rate $\bar{z}$ up to a constant.
\end{lemma}


\subsubsection*{$\bullet$ The z-transform system}\label{sec:transformed-system}

Introduce the following quantities:
\begin{equation}\label{eq:z-transform-iterates}\begin{aligned}
D_\Theta^K(z)    = \sum_{t=1}^K z^{-t}  \| \bTheta^{t-\frac{1}{2}} - \widehat{\bTheta}\|^2,\quad  
\Delta_\Theta^K (z)   = \sum_{t=1}^K z^{-t}  \| \Delta \bTheta^t \|^2, \quad  
\Theta_\bot^K (z)   = \sum_{t=1}^K z^{-t} \| \bTheta_\bot^{t} \|^2.
\end{aligned}\end{equation}
To apply \eqref{eq:z-transform-iterates} to   \eqref{eq:twoinone} we first bound the 
 term  $\|\bTheta^{t} - \hTheta\|^2 $ on the RHS there as  $\|\bTheta^{t} - \hTheta\|^2 \leq \|\bTheta^{t-\frac{1}{2}}-\hTheta\|^2$ under condition
\begin{equation}\label{eq:cond_1}
 {\gamma\geq   { \mu_\Sigma} -  {C_1 s (\tau_\mu + \tau_g)}}.
\end{equation}
Further assuming 
\begin{equation}\label{eq:cond_2}
	\gamma\geq     C_1 s \tau_g  + \frac{ C_2 \rho \cdot c_m \epsilon}{1 - \rho} \left( \ell_\Sigma +  s \tau_\ell \right),
\end{equation}
we can multiply  $z^{-t}$ ($z > \rho$) to  both sides of (\ref{eq:twoinone}) and (\ref{eq:cons_error_dym})  while summing up from $t = 1$ to $t = K$ and applying Lemma~\ref{lem:z_trans} (recall that $\Delta \bTheta^0= \mathbf{0}$), leading to
\begin{align}
\begin{split}
	&	\left(1 -   \frac{1}{\gamma} \left\{C_1 s \tau_g  + \frac{ C_2 \rho \cdot c_m \epsilon}{1 - \rho} \left( \ell_\Sigma +  s \tau_\ell  \right) + \frac{ C_2 \rho \cdot c_m \epsilon^{-1} s \tau_\ell}{z-\rho}\right\}\right)  z D_{\Theta}^{K}(z) 
	\\
	& - \left(1 - \frac{ \mu_\Sigma}{\gamma} + \frac{C_1 s (\tau_\mu + \tau_g)}{\gamma}   + \frac{ C_2}{\gamma} c_m  s \tau_\ell \epsilon^{-1} \frac{\rho}{z - \rho} \right) D_{\Theta}^{K}(z) 
	\\
	\leq   
	& - \left(1 - \frac{L_\Sigma}{\gamma}  - \frac{C_2}{\gamma}\frac{ c_m \ell_\Sigma}{ \epsilon}\frac{\rho}{ z - \rho} - \frac{ C_2}{\gamma} \frac{ c_m \ell_\Sigma}{ \epsilon}\frac{2\rho}{1-\rho} \cdot \left(\frac{\rho}{z - \rho}\right)^2 \right) \Delta_\Theta^K (z)   \\
	& + \overline{\Delta}_{\mathrm{stat}} \cdot \sum_{t = 1}^K z^{-t} + B(z),
\end{split}
\end{align}
with
\begin{align}\label{eq:BZ}
\begin{split}
B(z) \triangleq  	\ & \frac{ C_2}{\gamma} c_m  s \tau_\ell \epsilon^{-1} \frac{\rho}{z - \rho}  \|\bTheta^{0}-\hTheta\|^2 
+\frac{ C_2}{\gamma} m c_m \left( A_3 \cdot  \epsilon^{-1}  + s c_g^2 \ell_{\Sigma}^{-1} \epsilon^{-1} + c_g \nu\right) \frac{\rho}{z - \rho}\\
& + \left(1 -   \frac{1}{\gamma} \left\{C_1 s \tau_g  + \frac{ C_2 \rho \cdot c_m \epsilon}{1 - \rho} \left( \ell_\Sigma +  s \tau_\ell  \right) \right\}\right)   \|\bTheta^{1/2} - \hTheta\|^2
\end{split}
\end{align}
and
\begin{align}\label{eq:Delta_stat}
\begin{split}
\overline{\Delta}_{\mathrm{stat}} = \ & \frac{ C_2}{\gamma}\frac{\rho \cdot  c_m }{1 - \rho}  \tau_\ell \nu^2 (\epsilon + \epsilon^{-1}) + \frac{C_1 (\tau_\mu + \tau_g )}{\gamma}     \nu^2.
\end{split}
\end{align}

The following conditions are sufficient to invoke Lemma~\ref{lem:linear_rate} and prove the  convergence of  sequence   $\{(1/m) \| \bTheta^{t + \frac{1}{2}} - \hTheta\|_F^2\}$:
\begin{itemize}
    \item[i)] The step size condition (\ref{eq:cond_1}) and (\ref{eq:cond_2})  hold.
    \item[ii)]
     There exists some $z\in (\rho, 1)$ such that
\begin{subequations}\label{eq:rate_condtion}
	\begin{align}  
\begin{split}\label{eq:rate_condition_1}
& \left[1 -   \frac{1}{\gamma} \left(C_1 s \tau_g  + \frac{ C_2 \rho \cdot c_m \epsilon}{1 - \rho} \left( \ell_\Sigma +  s \tau_\ell  \right) + \frac{ C_2 \rho \cdot c_m \epsilon^{-1} s \tau_\ell}{z-\rho}\right)\right]  z   \\
& >  1 - \frac{ \mu_\Sigma}{\gamma} + \frac{C_1 s (\tau_\mu + \tau_g)}{\gamma}   + \frac{ C_2}{\gamma} c_m  s \tau_\ell \epsilon^{-1} \frac{\rho}{z - \rho}; \bigskip
\end{split} \bigskip\\
\begin{split}\label{eq:rate_condition_2}
1 - \frac{L_\Sigma}{\gamma}  - \frac{C_2}{\gamma}\frac{ c_m \ell_\Sigma}{ \epsilon}\frac{\rho}{ z - \rho} - \frac{ C_2}{\gamma} \frac{ c_m \ell_\Sigma}{ \epsilon}\frac{2\rho}{1-\rho} \cdot \left(\frac{\rho}{z - \rho}\right)^2 > 0.
\end{split}
\end{align}
\end{subequations}
\end{itemize}


 \subsection{Linear convergence under RSC/RSM (Proof of Theorem~\ref{thm:convergence_deterministic_1})}

The proof is organized  into the following steps:
  \begin{itemize}
      \item In Sec.~\ref{sec:rate-expression-step1} we  provide   choices of $\epsilon$ and $\gamma$ for the existence of  $z>\rho$  satisfying (\ref{eq:rate_condtion}).
      \item Under such choices, we bound $B(z)$ [cf.~(\ref{eq:BZ})] and 
 $\overline{\Delta}_{\mathrm{stat}}$ [cf.~\eqref{eq:Delta_stat}] in Sec.~\ref{sec:rate-expression-step2}.
    \item  In Sec.~\ref{sec:rate-expression-step3} we  establish conditions for the linear convergence of  $\left\{\frac{1}{m} \| \bTheta^{t + \frac{1}{2}} - \hTheta\|_F^2\right\}_t$  up to a residual error depending on the statistical precision  $\|\htheta - \btheta^*\|^2$.
  \end{itemize}

\subsubsection{Choice of $\epsilon$ and $\gamma$}\label{sec:rate-expression-step1}

Lemma~\ref{lem:rate_condition_2}  and~\ref{lem:rate_condition_1} below provide conditions and the range of $z>\rho$ for~\eqref{eq:rate_condition_2} and~\eqref{eq:rate_condition_1} to hold, respectively.  
For convenience we set $\epsilon$ as \begin{align}~\label{eq:epsilon}
\epsilon = \frac{\mu_{\Sigma}}{ \ell_{\Sigma}} \frac{1 - \rho}  {2 C_2   c_m }.
\end{align}

\begin{lemma}\label{lem:rate_condition_2}
Given $\gamma>L_\Sigma$,	condition~\eqref{eq:rate_condition_2} is satisfied for all $z >z_\rho$, with
	\begin{equation}\label{eq:z_rho}
		z_\rho \triangleq  \rho + \frac{4 \rho^2}{1 - \rho} \left( \sqrt{1 + \left(\gamma - L_\Sigma \right) \left(\frac{\mu_{\Sigma}}{ \ell_{\Sigma}^2} \frac{ 2 \rho}  { C_2^2   c_m^2 }\right)  } - 1 \right)^{-1}.
	\end{equation}
\end{lemma}
\begin{proof}
See~\ref{sec:proof_lem:rate_condition_2}.	
\end{proof}

We focus now on condition~\eqref{eq:rate_condition_1}. Denote for simplicity
\begin{align}\label{eq:AB}
	\begin{split}
		A & \triangleq  \frac{1 - \frac{ \mu_\Sigma}{\gamma} + \frac{C_1 s (\tau_\mu + \tau_g)}{\gamma} }{\frac{ C_2}{\gamma} c_m  s \tau_\ell \epsilon^{-1}}
		= \frac{\gamma -  \mu_\Sigma+ C_1 s (\tau_\mu + \tau_g) }{  \frac{ \ell_{\Sigma}}{\mu_{\Sigma}} \frac{2 C_2^2   c_m^2 }{1 - \rho}   s \tau_\ell  },
		\\
		D & \triangleq \frac{1 -   \frac{1}{\gamma} \left\{C_1 s \tau_g  + \frac{ C_2 \rho \cdot c_m \epsilon}{1 - \rho} \left( \ell_\Sigma +  s \tau_\ell  \right) \right\}}{\frac{ C_2}{\gamma} c_m  s \tau_\ell \epsilon^{-1}}
		= \frac{\gamma -   \left\{C_1 s \tau_g  +  \frac{\rho}{2} \frac{\mu_{\Sigma}}{ \ell_{\Sigma}}  \left( \ell_\Sigma +  s \tau_\ell  \right) \right\}}{ \frac{ \ell_{\Sigma}}{\mu_{\Sigma}} \frac{2 C_2^2   c_m^2 }{1 - \rho}   s \tau_\ell },
	\end{split}
\end{align}
and 
\begin{align}
	\sigma \triangleq \frac{A}{D}= \frac{\gamma -  \mu_\Sigma+ C_1 s (\tau_\mu + \tau_g) }{\gamma -   \left\{C_1 s \tau_g  +  \frac{\rho}{2} \frac{\mu_{\Sigma}}{ \ell_{\Sigma}}  \left( \ell_\Sigma +  s \tau_\ell  \right) \right\}}.
\end{align}
 {Note that, under \eqref{eq:cond_1} and \eqref{eq:cond_2} strictly satisfied, $A,D>0$.} Then~\eqref{eq:rate_condition_1} can be rewritten succinctly as
\begin{align}\label{eq:z_ub_2}
	z \geq \frac{A + G_{{\rm net}}(z)}{D - G_{{\rm net}}(z)},\qquad G_{{\rm net}}(z) \triangleq  \frac{\rho}{z - \rho}.
\end{align}
\begin{lemma}\label{lem:rate_condition_1}
	Given $\gamma$  satisfying~\eqref{eq:cond_1} and \eqref{eq:cond_2}, condition~\eqref{eq:rate_condition_1} is satisfied for all $z > z_\sigma$, with
	\begin{equation}\label{eq:z_sigma}
		z_\sigma \triangleq   \sigma +   \rho\cdot \left(1+\frac{1}{D}\right),
	\end{equation}
	and   $A$ and $D$  defined in~\eqref{eq:AB}.
\end{lemma}
\begin{proof} See~\ref{sec:proof_lem:rate_condition_1}.
\end{proof}

Combining Lemma~\ref{lem:rate_condition_2} and~\ref{lem:rate_condition_1} we conclude that, for any given   $\gamma>L_\Sigma$, under conditions~\eqref{eq:cond_1} and \eqref{eq:cond_2}, \eqref{eq:rate_condition_1}-\eqref{eq:rate_condition_2} are satisfied by the choices of  $z$ such that 
\begin{equation} \label{eq:bar_z}
z > \bar{z}, \quad \text{with}\quad  \bar{z} \triangleq  \max \{z_\rho, z_\sigma\}.
\end{equation}


In the remaining of this section,  we will simplify the expression of $\bar z$ by making a specific choice of   $\gamma$.  
We do not aim to find the optimal value of $\gamma$ but a convenient feasible choice yielding  an insightful expression of the rate.  We set
\begin{equation}\label{eq:gamma_expression}
	\gamma = L_\Sigma  +  C_3 \frac{  \ell_{\Sigma}^2 }{\mu_{\Sigma}}   \frac{c_m^2 \sqrt{\rho}}{ (1 - \rho)^4},
\end{equation}
where $C_3>0$ is some   absolute constant (to be determined). 
 Define \begin{equation}\label{eq:sigma_0}
	\sigma_{0} \triangleq  \frac{1 - \kappa^{-1} + C_1 s (\tau_\mu + \tau_g)/ L_\Sigma }{1 - C_1 s \tau_g /L_\Sigma},\quad \kappa \triangleq \frac{L_{\Sigma}}{\mu_\Sigma}.
\end{equation}

\begin{proposition}\label{prop:z-upbdd}
    Let $\gamma$ be chosen according to~\eqref{eq:gamma_expression}. Under the following conditions:
   \begin{align}
       \begin{split}\label{eq:cond_denom_z_sigma} 
           1 - 2 C_1 s \tau_g /L_\Sigma>0,
       \end{split}\\
       \begin{split}\label{eq:sample_condition}	
         	\mu_{\Sigma}   > 4 C_1 s (\tau_\mu + \tau_g)   +      \frac{ \ell_{\Sigma}} {\mu_{\Sigma}}\frac{8 C_2^2   c_m^2 \rho }{(1 - \rho)^2}   s \tau_\ell +  \rho  \cdot \frac{\mu_{\Sigma}}{ \ell_{\Sigma}}     s \tau_\ell  ,
       \end{split}\\
       \begin{split}
          \text{and} \qquad  \rho \leq 1/2,
       \end{split}
   \end{align}
   $\bar{z}$ can be upperbounded as
   \begin{align}\label{eq:rate_expression}
     \bar{z} \leq 	 &\max\left\{\rho +  \frac{ \sqrt{2} C_2   } {  \sqrt{C_3}  }  \rho (1 - \rho) + \frac{ 2 C_2^2   } {  C_3  } \sqrt{\rho} (1 - \rho)^3 ,\right.\nonumber\\ &\hspace{1.3cm}\left.    \sigma_0 + \sqrt{\rho} \cdot \frac{18 C_3  c_m^2 \frac{  \ell_{\Sigma}^2 }{\mu_{\Sigma}L_\Sigma} + \frac{ \ell_{\Sigma}}{\mu_{\Sigma} L_{\Sigma}} \cdot 4 C_2^2   c_m^2    s \tau_\ell   +   \frac{2 \mu_{\Sigma}}{ \ell_{\Sigma} L_{\Sigma}}     s \tau_\ell  }{1 - 2 C_1 s \tau_g/L_{\Sigma} } \right\}.  
   \end{align}
\end{proposition}
\begin{proof}
    See~\ref{pf:prop:rate_bound_1}.
\end{proof}

To further simplify (\ref{eq:rate_expression}), we require the second argument in the $\max$-expression in~\eqref{eq:rate_expression} to be  no larger than
\begin{equation}
\frac{1 - (2\kappa)^{-1} + C_1 s (\tau_\mu + \tau_g)/L_\Sigma}{1 - 2 C_1 s \tau_g /L_\Sigma}.
\end{equation}

\begin{corollary}\label{cor:z-upbdd}
    Let $\gamma$ be chosen as
    \begin{equation}\label{eq:gamma_expression_final}
	\gamma = L_\Sigma  +  4 \frac{  \ell_{\Sigma}^2 }{\mu_{\Sigma}}   \frac{c_m^2 \sqrt{\rho}}{ (1 - \rho)^4}.
\end{equation} 
Under the following conditions:
\begin{align}
\begin{split}\label{eq:sample_condition_1}
     \mu_\Sigma   >  36 C_1 s (\tau_{\mu} + \tau_g ),    
\end{split}\\
\begin{split}\label{eq:bound_rho_1}
     \rho  \leq \left\{2  \left(72 c_m^2 \frac{  \ell_{\Sigma}^2 }{\mu_{\Sigma}^2} + \frac{ \ell_{\Sigma}}{\mu_{\Sigma}^2 } \cdot 4 C_2^2   c_m^2    s \tau_\ell   +   \frac{2 }{ \ell_{\Sigma} }     s \tau_\ell \right) \right\}^{-2},  
\end{split}
\end{align}
it holds
\begin{equation}\label{eq:def_z_bar}
\bar{z}\leq {\bar{z}_\text{up}}\triangleq \max\left\{\rho +  C_2 \rho  + \frac{  C_2^2   } {  2  } \sqrt{\rho} , \frac{1 - (2\kappa)^{-1} + C_1 s (\tau_\mu + \tau_g)/L_\Sigma}{1 - 2 C_1 s \tau_g /L_\Sigma}\right\}.
\end{equation}
\end{corollary}

\begin{proof}
    See~\ref{pf:cor-z-upbdd}.
\end{proof}




\subsubsection{Bounding $B(z)$ [cf.~(\ref{eq:BZ})] and $\Delta_{\rm stat}$ [cf.~\eqref{eq:Delta_stat}]}\label{sec:rate-expression-step2}

   According to Eq.~\eqref{eq:rate_rho}, under the setting of Proposition~\ref{prop:z-upbdd}, any  $z$ satisfying $z \geq \bar{z}_{\text{up}}$ is lower bounded by  
\begin{align}\label{eq:lower-bound-z-step-2}
	\begin{split}
		z 
		\geq 
		 \rho +  \sqrt{\rho}(1 - \rho)^3 \frac{ C_2^2   } {  C_3  }.
	\end{split}
\end{align}
Substituting the expression of $\gamma$ [cf.~\eqref{eq:gamma_expression_final}] and $\epsilon$ [cf.~\eqref{eq:epsilon}] into $B(z)$ [cf.~(\ref{eq:BZ})] and $\Delta_{\rm stat}$ [cf.~\eqref{eq:Delta_stat}], together with~\eqref{eq:lower-bound-z-step-2} we obtain the following bound.

\begin{proposition}\label{prop:bound-B-STAT}
Under the setting of  Corollary~\ref{cor:z-upbdd} and  assume $C_2 \geq 1$, then 
\begin{align}
    \begin{split}\label{eq:B_final}
        \frac{B(z)}{m} & \leq  \frac{C_4}{m} \left( \|\bTheta^0 - \hTheta\|^2_F + \|\bTheta^*\|^2_F + \|\bTheta^{1/2} - \hTheta\|^2_F + \frac{m c_m \sqrt{\rho}}{L_\Sigma}\left( \frac{c_m s c_g^2}{\mu_\Sigma} + c_g \nu\right)\right) \triangleq \B,
    \end{split}\\ 
\begin{split}\label{eq:delta_stat}
        \overline{\Delta}_{\rm stat} 
& \leq \frac{\rho}{2 L_\Sigma}\left(   \frac{ \ell_{\Sigma}}{\mu_{\Sigma}} \frac{5 C_2^2   c_m^2 }{(1 - \rho)^2}  \right) \tau_\ell \nu^2  + \frac{C_1 (\tau_\mu + \tau_g )}{L_\Sigma}     \nu^2   \triangleq {\Delta}_{{\rm stat}} 
\end{split}
\end{align}
for some $C_4 > 0$, 
where  $c_g$ and $\nu$ are defined as 
\begin{align}
    	c_{g} = \max_{j \in [m]}\|\nabla \mathcal{L}_{j}(\btheta^{*})\|_{\infty} + \|\nabla \mathcal{L} (\btheta^{*})\|_{\infty} \quad \text{and} \quad 
    	\nu = 2 \|\htheta - \btheta^* \|_1 + 2 \sqrt{s} \| \htheta - \btheta^*\|.
\end{align}
\end{proposition}

\begin{proof}
See~\ref{pf:bound-B-STAT}. 
\end{proof}


\subsubsection{Linear convergence up to $o(\|\htheta - \btheta^*\|^2)$}\label{sec:rate-expression-step3}

 We prove Theorem~\ref{thm:convergence_deterministic_1} and Corollary~\ref{cor:convergence_deterministic}. Recall below the network connectivity condition~\eqref{eq:rho_final_condition} and rate expression $\rate$~\eqref{eq:rate_expression_final} given in Theorem~\ref{thm:convergence_deterministic_1} for convenience:
 \begin{align*}
     \begin{split}
             \rho   \leq \left\{2  \left(75 c_m^2 C_2^2 \frac{  \ell_{\Sigma}^2 }{\mu_{\Sigma}^2}  + \frac{ \ell_{\Sigma}}{\mu_{\Sigma}^2 } \cdot 6 C_2^2   c_m^2    s \tau_\ell   \right) \right\}^{-2}
     \end{split}
     \begin{split}
          \rate = \frac{1 - (2\kappa)^{-1} + C_1 s (\tau_\mu + \tau_g)/L_\Sigma}{1 - 2 C_1 s \tau_g /L_\Sigma}.
     \end{split}
 \end{align*}
   Based on Lemma~\ref{lem:linear_rate} and the results in Sec.~\ref{sec:rate-expression-step1} \&~\ref{sec:rate-expression-step2}, it suffices to show~\eqref{eq:rho_final_condition}
implies~\eqref{eq:bound_rho_1} and $\bar{z}_{\rm up}\leq \rate <1$.

\noindent $\bullet$ \eqref{eq:rho_final_condition}$\implies$~\eqref{eq:bound_rho_1}: the claim holds since $C_2 > 1$, $\ell_\Sigma \geq \mu_\Sigma$, and $c_m \geq 1$.


  \noindent $\bullet$ $\bar{z}_{\rm up}\leq \rate <1$: it is not difficult to check that  $\mu_\Sigma > 36 C_1 s (\tau_\mu + \tau_g)$  implies $ \lambda <1 - (4 \kappa)^{-1}$. Next we prove 
\begin{align*}
   \rho +  C_2 \rho  + \frac{  C_2^2   } {  2  } \sqrt{\rho} \leq \rate,
\end{align*}
and thus $\bar{z}_{\rm up} = \rate$. It is sufficient to check that the following chain of inequalities hold:
\begin{align}
\begin{split}
    & \left(\rho +  C_2 \rho  + \frac{  C_2^2   } {  2  } \sqrt{\rho} \right) \left( 1 - 2 C_1 s \tau_g/L_\Sigma\right)\\
    \leq & \ 3 C_2^2   \sqrt{\rho} \stackrel{\eqref{eq:rho_final_condition}}{\leq}  \frac{1}{50} \cdot  \frac{\mu_\Sigma^2}{c_m^2 \ell_\Sigma^2} \leq \frac{1}{2 \kappa} \leq 1 - \frac{1}{2 \kappa} + C_1 s (\tau_\mu + \tau_g)/L_\Sigma,
    \end{split}
\end{align}
 where we used $\rho < 1$, $c_m, C_2 \geq 1$, $\ell_\Sigma \geq L_\Sigma$ and $\kappa \geq 1$. This completes the proof of Theorem~\ref{thm:convergence_deterministic_1}.

Lastly, we prove Corollary~\ref{cor:convergence_deterministic}. Using~\citep[Lemma 5]{agarwal2012fast} yields
$
    \| \htheta - \btheta^* \|_1 \leq 2\sqrt{s} \|\htheta - \btheta^*\| 
$.  Using  $\rho \leq 1/4$ and   conditions in~\eqref{eq:sample_condition_cor_20} it is not difficult   to verify that $\Delta_{\rm stat}$ in \eqref{eq:Delta_stat_final} is $o \left( \|\htheta - \btheta^* \|^2 \right)$.


\subsection{Linear convergence for sparse vector regression (Proof of Theorem~\ref{theorem:main})}\label{sec:proof_convergence_regression}

We now customize  Theorem~\ref{thm:convergence_deterministic_1} for  the random design setting satisfying Assumption~\ref{assump:rand-design}.
We first state the following     high probability bounds for the  RSC/RSM properties.
\begin{proposition}[\cite{raskutti2010restricted}]~\label{prop:g-RSC}
	Given $m$ i.i.d. distributed random matrices $\bX_i \in \mathbb{R}^{n \times d}$ drawn from the $\Sigma$-Gaussian ensemble. Let $\bX = [\bX_1^\top, \ldots, \bX_m^\top]^\top \in \mathbb{R}^{N \times d}$, with $N = n \cdot m$. Then, there exist universal constants $c_0, c_1>0$  such that 
	\begin{align}
		\text{(Global RSC/RSM) }\qquad
		\begin{split}\label{eq:G_RE_hp}
			\frac{\| \bX \bu\|^2}{N} & \geq \frac{1}{2} \| \Sigma^{1/2} \bu\|^2 - c_1\zeta \frac{\log d}{N} \|\bu\|_1^2,\\
			\frac{\| \bX \bu\|^2}{N} & \leq 2 \| \Sigma^{1/2} \bu\|^2 + c_1 \zeta\frac{\log d }{N} \|\bu\|_1^2, 
		\end{split} \quad \forall \ \bu \in \real^d;
	\end{align}
	with probability greater than $1 - \exp(-c_0 N)$, where $\zeta = \max_{j} \Sigma_{jj}$.
\end{proposition}

By slightly modifying the proof of~\citep[Theorem 1]{raskutti2010restricted}, provided in~\ref{App:proof_prop:LRSS}  for completeness, we can obtain the following bound for the local RSM condition. 

\begin{proposition}~\label{prop:LRSS}
	Given a random matrix $\bX_i \in \mathbb{R}^{n \times d}$ drawn from the $\Sigma$-Gaussian ensemble,  there exists universal constants $c_0, c_1>0$ such that
\begin{align}\label{eq:L_RE}
	\frac{\| \bX_i \bu\|^2}{n} \leq 16 m\|\Sigma^{1/2} \bu\|^2 + c_1 \zeta \frac{m \log d}{n} \|\bu\|_1^2 , \quad \forall \ \bu \in \real^d, i \in [m]
\end{align}
with probability greater than $1 - \exp(-c_0 N)$, where $\zeta = \max_{j} \Sigma_{jj}$.  
\end{proposition}

Next we derive a high probability bound for $	c_{g} = \max_{j \in [m]}\|\nabla \mathcal{L}_{j}(\btheta^{*})\|_{\infty} + \|\nabla \mathcal{L} (\btheta^{*})\|_{\infty}$.  
  \begin{lemma}\label{lem:gradient_bound}
    Consider the linear   model~\eqref{eq:I/O} under Assumption~\ref{assump:rand-design}.  Then
\begin{align}\label{eq:bound_loc_gradient}
    \begin{split}
    \mathbb{P} \left( \max_{j \in [m]}\left\|\frac{\bX_j^\top \mathbf{n}_j}{n} \right\|_\infty \geq  \sqrt{\zeta} \sigma  \cdot \max \left\{ \sqrt{\frac{2}{c_3}\frac{ \log md}{n}},  
    \frac{2}{c_3}\frac{ \log md}{n}\right\}\right) \leq  2\exp \left( -  \log d\right)
    \end{split}
\end{align}
and
  \begin{align}\label{eq:bound_g_gradient}
      \mathbb{P} \left( \left\| \frac{\bX^\top \mathbf{n}}{N}\right\|_\infty \geq \sqrt{\zeta} \sigma   \cdot \max \left\{ \sqrt{\frac{2}{c_3}\frac{ \log d}{N}},  
    \frac{2}{c_3}\frac{ \log d}{N}\right\}\right) \leq  2\exp \left( -  \log d\right)
  \end{align}
  for some $c_3 >0$.
  \end{lemma}
   \begin{proof}
   	See~\ref{proof_lem:gradient_bound}. 
   \end{proof}
   According to~\eqref{eq:I/O} we  have $\|\nabla \cL(\btheta^*)\|_\infty = \|\bX^\top \mathbf{n}/N\|_\infty$ and $\| \nabla \cL_j (\btheta^*)\|_\infty = \|\bX_j^\top \mathbf{n}_j/n\|_\infty$. Using Lemma~\ref{lem:gradient_bound} we can readily obtain the following bound for $c_g^2$.
  \begin{corollary}\label{Cor_upper_bound_cg}
 Reintate the conditions  of Lemma~\ref{lem:gradient_bound}; there holds: 
  \begin{align}\label{eq:c_g_prob_bound}
      c_g^2 & \leq 4 \zeta \sigma^2 \left( \frac{2}{c_3} \frac{\log md}{n}   + \left( \frac{2}{c_3} \frac{\log md}{n}  \right)^2 + \frac{2}{c_3} \frac{\log d }{N}+  \left(\frac{2}{c_3} \frac{\log d}{N} \right)^2\right)\nonumber\\
      & \leq 4 \zeta \sigma^2 \left( \frac{4}{c_3} \frac{m\log md}{N}   + \left( \frac{4}{c_3} \frac{m\log md}{N}  \right)^2 \right),
  \end{align}
  with probability greater than $1 - 4 \exp{(-c_4 \log d)}$, for some $c_4>0$.
  \end{corollary}

  Leveraging  Theorem~\ref{thm:convergence_deterministic_1} with   the parameters  $(\mu_\Sigma,L_{\Sigma}, \ell_\Sigma)$,  $(\tau_{\mu},\tau_{g}, \tau_\ell)$, and   $c_g$ guaranteed by Proposition~\ref{prop:g-RSC},   Proposition~\ref{prop:LRSS}, and   Corollary~\ref{Cor_upper_bound_cg} we are ready to prove Theorem~\ref{theorem:main}. 
 
    Specifically, 
 let \begin{align}\label{eq:global_RSC-RSM_high-prob}\begin{split}
  &  \mu_{\Sigma}=\frac{1}{2} \sigma_{\min}(\Sigma)\quad  \text{and} \quad \tau_{\mu}   =c_1 \zeta \frac{\log d}{N} \quad (\text{for the RSC)},\\ & L_{\Sigma}=  2 \sigma_{\max}(\Sigma)\quad \text{and} \quad   \tau_g   =c_1 \zeta \frac{\log d}{N}  \quad (\text{for the RSM)},\end{split}
\end{align} and \begin{align}\label{eq:locsl_RSM_high-prob}
  &  \ell_{\Sigma}=16 m  \sigma_{\max}(\Sigma)\quad  \text{and} \quad \tau_{\ell}   = c_1 \zeta m^2 \frac{\log d}{N}  \quad (\text{for the local RSM)};  
\end{align} 
Propositions~\ref{prop:g-RSC} \&~\ref{prop:LRSS} and   Corollary~\ref{Cor_upper_bound_cg} guarantee  that, with the above choices for  $(\mu_\Sigma,L_{\Sigma}, \ell_\Sigma)$ and  $(\tau_{\mu},\tau_{g}, \tau_\ell)$, the global RSC and RSM conditions (cf. Assumption~\ref{assump:G_RSM}), the local RSM property (cf. Assumption~\ref{assump:L_RSS}), and the bound (\ref{eq:c_g_prob_bound}) on $c_g$  all hold 
with probability at least $1 - c_8\exp(-c_9 \log d)$, where we invoked the union bound and used  $\log d \leq s \log d \leq c_5 \frac{\mu_\Sigma}{\zeta} N \leq c_5 N$, due to $\frac{\zeta}{\mu_{\Sigma}} \frac{s \log d}{N}< c_5$. 

The rest of the proof is to (i) show ~\eqref{eq:cond_rho_hig-prob-setting} implies \eqref{cond_mu_sigma}-\eqref{eq:rho_final_condition}   in Theorem~\ref{thm:convergence_deterministic_1} and (ii) simplify    the expressions of the  rate expression~\eqref{eq:rate_expression_final}, the   statistical error~\eqref{eq:Delta_stat_final}, and $\B$ in (\ref{eq:B_final}) using \eqref{eq:global_RSC-RSM_high-prob}-\eqref{eq:locsl_RSM_high-prob}. The calculation is left to~\ref{pf:thm-hp}.


 
 \section{Numerical Results}\label{sec:numerical_result}

This section provides some numerical results that validate  our theoretical findings. 

We consider the following problem setup. Given the  distributed linear regression model  \eqref{eq:I/O},  the unknown $s$-sparse vector  $\btheta^*$ is generated with the first $s$ elements being  i.i.d.  $\cN(0,1)$; the remaining coordinates are set to zero. Each row of $\bX_i \in \mathbb{R}^{n \times d}$ is generated according to $\cN(\mathbf{0},\mathbf{I})$; each element of $\mathbf{n}_i$ follows $\cN(0,0.25)$; and $r$ is set $r = \|\btheta^*\|_1.$
Unless otherwise specified, we   the communication network is modeled as  a   Erd\H{o}s-R\'{e}nyi graph $G(m,p)$, where $m$ is the number of nodes  and each edge is included in the graph  with probability $p$, independently  from the others. The specific values of problem parameters $(N,d,s)$ and network parameters $(m,p)$ are specified in each simulation.

Aiming at validating Theorem~\ref{theorem:main}, the reported   experiments show   the following: (i) the scalability of the algorithm with respect to the total sample size $N$ and problem dimension $d$, for a fixed network; and (ii) the impact of network topology on the algorithm. Finally, we validate our findings on real data sets. 

\subsection{Scalability with respect to problem parameters for a fixed network}\label{sec:prob-scalability} 
Theorem~\ref{theorem:main} states that, for a fixed, sufficiently connected network, the average optimization error of  DGT among all agents decreases  at a linear, up to a tolerance  depending on the statistical precision $\|\htheta - \btheta^*\|$. Moreover, the expression of the convergence rate and the tolerance term depend exclusively on the ratio $\alpha = \frac{s \log d}{N}$. To validate numerically these theoretical findings, we conduct the following two simulations over an Erd\H{o}s-R\'{e}nyi graph  with $m = 50$ nodes  and edge activation probability $p = 0.5$. The reported results are the average of 100  trials with the same graph and independently randomly generated data sets.

\smallskip

\noindent\textbf{(i)} \textbf{Fixed    $\alpha$ and increasing dimension $d$:} We set $\alpha \approx 0.2$ and consider three sets of problem parameters, namely:  $d = 5000, 10000,20000$ and $s = \lceil\sqrt{d}\rceil$. Consequently,   the number of total samples $N$ and local sample size $n$ is set to  $N = 3050,4700,7100$ and $n = 61,94,142$, respectively. Observe that $N = m \cdot n$ with $m = 50.$  Since the expression of the proximal parameter $\gamma$ may be  conservative in practice, we tune it manually as described next. We first generate one instance of  the LASSO problem with $d = 5000$ and run DGT using the $\gamma$ that yields the fastest convergence rate. This value is selected via  grid search. The same $\gamma$ is used for all problem instances.
Fig.~\ref{fig:rateinvariance} plots the average normalized estimation error $(1/m) \sum_{i = 1}^m \|\btheta_i^t - \btheta^*\|^2 \|\btheta^*\|^{-2}$ versus the number of iterations, for $d = 5000, 10000,20000$. The  dashed-line curves represent the normalized statistical precision $(\|\htheta - \btheta^*\|^2\|\btheta^*\|^{-2}$ for the three sets of problems. The figure shows that, in all cases, the estimation error  decays at a linear rate before reaching the statistical precision. Further,    the rate is invariant   to the change of  $d$ as long as the ratio $\alpha$ remains fixed. 
This is exactly what Theorem \ref{theorem:main} predicts.\smallskip 



\noindent \textbf{(ii)} \textbf{Increasing $\alpha$:} This experiment explores the relation between the convergence rate of  DGT and $\alpha$. We set $d = 20000$, $s= \left\lceil (\log d)/2 \right\rceil$ and consider three  local sample sizes,  $n = 1,5,$ and $25$. Correspondingly, we have  $\alpha \approx \{1,\, 0.2,\, 0.04\}$. For each value of $n$, we choose the proximal parameter $\gamma$ via grid search.  The one that yields fastest convergence rate is kept fixed for all independent trials.  Fig.~\ref{fig:alphavar_single_plot} shows that  the algorithm's performance degrades as $\alpha$ increases--both the convergence rate of DGT and the statistical precision get worse. Note that this behavior is in agreement with  Theorem \ref{theorem:main}.

\begin{figure*}[t!]
    \centering
    \begin{subfigure}[t]{0.5\textwidth}
        \centering
        \resizebox{7cm}{5cm}{\input{constant_alpha.tikz}}
        \caption{Log-normalized error for $\alpha \approx 0.2.$ In these problems $d \in \{5000,10000,20000\},$ $s = \lceil \sqrt{d} \rceil,$  $n \in \{61,94,142\}.$ \label{fig:rateinvariance}}
    \end{subfigure}%
    ~\
    \begin{subfigure}[t]{0.5\textwidth}
        \centering
        \resizebox{7cm}{5cm}{\input{changing_alpha.tikz}}
        \caption{Log-normalized error for  different values of $\alpha$. In these problems $d = 20000,$ $s = \lceil{\log(d)/2}\rceil,$ $m=50$ and  $\alpha \in \{0.04,0.2,1\}.$ \label{fig:alphavar_single_plot}}
    \end{subfigure}
    \caption{Scalability of DGT for fixed networks.  The network is composed of $m = 50$ agents with connectivity $\rho = 0.0638$. Dashed lines indicate estimation error if using $\htheta$.}\vspace{-0.2cm}
\end{figure*}


\subsection{Impact of network scale and connectivity}\label{sec:net-scalability} 
This set of simulations investigates the scalability of DGT with respect to network parameters $(m, \rho)$, for fixed problem parameters $(N,d,s)$. In particular, 
Theorem~\ref{theorem:main} established the convergence of DGT under condition $\rho \lesssim m^{-8} \kappa^{-4}$, which requires better connectivity for a network involving more agents. To test the necessity of this condition, we  run DGT on  networks with increasing size $m = 50, 625, 1250, 2500$. Correspondingly  we varied the edge connectivity probability $p = 0.87, 0.4, 0.23, 0.15$ in order to keep   $\rho$ roughly constant, equal to  $\approx 0.18$.
The problem parameters are set to be $N = 2500$, $d = 5000$, $s = \lceil \sqrt{d} \rceil$. The step size is chosen via grid search and set to be the one yielding the fastest convergence. All curves reported for this simulation in one trial are conducted  on the same problem instance for different network configurations. The result is averaged over $100$ trials and is reported in Fig~\ref{fig:impossibility}. Observe that  keeping the graph's connectivity constant while increasing the number of agents leads to a degradation of the achievable    rate and estimation error.

Additionally, we empirically demonstrate that for Erd\H{o}s-R\'{e}nyi graphs with connectivity  $\rho = \mathcal{O}(1/\sqrt{m}))$--achieved  by fixing the network connectivity probability $p$--DGT converges  under a growing network size. The result  is reported in Fig.~\ref{fig:empirical_rho}. The simulation is conducted under the same setting as Fig.~\ref{fig:impossibility} except for  $p$   set to be $0.87$. 
The simulation suggests that the convergence condition $\rho \lesssim m^{-8} \kappa^{-4}$ might be  more restrictive than what is needed in practice, which could be  an artifact of our  analysis. 

\begin{figure}
\centering
\begin{subfigure}{.5\textwidth}
  \centering
  \resizebox{7cm}{5cm}{\input{fix_connectivity.tikz}}
  \caption{ Erd\H{o}s-R\'{e}nyi graphs with $\rho \approx 0.18.$   \label{fig:impossibility}}
\end{subfigure}%
\begin{subfigure}{.5\textwidth}
  \centering
  \resizebox{7cm}{5cm}{\input{correct_scaling.tikz}}
  \caption{ Erd\H{o}s-R\'{e}nyi graphs with link probability 0.87.}\label{fig:empirical_rho}
\end{subfigure}
\caption{Plot of the log-normalized estimation error for $m \in \{50, 625, 1250, 2500\}.$ In this problem $d = 5000,$ $s = \lceil \sqrt{d} \rceil,$ $n \in \{50,4,2,1\}$ and $N = 2500.$ All curves in each panel are obtained using the same data points. Black horizontal lines indicate the normalized estimation error if using $\htheta.$ }
\label{fig:net_study}\vspace{-0.6cm}
\end{figure}

{}{To conclude the assessment of the algorithm performance as function of the network connectivity parameter $\rho$, in  Table~\ref{table:scalability}, we report  the number of communication rounds required per iteration to satisfy the condition $\rho^K \leq m^{-8}.$ Observe that the number remains approximately constant for Erd\H{o}s-R\'{enyi} graphs while increasing dramatically for   line-graphs. Similarly, in Fig.~\ref{fig:twograph}, we display the number of total communication rounds required to solve a problem instance with $d = 5000,$ $s = \lceil\sqrt{d}\rceil,$ $m \in \{50,625,1250,2500\}$, and $n \in \{50,4,2,1\}.$ The results are consistent with those displayed in Table~\ref{table:scalability} and demonstrate that the quantity $\rho$ is predictive of the number of communication rounds required for convergence.}  \vspace{-0.5cm}

\begin{figure}[t!]
    \centering
    \begin{tabular}[width = 0.9\linewidth]{|c|c|c|}
\hline
     &  Erd\H{o}s-R\'{e}nyi ($p = 0.87$) & Line graph \\ \hline
     m = 50 & 18 &  23775  \\ 
     m = 625 &18 & 6115118\\
     m = 1250 & 18 & 27094153 \\
     m = 2500 & 19 & 118911225 \\
\hline
\end{tabular}
    \caption{Number of communication rounds required for each network to satisfy $\rho^k \leq m^{-8}$.}
    \label{table:scalability}
\end{figure}

\begin{figure}[t!]
\centering
\begin{subfigure}{.49\textwidth}
\centering
\scalebox{0.5}{\input{random.tikz}}
\caption{\label{fig:ergos_scalability}Erd\H{o}s-R\'{e}nyi graphs.}
\end{subfigure}
\begin{subfigure}{0.49\textwidth}%
\centering
\scalebox{0.5}{\input{line.tikz}}
\caption{\label{fig:line_scalability} Line graphs.}
\end{subfigure}
\caption{\label{fig:twograph} Plots of the log-normalized estimation error   for  Erd\H{o}s-R\'{e}nyi and line graphs, with $m \in \{50,625,1250,2500\}$. In this problem, $d = 5000,$ $s = \lceil \sqrt{d} \rceil,$ $n \in \{50,4,2,1\}$ and $N = 2500.$}\vspace{-0.4cm}
\end{figure}

{}{\subsection{Comparison with other distributed algorithms}
We test the performance of DGT when compared to other decentralized algorithms with known guarantees in the high-dimensional regime. Specifically,  we consider DGD-CTA \citep{Ji-DGD21}, DGD-ATC \citep{ji2023distributed}, and DGD$^2$ proposed in our companion work \citep{maros2022dgd}.  DGD-ATC, and DGD$^2$ are known to achieve centralized statistical precision at a rate  comparable that of the   centralized PGD, provided the network is sufficiently connected. Conversely, DGD-CTA tends to slow down when aiming for centralized statistical precision. Given that, except for DGD-CTA, all algorithms can potentially perform equivalently in terms of iterations--contingent on meeting specific network connectivity requirements--we focus on the number of communication rounds required to reach centralized statistical precision rather than just the total iteration count.  For this analysis, we create problem instances with  $d \in \{400,2000,4000,20000\},$ $s \in \{5,4,7,4\},$ and $N \in \{240,240,480,360\}.$ This yields $\alpha \approx 0.11$ for all instances. We solve each problem instance up to statistical precision using the centralized PGD and record the total number of performed iterations $T_{\mathrm{cent}}$ for each problem. In the decentralized setting, we generate a base mixing matrix $W$ with $m = 120$ and $p = 0.02,$ which results in a  poorly connected  base mixing matrix. For DGD-ATC, DGD$^2$ and DGT we select  a mixing matrix $\bar{W} = W^K$ with the smallest possible $K$ that allows the algorithms to achieve centralized statistical precision in at most $2T_{\mathrm{cent}}$ iterations. For DGD-ATC, DGD$^2$ and DGT we measure the total   communication rounds as iterations $\times$ K. For DGD-CTA we set the network to be fully connected with the largest possible step-size that yields convergence to the order of centralized statistical precision and record the total number of network uses, which in this case also corresponds to the number of iterations. The results of this experiment are reported in Fig.~\ref{fig:comm_cost}.}
\begin{figure}
\begin{center}
\includegraphics[width = 0.6\textwidth]{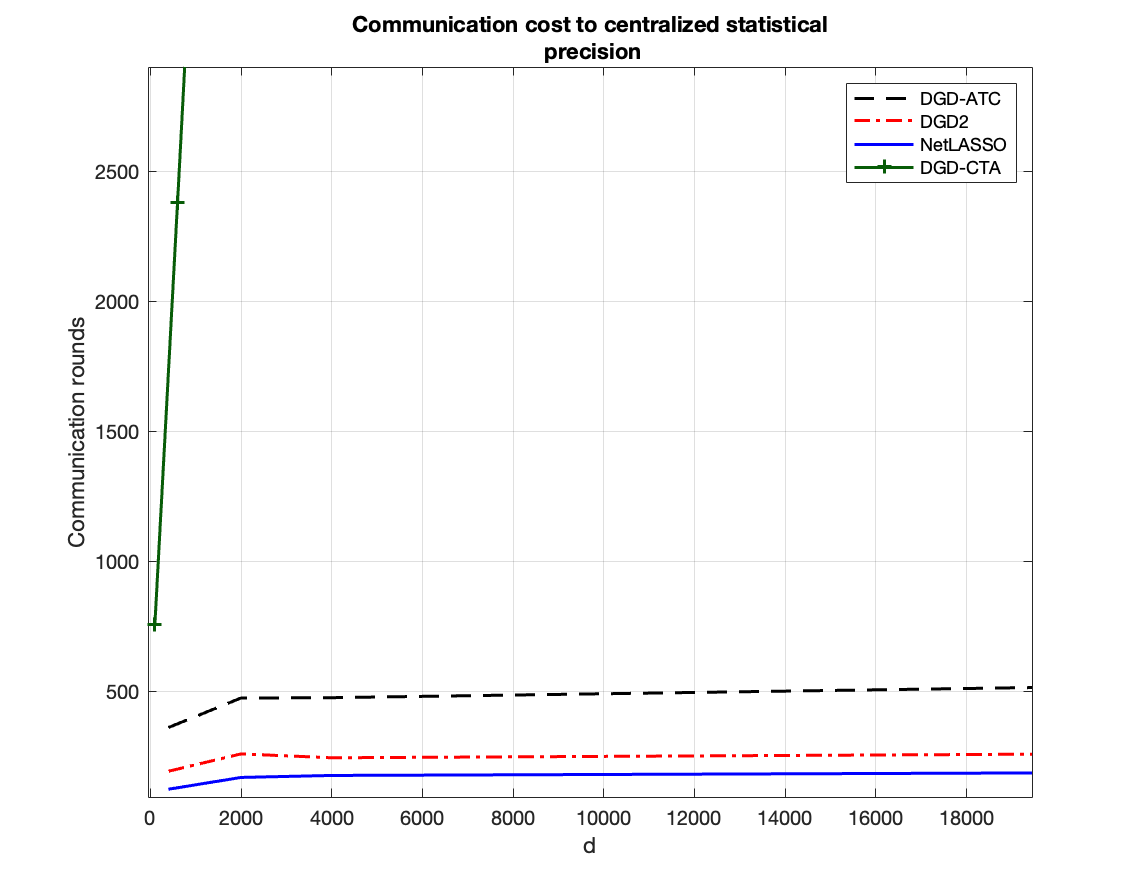}
\end{center}
\caption{\label{fig:comm_cost} Number of communication rounds for $m = 120,$ $d \in \{400,2000,4000,20000\},$ $s \in \{5,4,7,4\},$ and $N \in \{240, 240, 480, 360\}.$ All curves are obtained using the same data points and consist of 20 averages each.}\vspace{-0.7cm}
\end{figure}
{}{We observe that despite employing a fully connected network, DGD-CTA requires the most communication rounds of all schemes. This is because the network connectivity does not significantly aid the performance of DGD-CTA. Meanwhile, we observe that for DGD-ATC the number of communication rounds   increases moderately with the dimension $d$ while for DGD$^2$ and DGT the number of communication rounds remains invariant as the problem dimension increases. The statistical error achieved by   DGD$^2$ and DGT while of similar order is beneficial towards DGT.}

\subsection{Simulations on real data}
We test the DGT on the \texttt{Communities and Crime} data set (UC Irvine Machine Learning Repository) where we have removed data points with missing attributes (covariates), and removed attributes corresponding to community number/name or zip-code (attributes 1 to 4 in the data set) yielding a regression problem with $d = 123$ and total sample size of 123 which we have split into $N_{\mathrm{train}} = 82$ and $N_{\mathrm{test} = 41}.$ For the decentralized implementation, we have a network of $m=41$ agents corresponding to $n=2$ samples per agent. The agents communicate via a network built upon an Erd\H{o}s-R\'{e}nyi where the probability of an edge being present is set to $p = 0.5.$ Because the resulting graph has connectivity $\bar{\rho} = 0.5236$ we have agents perform three communication rounds per gradient computation yielding an effective connectivity $\rho = 0.1435 \approx \frac{1}{\sqrt{m}}$ which corresponds to the connectivity that is sufficient to yield centralized performance according to numerical simulations (c.f. Figure~\ref{fig:empirical_rho}). Observe that our theoretical results yield more conservative values for $\rho$; consequently, employing the connectivity predicted by the theory can only yield improved results. Our results are reported in Figure~\ref{fig:real_data} where the performance of the centralized projected gradient descent (PGD) is reported to the left, and the performance of the DGT is reported on the right panel. Denote by $(X_{\mathrm{train}},y_{\mathrm{train}})$ the data pairs included in the training data set, and analogously by $(X_{\mathrm{test}},y_{\mathrm{test}})$ the data pairs included in the testing data set. The training error is measured as $\frac{1}{2N_{\mathrm{train}}}\|X_{\mathrm{train}}\theta^t-y_{\mathrm{train}}\|^2_2$ where $\theta^t$ are iterates of PGD generated by minimizing the function $\frac{1}{2N_{\mathrm{train}}}\|X_{\mathrm{train}}\theta-y_{\mathrm{train}}\|^2$ over the set $\|\theta\|_1 \leq 0.85$ using step-size 0.09. The testing error is measured by $\frac{1}{2N_{\mathrm{test}}}\|X_{\mathrm{test}}\theta^t - y_{\mathrm{test}}\|^2.$ The right panel reports the performance of DGT with step-size set to be $0.05$. The training and testing errors are measured by 
$\frac{1}{m}\sum_{i=1}^m\frac{1}{2N_{train}}\|X_{\mathrm{train}}\theta_i^t-y_{\mathrm{train}}\|^2$ and $\frac{1}{m} \sum_{i=1}^m \|X_{\mathrm{test}}\theta_i^t - y_{\mathrm{test}}\|^2$,  respectively. {}{We report in the right panel the difference between the centralized and decentralized train and test errors respectively.} We conclude that the performance of centralized and decentralized schemes is practically identical (thanks to multiple but finite consensus rounds), both in training and testing errors.

\begin{figure}
\centering
\includegraphics[width = \textwidth]{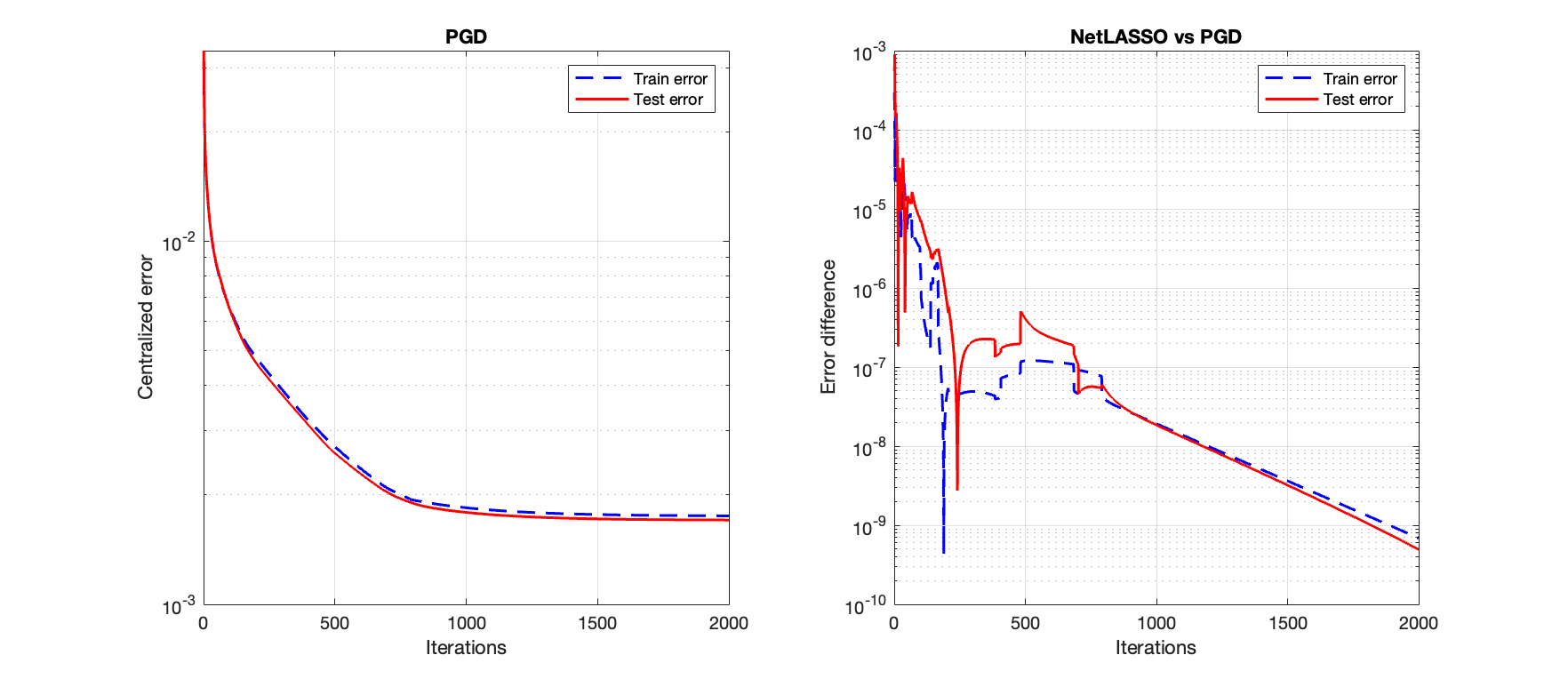}
\caption{\label{fig:real_data}Plots of train and test error on the \texttt{Communities and Crime} data set for centralized projected gradient descent (PGD) (left) and DGT (right). }
\end{figure}

    


\section{Conclusion}\label{sec:conclusions}

We studied sparse linear regression over mesh networks. We established statistical and computational guarantees in  high-dimensions for DGT, a decentralization of the PGD based on gradient tracking.  We proved that (near) optimal sample complexity  $\mathcal{O}(s\log d/N)$  for the distributed estimator is achievable over networks, even when the local sample size is not sufficient for statistical consistency. Convergence at a linear rate was proved, which is of the same order as that of PGD. This improves on   DGD-like methods \citep{Ji-DGD21,ji2023distributed}, which instead suffer from the speed-accuracy dilemma. 

Our work, together with the study of DGD~\citep{Ji-DGD21,ji2023distributed} and DGD$^2$~\citep{maros2022dgd}, reveal the (non)-convergence of decentralized algorithms is significantly affected by the gradient estimator used in the local optimization step. A future investigation will be the study of other distributed methods; of particular interest is understanding the role of other   forms of gradient corrections towards statistical, computation, and communication tradeoffs.

\appendix

\section{Proof of Proposition~\ref{eq:inexact_opt}}\label{pf:inexact_opt}
First note that the optimality of $\btheta_i^{t+\frac{1}{2}}$ implies:
\begin{equation}\label{eq:FOC_loc_opt_regularize}
	\Big(\btheta - \btheta_i^{t+\frac{1}{2}} \Big)^\top \left(\bg_i^t + \gamma \Delta \btheta_i^t \right) \geq 0, \quad \forall \ \|\btheta\|_1 \leq r,
\end{equation}
where we recall $\Delta \btheta_i^t \triangleq \btheta_i^{t + \frac{1}{2}} - \btheta_i^t$. 

We use now the global RSC/RSM property of $\cL$ (cf. Assumption~\ref{assump:G_RSM}) to compute the reduction of optimality gap $ \cL (\btheta) - \cL (\htheta)$ by executing step~\eqref{eq:loc_opt}. 

\begin{align*}
	\begin{split}
		\cL (\btheta_i^{t+\frac{1}{2}})& 
		=  \cL (\btheta_i^t)  +   \nabla \cL (\btheta_i^t)^\top (\btheta_i^{t + \frac{1}{2}} - \btheta_i^{t}) + \frac{1}{2 } (\btheta_i^{t + \frac{1}{2}} - \btheta_i^{t})^\top \left(\frac{\bX^\top \bX}{N}\right) (\btheta_i^{t + \frac{1}{2}} - \btheta_i^{t}) \\
		\stackrel{(RSM)}{\leq} & \cL (\btheta_i^t)  +   \nabla \cL (\btheta_i^t)^\top (\btheta_i^{t + \frac{1}{2}} - \btheta_i^{t}) + \frac{L_\Sigma}{2} \| \btheta_i^{t + \frac{1}{2}} - \btheta_i^{t} \|^2 +  \frac{\tau_g}{2}\| \btheta_i^{t + \frac{1}{2}} - \btheta_i^{t} \|_1^2\\
		= &\cL (\htheta) + \nabla \cL (\btheta_i^t)^\top (\btheta_i^{t + \frac{1}{2}} - \htheta) - \frac{1}{2} ( \btheta_i^{t} - \htheta )^\top \left( \frac{\bX^\top \bX}{N} \right) ( \btheta_i^{t} - \htheta ) \\
		 & +  \frac{L_\Sigma}{2} \| \btheta_i^{t + \frac{1}{2}} - \btheta_i^{t} \|^2 +  \frac{\tau_g}{2} \| \btheta_i^{t + \frac{1}{2}} - \btheta_i^{t} \|_1^2 \\
		\stackrel{(RSC)}{\leq}  & \cL (\htheta) + \nabla \cL (\btheta_i^t)^\top (\btheta_i^{t + \frac{1}{2}} - \htheta) - \frac{\mu_\Sigma}{2} \| \btheta_i^{t} - \htheta\|^2  +  \frac{L_\Sigma}{2} \| \btheta_i^{t + \frac{1}{2}} - \btheta_i^{t} \|^2  \\
		& + \frac{\tau_\mu}{2} \| \btheta_i^{t} - \htheta\|_1^2 +  \frac{\tau_g}{2}\| \btheta_i^{t + \frac{1}{2}} - \btheta_i^{t} \|_1^2.
	\end{split}
\end{align*}
Letting $\btheta = \htheta$ in~\eqref{eq:FOC_loc_opt_regularize} and adding to the above inequality we have
\begin{align*}
	\begin{split}
		\cL (\btheta_i^{t+\frac{1}{2}})  
		\leq & \cL (\htheta) - \frac{\gamma}{2} \left( \|\btheta_i^{t+\frac{1}{2}} -  \htheta \|^2 - \|\btheta_i^{t} -  \htheta\|^2 + \| \btheta_i^{t + \frac{1}{2}} - \btheta_i^{t}\|^2\right)+ ( \nabla \cL (\btheta_i^t) - \bg_i^t)^\top( \btheta_i^{t + \frac{1}{2}} - \htheta)\\
		& - \frac{\mu_\Sigma}{2} \| \btheta_i^{t} - \htheta\|^2  +  \frac{L_\Sigma}{2} \| \btheta_i^{t + \frac{1}{2}} - \btheta_i^{t} \|^2   + \frac{\tau_\mu}{2} \| \btheta_i^{t} - \htheta\|_1^2  + \frac{\tau_g}{2} \| \btheta_i^{t + \frac{1}{2}} - \btheta_i^{t} \|_1^2.
	\end{split}
\end{align*}
Now applying the triangle inequality to the last term and Lemma~\ref{lem:norm_bound} yields:
\begin{align*}
		& \cL (\btheta_i^{t+\frac{1}{2}}) \\
		\leq & \cL (\htheta) - \frac{\gamma}{2} \left( \|\btheta_i^{t+\frac{1}{2}} -  \htheta \|^2 - \|\btheta_i^{t} -  \htheta\|^2 + \| \btheta_i^{t + \frac{1}{2}} - \btheta_i^{t}\|^2\right)+ ( \nabla \cL (\btheta_i^t) - \bg_i^t)^\top( \btheta_i^{t + \frac{1}{2}} - \htheta )\\
		& - \frac{\mu_\Sigma}{2} \| \btheta_i^{t} - \htheta\|^2  +  \frac{L_\Sigma}{2} \| \btheta_i^{t + \frac{1}{2}} - \btheta_i^{t} \|^2  + \frac{\tau_\mu}{2} (8 s \| \btheta_i^{t} - \htheta\|^2 + 2 \nu^2)\\
		& 
		+ \frac{\tau_g}{2} (16 s \| \btheta_i^{t + \frac{1}{2}} - \htheta \|^2  + 16s \|  \btheta_i^{t} - \htheta\|^2 + 8 \nu^2)\\
		\leq & \cL (\htheta) - \frac{\gamma}{2} \left( \|\btheta_i^{t+\frac{1}{2}} -  \htheta \|^2 - \|\btheta_i^{t} -  \htheta\|^2 + \| \btheta_i^{t + \frac{1}{2}} - \btheta_i^{t}\|^2\right)+ ( \nabla \cL (\btheta_i^t) - \bg_i^t)^\top( \btheta_i^{t + \frac{1}{2}} - \htheta )\\
		& - \frac{\mu_\Sigma}{2} \| \btheta_i^{t} - \htheta\|^2  +  \frac{L_\Sigma}{2} \| \btheta_i^{t + \frac{1}{2}} - \btheta_i^{t} \|^2   + 8 s \tau_g \| \btheta_i^{t + \frac{1}{2}} - \htheta \|^2  + (4 s \tau_\mu + 8 s \tau_g) \| \btheta_i^{t} - \htheta\|^2 \\
		& + ( \tau_\mu + 4\tau_g )\nu^2.
\end{align*}
Rearranging terms and using the fact that $\cL (\btheta_i^{t+\frac{1}{2}}) \geq \cL (\htheta)$, yields
\begin{multline}\label{eq:decent_w_err}
		\left( 1 -\frac{C_1 s \tau_g}{\gamma}  \right)  \| \btheta_i^{t + \frac{1}{2}} - \htheta \|^2 \leq  \left( 1 - \frac{\mu_\Sigma}{\gamma} +  \frac{C_1 s (\tau_\mu + \tau_g)}{\gamma} \right) \| \btheta_i^{t} - \htheta\|^2 \\
	 - \left( 1 - \frac{L_\Sigma}{\gamma}\right)\| \btheta_i^{t + \frac{1}{2}} - \btheta_i^{t} \|^2 + \frac{2}{\gamma} \cdot ( \nabla \cL (\btheta_i^t) - \bg_i^t)^\top( \btheta_i^{t + \frac{1}{2}} - \htheta ) 
	 +  \frac{C_1 (\tau_\mu + \tau_g)}{\gamma} \nu^2.
	\end{multline}
Summing over 	$i=1,\ldots,m$, 
	completes the proof. \hfill $\square$

\section{Proof of Proposition~\ref{prop:track_error}}~\label{sec:pf_track_error}
The difference $\frac{1}{m} \sum_{j=1}^m \nabla \cL_j(\btheta_i^t ) - \frac{1}{m} \sum_{j=1}^m \nabla \cL_j (\btheta_j^t)$ can be  expressed  as
	\begin{align*}
	\begin{split}
	&\frac{1}{m} \sum_{j=1}^m \nabla \cL_j(\btheta_i^t ) - \frac{1}{m} \sum_{j=1}^m \nabla \cL_j (\btheta_j^t)
	=  \frac{1}{N} \bX^\top \bX (\btheta_i^t  - \bar{\btheta}^t) + \frac{1}{m} \sum_{j=1}^m \frac{1}{n} \bX_j^\top \bX_j   (\bar{\btheta}^t - \btheta_j^t).
	\end{split}
	\end{align*}
		Substituting into the expression of $\delta^t$ we can split it into three error terms:
	
	\begin{align}\label{eq:delta_3_terms}
	\begin{split}
	\delta^t 
	= 
	&  \underbrace{\sum_{i=1}^m   ({\btheta}_i^{t +\frac{1}{2}} - \widehat{\btheta})^\top \Big(\frac{1}{N} \bX^\top \bX \Big) (\btheta_i^t  - \bar{\btheta}^t)}_{T_1}  + \underbrace{\frac{1}{m} \sum_{i=1}^m \sum_{j=1}^m ({\btheta}_i^{t +\frac{1}{2}} - \widehat{\btheta})^\top \Big( \frac{1}{n} \bX_j^\top \bX_j  \Big) (\bar{\btheta}^t - \btheta_j^t)}_{T_2}\\
	&+ \underbrace{ \sum_{i=1}^m \left(\bar{\bg}^t- \bg_i^t \right)^\top ({\btheta}_i^{t + \frac{1}{2}} - \widehat{\btheta})}_{T_3}.
	\end{split}
	\end{align} 
	We are now ready to bound $\delta^t$. For this, we  substitute the expressions of   $\bTheta_\bot^t$ and $\bG_\bot^t$ [cf. (\ref{eq:update_theta_perp})-(\ref{eq:update_err_y})] into~\eqref{eq:delta_3_terms}. We then use the RSM (cf. Assumption \ref{assump:G_RSM}) to upper bound $\delta^{t}$ by quantities that depend exclusively on the sequences  $\|\bTheta^{t+\frac{1}{2}} - \hTheta\|$ and $\|\Delta \bTheta^t\|.$
	
	\textbf{1) Bounding $T_1$:}  Iterating (\ref{eq:update_theta_perp}) telescopically to $t=0$, yields
	\begin{equation}\bTheta_{\perp}^{t}    = (\bW - \mathbf{J})\bTheta^0_{\perp} + \sum_{s=0}^{t-1}(\bW - \mathbf{J})^{t-s}\Delta \bTheta^s,\quad t=1,2,\ldots, .\label{eq:update_err_theta}\end{equation}
	This shows that $\bTheta_{\perp}^{t}$ lies in the span of $\{\bTheta_{\perp}^{0},\Delta\bTheta^{0},\ldots ,\Delta\bTheta^{t-1}\}$. Therefore, 
\begin{equation}
\label{eq:update_err_y_vector}
\btheta^t_i - \bar{\btheta}^t = \sum_{j=1}^m \ell_{ij}^{(t)}(\btheta_{j,\perp}^{0}) + \sum_{j=1}^m\sum_{s=0}^{t-1}\ell_{ij}^{(t-s)}\Delta \btheta^s_{j}, \quad ,\quad t=1,2,\ldots, \end{equation}
where $\ell_{ij}^{(t)}$ denotes the $ij-$th element of the matrix $(\bW - \mathbf{J})^t$. 
	 
	Substituting \eqref{eq:update_err_y_vector} in the expression of $T_1$, we can bound it as
	\begin{align}\label{eq:T_1_bound}
	\begin{split}
	T_ 1 
	= &\sum_{i=1}^m ({\btheta}_i^{t +\frac{1}{2}} - \widehat{\btheta})^\top \Big(\frac{1}{N} \bX^\top \bX \Big) \left( \sum_{j=1}^{m} \ell_{ij}^{(t)} \btheta_{j,\bot}^0 + \sum_{j=1}^{m}  \sum_{s=0}^{t-1} \ell_{ij}^{(t-s)} \Delta \btheta_j^s\right)\\
	\leq & \sum_{i = 1}^{m}  \sum_{j = 1}^{m}  |\ell_{ij}^{(t)}|\left\|\frac{1}{\sqrt{N}}\bX (\btheta_i^{t+\frac{1}{2}} - \htheta) \right\| \left\| \frac{1}{\sqrt{N}}\bX \btheta_{j,\bot}^0\right\| \\
	& + \sum_{i = 1}^{m}  \sum_{j = 1}^{m} \sum_{s=0}^{t-1}  |\ell_{ij}^{(t-s)}|\left\| \frac{1}{\sqrt{N}}\bX (\btheta_i^{t+\frac{1}{2}} - \htheta) \right\| \left\| \frac{1}{\sqrt{N}}\bX \Delta\btheta_{j}^s\right\|\\
	\leq &\sum_{i = 1}^{m}  \sum_{j = 1}^{m}  |\ell_{ij}^{(t)}| \left( \frac{\epsilon}{2} \norm{\frac{1}{\sqrt{N}}\bX (\btheta_i^{t+\frac{1}{2}} - \htheta) }^2 + \frac{1}{2 \epsilon} \norm{\frac{1}{\sqrt{N}}\bX \btheta_{j,\bot}^0}^2\right)\\
	& + \sum_{i = 1}^{m}  \sum_{j = 1}^{m} \sum_{s=0}^{t-1}  |\ell_{ij}^{(t-s)}| \left( \frac{\epsilon}{2} \norm{\frac{1}{\sqrt{N}}\bX (\btheta_i^{t+\frac{1}{2}} - \htheta)}^2 + \frac{1}{2 \epsilon} \norm{\frac{1}{\sqrt{N}}\bX \Delta\btheta_{j}^s}^2\right),
	\end{split}
	\end{align}
	where we have applied Cauchy-Schwartz and Young's inequality with $\epsilon >0$.
	Invoking the RSM (Assumption~\ref{assump:G_RSM}) along with Assumption \ref{assump:weight} and  Lemma~\ref{lem:norm_bound}, we can bound the RHS of (\ref{eq:T_1_bound}) as (cf. ~\ref{app:bound_T1}): for all $t \geq 1$, 
	\begin{align*}
T_1 \leq \ & \frac{  \rho \cdot c_m \epsilon}{1 - \rho} \left(L_\Sigma + 8 s \tau_g\right) \norm{\bTheta^{t+\frac{1}{2}} - \hTheta}_F^2  + \frac{ c_m L_\Sigma}{2 \epsilon}\sum_{s=0}^{t-1} \rho^{t-s}  \| \Delta \bTheta^s\|_F^2 \notag\\
& + c_m   \cdot 8 s \tau_g \epsilon^{-1} \sum_{s=0}^{t-1}\rho^{t-s}  \| \bTheta^{s+\frac{1}{2}} - \hTheta\|_F^2 
+ c_m \cdot 8 s \tau_g \epsilon^{-1} \sum_{s=0}^{t-1} \rho^{t-s}    \| \bTheta^{s} - \hTheta\|_F^2 \notag\\
& + \frac{  \rho \cdot m c_m }{1 - \rho} \cdot 4 \tau_g \nu^2 (\epsilon + \epsilon^{-1}) +  \underbrace{\frac{1}{2 \epsilon}  c_m \rho^t \left(  L_\Sigma \|  \bTheta_\bot^0  \|_F^2 + \tau_g \sum_{j = 1}^m \| \btheta_{j}^0 - \bar{\btheta}^0 \|_1^2 \right)}_{=0 \text{ by initialization}}.
	\end{align*}
	
	\textbf{2) Bounding $T_2$:} 
	Recall that
	\begin{align*}
	T_2 & =  \frac{1}{m } \sum_{i=1}^m \sum_{j=1}^m \underbrace{({\btheta}_i^{t +\frac{1}{2}} - \widehat{\btheta})^\top \Big( \frac{1}{n} \bX_j^\top \bX_j  \Big) (\bar{\btheta}^t - \btheta_j^t)}_{\triangleq T_2^{ij}}.
	\end{align*}
		Focusing on the $ij$-th summand and using~\eqref{eq:update_err_y_vector} we can write 
	\begin{align}\label{eq:T_2ij}
	T_2^{ij} & 
	  = ({\btheta}_i^{t +\frac{1}{2}} - \widehat{\btheta})^\top \Big( \frac{1}{n} \bX_j^\top \bX_j  \Big) \left(\sum_{k=1}^m \big(-\ell_{jk}^{(t)} \big)\btheta_{k,\bot}^0  + \sum_{k=1}^m \sum_{s=0}^{t-1} \big(-\ell_{jk}^{(t-s)}\big) \Delta \btheta_k^s\right).
	\end{align}
	Note that $T_2^{ij}$ has the same form as $T_1$ and thus following similar steps  as in \eqref{eq:bound_T1} while using the local RSM (Assumption \ref{assump:L_RSS}),  we can bound $T_2$ as (cf.~\ref{app:bound_T2}): for all $t \geq 1$,   
	\begin{align}\label{eq:T2_bound}
\begin{split}
T_2 \ \leq &  \frac{ \rho \cdot  c_m   \epsilon}{1 - \rho} \left( \ell_\Sigma + 8 s \tau_\ell \right) \norm{\bTheta^{t+\frac{1}{2}} - \hTheta}_F^2   +  \frac{ c_m \ell_\Sigma}{2\epsilon}  \sum_{s=0}^{t-1}  \rho^{t-s}  \| \Delta \bTheta^s\|_F^2 \\
& +  c_m 8 s \tau_\ell \epsilon^{-1} \sum_{s=0}^{t-1}  \rho^{t-s}  \| \bTheta^{s+\frac{1}{2}} - \hTheta\|_F^2
+ c_m 8 s \tau_\ell \epsilon^{-1} \sum_{s=0}^{t-1}  \rho^{t-s}  \| \bTheta^{s} - \hTheta\|_F^2 \\
& +  \frac{ \rho \cdot m c_m}{ 1 - \rho} 4  \tau_\ell \nu^2 \cdot (\epsilon +  \epsilon^{-1}) + \underbrace{\frac{1}{2 \epsilon}  c_m \rho^t \left(  \ell_\Sigma \|\bTheta_{\bot}^0\|_F^2 + \tau_\ell \sum_{k =1}^m \| \btheta_k^0 - \bar{\btheta}^0\|_1^2\right)}_{ = 0}.
\end{split}
	\end{align}

	\textbf{3) Bounding $T_3$:}
	Using \eqref{eq:update_err_y} we can express $\big({\btheta}_i^{t + \frac{1}{2}} - \widehat{\btheta}\big)^\top\left(\bar{\bg}^t- \bg_i^t \right) $ in $T_3$ as
	\begin{align*}
	\begin{split}
	& \big({\btheta}_i^{t + \frac{1}{2}} - \widehat{\btheta}\big)^\top\left(\bar{\bg}^t- \bg_i^t \right) \\
	= &  \left({\btheta}_i^{t + \frac{1}{2}}  - \widehat{\btheta} \right)^\top \left(\sum_{j=1}^m \big(- \ell_{ij}^{(t)}\big) \bg_{j,\bot}^0 + \sum_{s=0}^{t-1} \sum_{j=1}^m \big(- \ell_{ij}^{(t-s)}\big) \Big( \frac{1}{n}\bX_j^\top \bX_j \Big) (\btheta_j^{s+1} - \btheta_j^{s}) \right)\\
	\leq &  \sum_{j=1}^m \vert \ell_{ij}^{(t)} \vert \underbrace{ \left\vert \left({\btheta}_i^{t + \frac{1}{2}} - \widehat{\btheta} \right)^\top\bg_{j,\bot}^0 \right\vert}_{T_{3,1}} + \sum_{s=0}^{t-1} \sum_{j=1}^m \vert \ell_{ij}^{(t-s)} \vert \underbrace{\left\vert  \left({\btheta}_i^{t + \frac{1}{2}} - \widehat{\btheta} \right)^\top \Big( \frac{1}{n}\bX_j^\top \bX_j \Big)  (\btheta_j^{s+1} - \btheta_j^{s})  \right\vert}_{T_{3,2}}.
	\end{split}
	\end{align*}
	Following again  similar steps as in the bounds of   $T_1$ and $T_2$, we can obtain the following bound for  $T_3$ (cf.~\ref{app:bound_T3}): 
	\begin{align*}
    \begin{split}
    	T_3 \leq & \frac{3}{2 } \cdot \frac{\rho \cdot c_m \epsilon}{1 - \rho} \left( 2 \ell_\Sigma + 8 s \tau_\ell  \right) \| {\bTheta}^{t + \frac{1}{2}} - \widehat{\bTheta} \|_F^2 + \frac{c_m \ell_\Sigma}{2 \epsilon} \sum_{s=0}^{t-1} \rho^{t-s} \|\bTheta^{s+1} - \bTheta^s\|_F^2\\
    	& +c_m 8 s \tau_\ell \epsilon^{-1} \sum_{s=0}^{t-1}  \rho^{t-s}  \|\bTheta^{s+1} - \hTheta\|_F^2
    	+ c_m 8 s \tau_\ell \epsilon^{-1} \sum_{s=0}^{t-1}  \rho^{t-s}  \|\bTheta^{s} - \hTheta\|_F^2\\
    	& + m c_m \rho^t \left( A_3  \epsilon^{-1}  + s c_g^2 \ell_{\Sigma}^{-1} \epsilon^{-1} + c_g \nu \right) + \frac{\rho \cdot m c_m }{1 - \rho} 4 \tau_\ell\nu^2 (\epsilon + \epsilon^{-1}),
    \end{split}
	\end{align*}
	where 
	\begin{equation}
	\begin{array}{ll}
		A_3  &\triangleq  \frac{1}{2}\left( \ell_\Sigma \| \btheta^* \|^2 + \tau_\ell\|\btheta^* \|_1^2 + \frac{1}{m} \ell_\Sigma\|  \bTheta^*\|_F^2 + \tau_\ell\| \btheta^*\|_1^2\right),\medskip \\
c_g & \triangleq \|\nabla \cL_j (\btheta^*)\|_\infty + \| \nabla \cL (\btheta^*) \|_\infty.
	\end{array}
	\end{equation}

	\textbf{4) Combining the bounds of $T_1$, $T_2$, and $T_3$:} Adding up the bounds of $T_1, T_2$,  and $T_3$, and using the following two inequalities 
\begin{equation}\label{eq:bound_tetha_plus_thata_half}
\| \bTheta^{t+1} - \widehat{\bTheta} \|_F^2   = \| \mathbf{W}\big(\bTheta^{t+\frac{1}{2}} - \widehat{\bTheta}\big) \|_F^2  \leq \| \bTheta^{t+\frac{1}{2}} - \widehat{\bTheta} \|_F^2\end{equation}
and
	\begin{align}\label{eq:bound_Delta_theta}
	\| \bTheta^{s+1} - \bTheta^s \|_F^2 & \stackrel{\eqref{eq:update_theta_matrix}}{=} \|  \bW \big(  \bTheta^{s} + \Delta \bTheta^s \big)  - \bTheta^s  \|_F^2  \leq 4 \| \bTheta_\bot^s\|_F^2 + 2 \|  \Delta \bTheta^s \|_F^2,
	\end{align} 
 yield the desired bound, detailed calculations are provided in~\ref{app:bound_delta}. $\hfill \square$

\section{Supplement to~\ref{sec:pf_track_error} }
In this section, we upperbound the tracking errors $T_1$, $T_2$, and $T_3$ in \eqref{eq:delta_3_terms}. 
\subsection{Bounding $T_1$ in \eqref{eq:T_1_bound}}\label{app:bound_T1}
Invoking the RSC property (Assumption \ref{assump:G_RSM}) and  Lemma~\ref{lem:norm_bound}, we have
\begin{align*}
\norm{\frac{1}{\sqrt{N}}\bX (\btheta_i^{t+\frac{1}{2}} - \htheta) }^2 & \leq  L_\Sigma \norm{\btheta_i^{t+\frac{1}{2}} - \htheta}^2 +   \tau_g \norm{\btheta_i^{t+\frac{1}{2}} - \htheta}_1^2\\
& \leq  \left(L_\Sigma + 8 s \tau_g \right) \norm{\btheta_i^{t+\frac{1}{2}} - \htheta}^2 +  2 \tau_g \nu^2.
\end{align*}
Similarly, the other two terms in~\eqref{eq:T_1_bound} can be bounded respectively as
\begin{align*}
\norm{\frac{1}{\sqrt{N}}\bX \btheta_{j,\bot}^0}^2 \leq L_\Sigma \|  \btheta_{j}^0 - \bar{\btheta}^0 \|^2 +  \tau_g \| \btheta_{j}^0 - \bar{\btheta}^0 \|_1^2
\end{align*}
and 
\begin{align*}
\norm{\frac{1}{\sqrt{N}}\bX \Delta\btheta_{j}^s}^2 & \leq L_\Sigma \| \Delta\btheta_{j}^s \|^2 +  \tau_g \|\btheta_{j}^{s+\frac{1}{2}} - \btheta_{j}^{s}\|_1^2 \\
& \leq  L_\Sigma \| \Delta\btheta_{j}^s \|^2 + 2 \tau_g \|\btheta_{j}^{s+\frac{1}{2}} - \htheta\|_1^2  +  2 \tau_g \|\btheta_{j}^{s} - \htheta\|_1^2\\
& \leq L_\Sigma \| \Delta\btheta_{j}^s \|^2 +  16s \tau_g \|\btheta_{j}^{s+\frac{1}{2}} - \htheta\|^2  +  16s \tau_g \|\btheta_{j}^{s} - \htheta\|^2 +  8 \tau_g \nu^2.
\end{align*}
Summing up the terms and using $\vertiii{(\bW - \mathbf{J})^{t}}_\infty \leq c_m \rho^t$ (due to Assumption~\ref{assump:weight}), we arrive to the following bound for $T_1$:
\begin{align*}
T_1 \leq &\sum_{i = 1}^{m} \sum_{j = 1}^{m}  |\ell_{ij}^{(t)}| \frac{\epsilon}{2}\left(   \left(L_\Sigma + 8 s \tau_g \right) \norm{\btheta_i^{t+\frac{1}{2}} - \htheta}^2 + 2 \tau_g \nu^2 \right)\\
& + \sum_{i = 1}^{m} \sum_{j = 1}^{m}  |\ell_{ij}^{(t)}| \frac{1}{2 \epsilon} \left(  L_\Sigma \|  \btheta_{j}^0 - \bar{\btheta}^0 \|^2 + \tau_g \| \btheta_{j}^0 - \bar{\btheta}^0 \|_1^2 \right)\\
& + \sum_{i = 1}^{m}  \sum_{j = 1}^{m} \sum_{s=0}^{t-1}  |\ell_{ij}^{(t-s)}| \frac{\epsilon}{2} \left(  \left(L_\Sigma + 8 s \tau_g  \right) \norm{\btheta_i^{t+\frac{1}{2}} - \htheta}^2 + 2 \tau_g  \nu^2\right)\\
& + \sum_{i = 1}^{m}  \sum_{j = 1}^{m} \sum_{s=0}^{t-1}  |\ell_{ij}^{(t-s)}| \frac{1}{2 \epsilon} \left( L_\Sigma \| \Delta\btheta_{j}^s \|^2 + 16 s \tau_g \|\btheta_{j}^{s+\frac{1}{2}} - \htheta\|^2  + 16 s \tau_g  \|\btheta_{j}^{s} - \htheta\|^2 + 8 \tau_g \nu^2\right)\\
\leq & \frac{\epsilon}{2} c_m \rho^t \left(L_\Sigma + 8 s \tau_g  \right) \norm{\bTheta^{t+\frac{1}{2}} - \hTheta}_F^2 +  \frac{\epsilon}{2} \cdot m c_m \rho^t  \cdot 2 \tau_g  \nu^2\\
& + \frac{1}{2 \epsilon}  c_m \rho^t \left(  L_\Sigma \|  \bTheta_\bot^0  \|_F^2 + \tau_g \sum_{j = 1}^m \| \btheta_{j}^0 - \bar{\btheta}^0 \|_1^2 \right) \\
& + \sum_{s=0}^{t-1} c_m \rho^{t-s} \frac{\epsilon}{2} \left(L_\Sigma + 8 s \tau_g \right) \norm{\bTheta^{t+\frac{1}{2}} - \hTheta}_F^2 + \sum_{s=0}^{t-1} m c_m \rho^{t-s} \cdot \frac{\epsilon}{2} \cdot 2 \tau_g \nu^2\\
& + \sum_{s=0}^{t-1}c_m \rho^{t-s} \frac{1}{ 2 \epsilon} L_\Sigma \| \Delta \bTheta^s\|_F^2 + \sum_{s=0}^{t-1}c_m \rho^{t-s} \cdot \frac{1}{ 2 \epsilon} \cdot 16 s \tau_g \| \bTheta^{s+\frac{1}{2}} - \hTheta\|_F^2 \\
& + \sum_{s=0}^{t-1}c_m \rho^{t-s} \frac{1}{ 2 \epsilon} \cdot 16 s \tau_g \| \bTheta^{s} - \hTheta\|_F^2
+ \sum_{s=0}^{t-1} m c_m \rho^{t-s} \frac{1}{ 2 \epsilon} 8 \tau_g \nu^2.\\
\end{align*}
Combining terms yields
\begin{align*}
T_1\leq & \frac{  \rho \cdot c_m \epsilon}{1 - \rho} \left(L_\Sigma + 8 s \tau_g\right) \norm{\bTheta^{t+\frac{1}{2}} - \hTheta}_F^2  + \frac{ c_m L_\Sigma}{2 \epsilon}\sum_{s=0}^{t-1} \rho^{t-s}  \| \Delta \bTheta^s\|_F^2 \\
& + c_m   \cdot 8 s \tau_g \epsilon^{-1} \sum_{s=0}^{t-1}\rho^{t-s}  \| \bTheta^{s+\frac{1}{2}} - \hTheta\|_F^2 
+ c_m \cdot 8 s \tau_g \epsilon^{-1} \sum_{s=0}^{t-1} \rho^{t-s}    \| \bTheta^{s} - \hTheta\|_F^2\\
& + \frac{  \rho \cdot m c_m }{1 - \rho} \cdot 4 \tau_g \nu^2 (\epsilon + \epsilon^{-1}) +  \frac{1}{2 \epsilon}  c_m \rho^t \left(  L_\Sigma \|  \bTheta_\bot^0  \|_F^2 + \tau_g \sum_{j = 1}^m \| \btheta_{j}^0 - \bar{\btheta}^0 \|_1^2 \right).
\end{align*}

\subsection{Bounding $T_2$ in \eqref{eq:delta_3_terms}}\label{app:bound_T2}
Using the local RSM (cf. Assumption \ref{assump:L_RSS}) and  Lemma~\ref{lem:norm_bound}, we can bounds the summands in $T_2^{ij}$ [cf. \eqref{eq:T_2ij}] as 
\begin{align*}
\begin{split}
& ({\btheta}_i^{t +\frac{1}{2}} - \widehat{\btheta})^\top \Big( \frac{1}{n} \bX_j^\top \bX_j  \Big) \ell_{jk}^{(t)} \btheta_{k,\bot}^0 \\
\leq & |\ell_{jk}^{(t)} | \norm{\frac{1}{\sqrt{n}} \bX_j ({\btheta}_i^{t +\frac{1}{2}} - \widehat{\btheta})} \cdot \norm{\frac{1}{\sqrt{n}} \bX_j (\btheta_k^0 - \bar{\btheta}^0)}\\
\leq & |\ell_{jk}^{(t)}|\  \frac{\epsilon}{2} \left( \left( \ell_\Sigma + 8s \tau_\ell\right) \| {\btheta}_i^{t +\frac{1}{2}} - \widehat{\btheta} \|^2 + 2 \tau_\ell \nu^2\right)
 +  |\ell_{jk}^{(t)}| \ \frac{1}{2 \epsilon} \left( \ell_\Sigma \|\btheta_k^0 - \bar{\btheta}^0 \|^2 + \tau_\ell \|\btheta_k^0 - \bar{\btheta}^0 \|_1^2\right)
\end{split}
\end{align*}
and
\begin{align*}
\begin{split}
& ({\btheta}_i^{t +\frac{1}{2}} - \widehat{\btheta})^\top \Big( \frac{1}{n} \bX_j^\top \bX_j  \Big)\ell_{jk}^{(t-s)} \Delta \btheta_k^s\\
\leq & | \ell_{jk}^{(t-s)} | \ \norm{\frac{1}{\sqrt{n}} \bX_j ({\btheta}_i^{t +\frac{1}{2}} - \widehat{\btheta})} \norm{\frac{1}{\sqrt{n}} \bX_j \Delta \btheta_k^s}\\
\leq & |\ell_{jk}^{(t - s)}|\  \frac{\epsilon}{2} \left( \left( \ell_\Sigma + 8 s \tau_\ell\right) \| {\btheta}_i^{t +\frac{1}{2}} - \widehat{\btheta} \|^2 + 2 \tau_\ell\nu^2\right)\\
& +  |\ell_{jk}^{(t - s)}|\  \frac{1}{2 \epsilon} \left( \ell_\Sigma \|\Delta \btheta_k^{s} \|^2 + 16 s\tau_\ell  \|\btheta_k^{s+\frac{1}{2}} - \htheta \|^2 +  16 s\tau_\ell \|{\btheta}_k^{s} - \htheta \|^2 + 8 \tau_\ell \nu^2\right),
\end{split}
\end{align*}
respectively. 

Summing over $i,j,k$ and using $\vertiii{(\bW - \mathbf{J})^{t}}_\infty \leq c_m \rho^t$, we can bound $T_2$ as:
\begin{align*}
& m \cdot T_2 \notag\\
\leq & \sum_{i=1}^m  \sum_{j=1}^m \sum_{k=1}^m |\ell_{jk}^{(t)}|\  \frac{\epsilon}{2} \left( \left( \ell_\Sigma + 8 s \tau_\ell \right) \| {\btheta}_i^{t +\frac{1}{2}} - \widehat{\btheta} \|^2 + 2 \tau_\ell \nu^2\right)\notag\\
& 
+ \sum_{i=1}^m  \sum_{j=1}^m\sum_{k=1}^m |\ell_{jk}^{(t)}| \ \frac{1}{2 \epsilon} \left( \ell_\Sigma \|\btheta_k^0 - \bar{\btheta}^0 \|^2 + \tau_\ell \|\btheta_k^0 - \bar{\btheta}^0 \|_1^2\right) \notag\\
&
+ \sum_{i=1}^m  \sum_{j=1}^m\sum_{k=1}^m \sum_{s=0}^{t-1}  |\ell_{jk}^{(t - s)}|\  \frac{\epsilon}{2} \left( \left( \ell_\Sigma + 8 s \tau_\ell \right) \| {\btheta}_i^{t +\frac{1}{2}} - \widehat{\btheta} \|^2 + 2 \tau_\ell \nu^2\right) \notag\\
& +\sum_{i=1}^m  \sum_{j=1}^m\sum_{k=1}^m \sum_{s=0}^{t-1} |\ell_{jk}^{(t - s)}|\  \frac{1}{2 \epsilon} \left( \ell_\Sigma \|\Delta \btheta_k^s \|^2 + 16 s \tau_\ell \|\btheta_k^{s+\frac{1}{2}} - \htheta \|^2 +  16 s \tau_\ell \|{\btheta}_k^{s} - \htheta \|^2 + 8 \tau_\ell \nu^2\right). \notag\\
\leq &  m c_m \rho^t \cdot \frac{\epsilon}{2}\left( \ell_\Sigma + 8 s \tau_\ell \right) \norm{\bTheta^{t+\frac{1}{2}} - \hTheta}_F^2 + m^2 c_m \rho^t  \cdot \tau_\ell \nu^2 \cdot \epsilon \notag \\
& + m c_m \rho^t \cdot \frac{1}{2 \epsilon} \ell_\Sigma \ \| \bTheta_\bot^0\|_F^2 + m c_m \rho^t \cdot \frac{1}{2 \epsilon} \tau_\ell \sum_{k = 1}^m \| \btheta_k^0 - \bar{\btheta}^0\|_1^2 \notag\\
& + \sum_{s = 0}^{t-1} m c_m \rho^{t-s} \frac{\epsilon}{2} \left( \ell_\Sigma +8 s \tau_\ell \right) \| \bTheta^{t+\frac{1}{2}} - \hTheta\|_F^2 +  \sum_{s = 0}^{t-1} m^2 c_m \rho^{t-s} \tau_\ell \nu^2 \cdot \epsilon \notag\\
& + \sum_{s=0}^{t-1} m c_m \rho^{t-s} \frac{1}{2\epsilon} \ell_\Sigma \| \Delta \bTheta^s\|_F^2 + \sum_{s=0}^{t-1} m c_m \rho^{t-s} \frac{1}{\epsilon} 8 s \tau_\ell\| \bTheta^{s+\frac{1}{2}} - \hTheta\|_F^2 \notag\\
& + \sum_{s=0}^{t-1} m c_m \rho^{t-s}  \frac{1}{\epsilon} 8 s \tau_\ell \| \bTheta^{s} - \hTheta\|_F^2 + \sum_{s=0}^{t-1} m^2 c_m \rho^{t-s} \frac{1}{\epsilon} 4 \tau_\ell \nu^2\notag\\
\leq &  \frac{ \rho \cdot m c_m   \epsilon}{1 - \rho} \left( \ell_\Sigma + 8 s \tau_\ell \right) \norm{\bTheta^{t+\frac{1}{2}} - \hTheta}_F^2   +  \frac{m c_m \ell_\Sigma}{2\epsilon}  \sum_{s=0}^{t-1}  \rho^{t-s}  \| \Delta \bTheta^s\|_F^2 \notag\\
& + m c_m \frac{1}{\epsilon} 8 s \tau_\ell   \sum_{s=0}^{t-1}  \rho^{t-s}  \| \bTheta^{s+\frac{1}{2}} - \hTheta\|_F^2
+m c_m \frac{1}{\epsilon} 8 s \tau_\ell   \sum_{s=0}^{t-1}  \rho^{t-s}  \| \bTheta^{s} - \hTheta\|_F^2 \notag\\
& +  \frac{ \rho \cdot m^2 c_m}{ 1 - \rho} 4  \tau_\ell \nu^2 \cdot (\epsilon +  \epsilon^{-1}) + \frac{1}{2 \epsilon} m c_m \rho^t \left(  \ell_\Sigma \|\bTheta_{\bot}^0\|_F^2 + \tau_\ell \sum_{k =1}^m \| \btheta_k^0 - \bar{\btheta}^0\|_1^2\right).
\end{align*}

\subsection{Bounding $T_3$ in \eqref{eq:delta_3_terms}}\label{app:bound_T3}
We bound first $T_{3,1}$ and $T_{3,2}$.
Invoking the definition of $\bg_{j,\bot}^0$  we have
\begin{align*}
T_{3,1}= &  \left\vert  \left({\btheta}_i^{t + \frac{1}{2}} - \widehat{\btheta} \right)^\top\left(\nabla\cL_j (\btheta_j^0) - \frac{1}{m}\sum_{k=1}^m \nabla \cL_k (\btheta_k^0) \right) \right\vert\\
=  & \left\vert   \left({\btheta}_i^{t + \frac{1}{2}} - \widehat{\btheta} \right)^\top \left( \frac{1}{n} \bX_j^\top (\bX_j \btheta_j^0 - \by_j) - \frac{1}{mn} \sum_{k=1}^{m} \bX_k^\top (\bX_k \btheta_k^0 - \by_k) \right) \right\vert \\
= &  \left\vert   \left({\btheta}_i^{t + \frac{1}{2}} - \widehat{\btheta} \right)^\top \left( \frac{1}{n} \bX_j^\top (\bX_j \btheta_j^0 - \bX_j\btheta^* + \bX_j\btheta^*  - \by_j) - \frac{1}{mn} \sum_{k=1}^{m} \bX_k^\top (\bX_k \btheta_k^0 - \bX\btheta^* + \bX\btheta^*- \by_k) \right) \right\vert \\
\leq &  \left\vert   \left({\btheta}_i^{t + \frac{1}{2}} - \widehat{\btheta} \right)^\top \left( \frac{1}{n} \bX_j^\top \bX_j (\btheta_j^0 - \btheta^*) - \frac{1}{mn} \sum_{k=1}^{m} \bX_k^\top \bX_k (\btheta_k^0 - \btheta^* ) \right) \right\vert \\
& + \left\vert  \left({\btheta}_i^{t + \frac{1}{2}} - \widehat{\btheta} \right)^\top\nabla \cL_j (\btheta^*) \right\vert +  \left\vert  \left({\btheta}_i^{t + \frac{1}{2}} - \widehat{\btheta} \right)^\top \nabla \cL (\btheta^*)\right\vert.
\end{align*}
Therefore,
\begin{align}\label{eq:bound_T31}
T_{3,1} 
\leq & \frac{\epsilon}{2} \left(\ell_\Sigma \| {\btheta}_i^{t + \frac{1}{2}} - \widehat{\btheta} \|^2 + \tau_\ell \| {\btheta}_i^{t + \frac{1}{2}} - \widehat{\btheta} \|_1^2 \right) + \frac{1}{2 \epsilon} \left( \ell_\Sigma \| \btheta_j^0 - \btheta^* \|^2 + \tau_\ell \| \btheta_j^0 - \btheta^* \|_1^2 \right) \notag\\
& + \frac{1}{m} \sum_{k=1}^m \left( \frac{\epsilon}{2} \big( \ell_\Sigma \| {\btheta}_i^{t + \frac{1}{2}} - \widehat{\btheta} \|^2 + \tau_\ell\| {\btheta}_i^{t + \frac{1}{2}} - \widehat{\btheta} \|_1^2\big) + \frac{1}{2 \epsilon} \left(\ell_\Sigma \| \btheta_k^0 - \btheta^* \|^2 + \tau_\ell \|\btheta_k^0 - \btheta^* \|_1^2\right) \right) \notag\\
& + \| {\btheta}_i^{t + \frac{1}{2}} - \widehat{\btheta} \|_1 \underbrace{\left( \|\nabla \cL_j (\btheta^*)\|_\infty + \| \nabla \cL (\btheta^*) \|_\infty\right)}_{c_g} \notag\\
 \stackrel{(a)}{\leq} & \frac{\epsilon}{2} \left( \left( \ell_\Sigma + 8 s\tau_\ell \right) \| {\btheta}_i^{t + \frac{1}{2}} - \widehat{\btheta} \|^2 + 2 \tau_\ell \nu^2 \right) +  \frac{1}{2 \epsilon} \left( \ell_\Sigma \| \btheta_j^0 - \btheta^* \|^2 + \tau_\ell \| \btheta_j^0 - \btheta^* \|_1^2 \right) \notag\\
& + \frac{1}{m} \sum_{k=1}^m \left( \frac{\epsilon}{2} \left( \left( \ell_\Sigma + 8 s \tau_\ell \right)  \| {\btheta}_i^{t + \frac{1}{2}} - \widehat{\btheta} \|^2 + 2 \tau_\ell \nu^2\right) \right)\\
 & + \frac{1}{2 \epsilon} \frac{1}{m}  \left( \ell_\Sigma \| \bTheta^0 - \bTheta^*\|_F^2 + \tau_\ell \sum_{k = 1}^m \|\btheta_k^0 - \btheta^*\|_1^2\right)  
 +  c_g \left( 2\sqrt{s} \| {\btheta}_i^{t + \frac{1}{2}} - \widehat{\btheta} \| + \nu \right)  \notag\\
\stackrel{(b)}{\leq} & \epsilon \left( \left( \ell_\Sigma + 8 s \tau_\ell \right) \| {\btheta}_i^{t + \frac{1}{2}} - \widehat{\btheta} \|^2 + 2 \tau_\ell \nu^2 \right) \notag\\
& +  \frac{1}{2 \epsilon} \left( \ell_\Sigma \| \btheta_j^0 - \btheta^* \|^2 + \tau_\ell \| \btheta_j^0 - \btheta^* \|_1^2 + \frac{1}{m} \ell_\Sigma \| \bTheta^0 - \bTheta^*\|_F^2 + \tau_\ell \frac{1}{m} \sum_{k = 1}^m \| \btheta_k^0 - \btheta^*\|_1^2\right) \notag \\
&  + \epsilon \ell_{\Sigma} \cdot   \| {\btheta}_i^{t + \frac{1}{2}} - \widehat{\btheta} \|^2 + s c_g^2 \ell_{\Sigma}^{-1}\epsilon^{-1}  + c_g \nu \notag\\
\leq & \epsilon \left( \left( 2 \ell_\Sigma + 8 s \tau_\ell  \right) \| {\btheta}_i^{t + \frac{1}{2}} - \widehat{\btheta} \|^2 + 2 \tau_\ell \nu^2 \right)  \notag\\
&  +   \frac{1}{ \epsilon} \underbrace{ \frac{1}{2}\left( \ell_\Sigma \| \btheta_j^0 - \btheta^* \|^2 + \tau_\ell\| \btheta_j^0 - \btheta^* \|_1^2 + \frac{1}{m} \ell_\Sigma\| \bTheta^0 - \bTheta^*\|_F^2 + \tau_\ell\frac{1}{m} \sum_{k = 1}^m \| \btheta_k^0 - \btheta^*\|_1^2\right)}_{A_{3}} \notag\\
&  + s c_g^2 \ell_{\Sigma}^{-1} \epsilon^{-1} + c_g \nu,\notag
\end{align}
where   (a) follows from  Lemma~\ref{lem:norm_bound} while in (b) used the bound $2 \cdot (c_g \sqrt{s})\cdot  \| {\btheta}_i^{t + \frac{1}{2}} - \widehat{\btheta} \|\leq \epsilon \ell_{\Sigma} \cdot   \| {\btheta}_i^{t + \frac{1}{2}} - \widehat{\btheta} \|^2 + s c_g^2 \ell_{\Sigma}^{-1}\epsilon^{-1}$. 

Similarly,  we can bound  $T_{3,2}$  as
\begin{equation}\label{eq:bound_T32}\begin{aligned} 
T_{3,2} = & \left\vert  \left({\btheta}_i^{t + \frac{1}{2}} - \widehat{\btheta} \right)^\top \Big( \frac{1}{n}\bX_j^\top \bX_j \Big)  (\btheta_j^{s+1} - \btheta_j^{s})  \right\vert \\
\leq & \frac{\epsilon}{2} \left( \ell_\Sigma \| {\btheta}_i^{t + \frac{1}{2}} - \widehat{\btheta} \|^2 + \tau_\ell\| {\btheta}_i^{t + \frac{1}{2}} - \widehat{\btheta} \|_1^2\right) + \frac{1}{2 \epsilon} \left( \ell_\Sigma \|\btheta_j^{s+1} - \btheta_j^{s} \|^2 + \tau_\ell \|\btheta_j^{s+1} - \btheta_j^{s}\|_1^2\right)\\
\leq & \frac{\epsilon}{2}  \left(\ell_\Sigma + 8 s \tau_\ell \right) \| {\btheta}_i^{t + \frac{1}{2}} - \widehat{\btheta} \|^2   \\
&
+ \frac{1}{2 \epsilon}\left(\ell_\Sigma   \|\btheta_j^{s+1} - \btheta_j^{s}\|^2 + 16 s \tau_\ell \|{\btheta}_j^{s+1} - {\htheta}\|^2 + 16 s \tau_\ell \|{\btheta}_j^{s} - \widehat{\btheta}\|^2 \right)
+ \tau_\ell \nu^2 (\epsilon + 4 \epsilon^{-1})
\end{aligned}\end{equation}
Using (\ref{eq:bound_T31}) and (\ref{eq:bound_T32}) along with $\vertiii{(\bW - \mathbf{J})^{t}}_\infty \leq c_m \rho^t$, we can bound $T_3$ as follows:
\begin{align*}
& T_3 = \sum_{i=1}^m \left(\bar{\bg}^t- \bg_i^t \right)^T ({\btheta}_i^{t + \frac{1}{2}} - \widehat{\btheta}) 
\leq   \sum_{i=1}^m \sum_{j=1}^m \vert \ell_{ij}^{(t)}\vert \cdot T_{3,1} + \sum_{s=0}^{t-1}\sum_{i=1}^m \sum_{j=1}^m \vert \ell_{ij}^{(t-s)}\vert \cdot T_{3,2}
\end{align*}
and thus
\begin{align*}
T_3 \leq & \sum_{i=1}^m \sum_{j=1}^m \vert \ell_{ij}^{(t)}\vert \left(  \epsilon  \left( 2\ell_\Sigma + 8 s \tau_\ell \right) \| {\btheta}_i^{t + \frac{1}{2}} - \widehat{\btheta} \|^2   \right)\\
& +   \sum_{i=1}^m \sum_{j=1}^m \vert \ell_{ij}^{(t)}\vert \left( \epsilon \cdot 2 \tau_\ell \nu^2 + \epsilon^{-1} A_3  + s c_g^2 \ell_{\Sigma}^{-1} \epsilon^{-1} + c_g \nu\right)\\
& + \sum_{s=0}^{t-1}\sum_{i=1}^m \sum_{j=1}^m \vert \ell_{ij}^{(t-s)}\vert \left( \frac{\epsilon}{2}  \left(\ell_\Sigma + 8 s \tau_\ell\right) \| {\btheta}_i^{t + \frac{1}{2}} - \widehat{\btheta} \|^2   \right)\\
& + \sum_{s=0}^{t-1}\sum_{i=1}^m \sum_{j=1}^m \vert \ell_{ij}^{(t-s)}\vert \frac{1}{2 \epsilon}\left( \ell_\Sigma   \|\btheta_j^{s+1} - \btheta_j^{s}\|^2 + 16 s \tau_\ell \|{\btheta}_j^{s+1} - {\htheta}\|^2 + 16 s \tau_\ell\|{\btheta}_j^{s} - \widehat{\btheta}\|^2  \right)\\
& +  \sum_{s=0}^{t-1}\sum_{i=1}^m \sum_{j=1}^m \vert \ell_{ij}^{(t-s)}\vert \tau_\ell \nu^2 (\epsilon + 4 \epsilon^{-1})\\
 \leq & c_m \rho^t  \epsilon  \left( 2\ell_\Sigma + 8 s \tau_\ell  \right) \| {\bTheta}^{t + \frac{1}{2}} - \widehat{\bTheta} \|_F^2 
 + m c_m \rho^t \left( \epsilon \cdot 2 \tau_\ell \nu^2 + \epsilon^{-1} A_3   + s c_g^2 \ell_{\Sigma}^{-1} \epsilon^{-1} + c_g \nu\right)\\
& + \sum_{s=0}^{t-1} c_m \rho^{t-s} \frac{\epsilon}{2} \left(\ell_\Sigma + 8 s \tau_\ell\right) \| \bTheta^{t+\frac{1}{2}} - \hTheta\|_F^2 \\
& + \sum_{s=0}^{t-1} c_m \rho^{t-s} \frac{1}{2 \epsilon} \ell_\Sigma \|\bTheta^{s+1} - \bTheta^s\|_F^2 + \sum_{s=0}^{t-1} c_m \rho^{t-s} \frac{1}{\epsilon} 8 s \tau_\ell \|\bTheta^{s+1} - \hTheta\|^2\\
&  + \sum_{s=0}^{t-1} c_m \rho^{t-s} \frac{1}{ \epsilon} 8 s \tau_\ell \|\bTheta^{s} - \hTheta\|^2 + \sum_{s=0}^{t-1}m  c_m \rho^{t-s} \tau_\ell \nu^2 (\epsilon + 4 \epsilon^{-1})  \\
\leq & \frac{3}{2 } \cdot \frac{\rho \cdot c_m \epsilon}{1 - \rho} \left( 2 \ell_\Sigma + 8 s \tau_\ell  \right) \| {\bTheta}^{t + \frac{1}{2}} - \widehat{\bTheta} \|_F^2 + \frac{c_m \ell_\Sigma}{2 \epsilon} \sum_{s=0}^{t-1} \rho^{t-s} \|\bTheta^{s+1} - \bTheta^s\|_F^2\\
& +c_m 8 s \tau_\ell \epsilon^{-1} \sum_{s=0}^{t-1}  \rho^{t-s}  \|\bTheta^{s+1} - \hTheta\|_F^2
+ c_m 8 s \tau_\ell \epsilon^{-1} \sum_{s=0}^{t-1}  \rho^{t-s}  \|\bTheta^{s} - \hTheta\|_F^2\\
& + m c_m \rho^t \left( A_3  \epsilon^{-1}  + s c_g^2 \ell_{\Sigma}^{-1} \epsilon^{-1} + c_g \nu \right) + \frac{\rho \cdot m c_m }{1 - \rho} 4 \tau_\ell\nu^2 (\epsilon + \epsilon^{-1})\\
\leq & \frac{3}{2 } \cdot \frac{\rho \cdot c_m \epsilon}{1 - \rho} \left( 2 \ell_\Sigma + 8 s \tau_\ell  \right) \| {\bTheta}^{t + \frac{1}{2}} - \widehat{\bTheta} \|_F^2 + \frac{c_m \ell_\Sigma}{2 \epsilon} \sum_{s=0}^{t-1} \rho^{t-s} \|\bTheta^{s+1} - \bTheta^s\|_F^2\\
& +c_m 8 s \tau_\ell \epsilon^{-1} \sum_{s=0}^{t-1}  \rho^{t-s}  \|\bTheta^{s+1} - \hTheta\|_F^2
+ c_m 8 s \tau_\ell \epsilon^{-1} \sum_{s=0}^{t-1}  \rho^{t-s}  \|\bTheta^{s} - \hTheta\|_F^2\\
& + m c_m \rho^t \left( A_3  \epsilon^{-1}  + s c_g^2 \ell_{\Sigma}^{-1} \epsilon^{-1} + c_g \nu \right) + \frac{\rho \cdot m c_m }{1 - \rho} 4 \tau_\ell\nu^2 (\epsilon + \epsilon^{-1}) .
\end{align*}  
\subsection{Bounding $\delta^t$ in \eqref{eq:delta_3_terms}}\label{app:bound_delta}
Using the bounds on $T_1$, $T_2$, and $T_3$, we can finally bound $\delta^t$:
\begin{align*}
\delta^t & = \sum_{i=1}^m(\nabla \cL(\btheta_i^t )- \by_i^t )^\top ({\btheta}_i^{t + \frac{1}{2}} - \widehat{\btheta})\\
&\leq  T_1 +  T_2 + T_3\\
\leq & \frac{  \rho \cdot c_m \epsilon}{1 - \rho} \left(L_\Sigma + 8 s \tau_g \right) \norm{\bTheta^{t+\frac{1}{2}} - \hTheta}_F^2  + \frac{ c_m L_\Sigma}{2 \epsilon}\sum_{s=0}^{t-1} \rho^{t-s}  \| \Delta \bTheta^s\|_F^2 \\
& + c_m \cdot  8 s \tau_g \epsilon^{-1} \sum_{s=0}^{t-1}\rho^{t-s}  \| \bTheta^{s+\frac{1}{2}} - \hTheta\|_F^2 
+ c_m \cdot 8 s \tau_g \epsilon^{-1} \sum_{s=0}^{t-1} \rho^{t-s}    \| \bTheta^{s} - \hTheta\|_F^2\\
& + \frac{  \rho \cdot m c_m }{1 - \rho} 4 \tau_g \nu^2 (\epsilon + \epsilon^{-1})  \\
& + \frac{ \rho \cdot  c_m   \epsilon}{1 - \rho} \left( \ell_\Sigma + 8 s \tau_\ell\right) \norm{\bTheta^{t+\frac{1}{2}} - \hTheta}_F^2   +  \frac{ c_m \ell_\Sigma}{2\epsilon}  \sum_{s=0}^{t-1}  \rho^{t-s}  \| \Delta \bTheta^s\|_F^2 \\
& + c_m 8 s \tau_\ell \epsilon^{-1} \sum_{s=0}^{t-1}  \rho^{t-s}  \| \bTheta^{s+\frac{1}{2}} - \hTheta\|_F^2
+ c_m 8 s \tau_\ell \epsilon^{-1} \sum_{s=0}^{t-1}  \rho^{t-s}  \| \bTheta^{s} - \hTheta\|_F^2 \\
& +    \frac{ \rho \cdot m c_m}{ 1 - \rho} 4 \tau_\ell \nu^2 \cdot (\epsilon +  \epsilon^{-1}) \\
& + \frac{3}{2 } \cdot \frac{\rho \cdot c_m \epsilon}{1 - \rho} \left( 2 \ell_\Sigma + 8 s \tau_\ell  \right) \| {\bTheta}^{t + \frac{1}{2}} - \widehat{\bTheta} \|_F^2 + \frac{c_m \ell_\Sigma}{2 \epsilon} \sum_{s=0}^{t-1} \rho^{t-s} \|\bTheta^{s+1} - \bTheta^s\|_F^2\\
& +c_m 8s \tau_\ell \epsilon^{-1} \sum_{s=0}^{t-1}  \rho^{t-s}  \|\bTheta^{s+1} - \hTheta\|_F^2
+ c_m 8 s \tau_\ell \epsilon^{-1} \sum_{s=0}^{t-1}  \rho^{t-s}  \|\bTheta^{s} - \hTheta\|_F^2\\
& + m c_m \rho^t \left( A_3 \epsilon^{-1}  + s c_g^2 \ell_{\Sigma}^{-1} \epsilon^{-1} + 1 \right) + \frac{\rho \cdot m c_m }{1 - \rho} 4 \tau_\ell \nu^2 (\epsilon + \epsilon^{-1}) + \frac{\rho \cdot m c_m}{1 - \rho} c_g^2 \nu^2.
\end{align*}
\begin{align*}
\delta^t \stackrel{(a)}{\leq} & \frac{7}{2 } \cdot \frac{\rho \cdot c_m \epsilon}{1 - \rho} \left( 2 \ell_\Sigma + 8 s \tau_\ell \right) \| {\bTheta}^{t + \frac{1}{2}} - \widehat{\bTheta} \|_F^2 \\
& +  \frac{ c_m \ell_\Sigma}{ \epsilon}\sum_{s=0}^{t-1} \rho^{t-s}  \| \Delta \bTheta^s\|_F^2 
+ \frac{ c_m \ell_\Sigma}{ 2\epsilon}\sum_{s=0}^{t-1} \rho^{t-s}  \|  \bTheta^{s+1} - \bTheta^s\|_F^2\\
& +c_m 24 \tau_\ell \epsilon^{-1} \sum_{s=0}^{t-1}  \rho^{t-s}  \| \bTheta^{s+\frac{1}{2}} - \hTheta\|_F^2
+ c_m 24 \tau_\ell \epsilon^{-1} \sum_{s=0}^{t-1}  \rho^{t-s}  \| \bTheta^{s} - \hTheta\|_F^2 \\
& +  m c_m \rho^t \left( A_3\cdot  \epsilon^{-1}  + s c_g^2 \ell_{\Sigma}^{-1} \epsilon^{-1} + 1\right) + \frac{\rho \cdot m c_m }{1 - \rho} 12 \tau_\ell \nu^2 (\epsilon + \epsilon^{-1}) + \frac{\rho \cdot m c_m}{1 - \rho} c_g^2 \nu^2\\
\stackrel{(b)}{\leq} &\frac{7}{2 } \cdot \frac{\rho \cdot c_m \epsilon}{1 - \rho} \left( 2 \ell_\Sigma + 8 s \tau_\ell  \right) \| {\bTheta}^{t + \frac{1}{2}} - \widehat{\bTheta} \|_F^2 \\
& +  \frac{ 2c_m \ell_\Sigma}{ \epsilon}\sum_{s=0}^{t-1} \rho^{t-s}  \| \Delta \bTheta^s\|_F^2 
+ \frac{ 2 c_m \ell_\Sigma}{ \epsilon}\sum_{s=0}^{t-1} \rho^{t-s}  \|  \bTheta_\bot^{s}\|_F^2\\
& +c_m 24 s\tau_\ell \epsilon^{-1} \sum_{s=0}^{t-1}  \rho^{t-s}  \| \bTheta^{s+\frac{1}{2}} - \hTheta\|_F^2
+ c_m 24 s \tau_\ell \epsilon^{-1} \sum_{s=0}^{t-1}  \rho^{t-s}  \| \bTheta^{s - \frac{1}{2}} - \hTheta\|_F^2 \\
&  + m c_m \rho^t \left( A_3 \cdot  \epsilon^{-1}  + s c_g^2 \ell_{\Sigma}^{-1} \epsilon^{-1} + 1 \right) + \frac{\rho \cdot m c_m }{1 - \rho} 12 \tau_\ell \nu^2 (\epsilon + \epsilon^{-1}) + \frac{\rho \cdot m c_m}{1 - \rho} c_g^2 \nu^2,
\end{align*}
where in $(a)$ we used $L_\Sigma \leq \ell_\Sigma$, $\tau_g \leq \tau_\ell$, and 
\eqref{eq:bound_tetha_plus_thata_half}; 
and in $(b)$ we used again \eqref{eq:bound_tetha_plus_thata_half}
 and \eqref{eq:bound_Delta_theta}. \hfill $\square$

\section{Proof of Proposition~\ref{prop:consenus_err}}~\label{sec:pf_consensus_err}
Using Assumptions \ref{assump:net1} and \ref{assump:weight} on (\ref{eq:update_err_theta})
we have
	\begin{align*}
	\|\bTheta_\bot^{t}\| \leq \rho^{t} \| \bTheta_\bot^{0}\| + \sum_{s=0}^{t-1} \rho^{t-s} \|\Delta \bTheta^s \|.
	\end{align*}
Squaring both sides	gives
	\begin{align*}
	\|\bTheta_\bot^{t}\|^2 & \leq \left( \rho^{t} \| \bTheta_\bot^{0}\| + \sum_{s=0}^{t-1} \rho^{t-s} \|\Delta \bTheta^s \| \right)^2
 \leq 2 \rho^{2t} \| \bTheta_\bot^{0}\|^2 + 2 \left(\sum_{s=0}^{t-1} \rho^{t-s} \|\Delta \bTheta^s \|\right)^2 \\
& \leq 2 \rho^{2t} \| \bTheta_\bot^{0}\|^2 +  \frac{2\rho}{1-\rho} \cdot \sum_{s=0}^{t-1} \rho^{t-s} \|\Delta \bTheta^s \|^2,
	\end{align*}
	where the last inequality follows from Cauchy-Schwartz inequality. $\hfill \square$

\section{Proof of Lemma \ref{lem:rate_condition_2}}\label{sec:proof_lem:rate_condition_2}
	Recall $G_{{\rm net}}(z) = \frac{\rho}{z - \rho}$. Since   \eqref{eq:rate_condition_2} is quadratic in $G_{{\rm net}}(z)$, we obtain  the following feasible range  of $G_{{\rm net}}(z)$:
	\begin{align*}
		\begin{split}
			G_{{\rm net}}(z) & \leq \frac{ \frac{C_2}{\gamma}\frac{ c_m \ell_\Sigma}{ \epsilon} -\sqrt{\left(\frac{C_2}{\gamma}\frac{ c_m \ell_\Sigma}{ \epsilon} \right)^2+ \frac{4 C_2}{\gamma} \frac{ c_m \ell_\Sigma}{ \epsilon}\frac{2\rho}{1-\rho} \left(1 - \frac{L_\Sigma}{\gamma} \right)}}{- \frac{2C_2}{\gamma} \frac{ c_m \ell_\Sigma}{ \epsilon}\frac{2\rho}{1-\rho}}\\
			& = \frac{ 1 - \rho}{4 \rho} \left( \sqrt{1 + \frac{8\rho}{1-\rho} \left(1 - \frac{L_\Sigma}{\gamma} \right) \left(\frac{\gamma}{C_2}\frac{ \epsilon} { c_m \ell_\Sigma}\right)  } - 1 \right)\\
			& \stackrel{\eqref{eq:epsilon}}{=} \frac{ 1 - \rho}{4 \rho} \left( \sqrt{1 + \left(\gamma - L_\Sigma \right) \left(\frac{\mu_{\Sigma}}{ \ell_{\Sigma}^2} \frac{ 2 \rho}  { C_2^2   c_m^2 }\right)  } - 1 \right).
		\end{split}
	\end{align*}
	Therefore, condition~\eqref{eq:rate_condition_2} is fulfilled for all $z$ in the form 
	\begin{align*}
		z & = \rho + \frac{\rho}{G_{{\rm net}}(z) }
		\geq \rho + \frac{4 \rho^2}{1 - \rho} \left( \sqrt{1 + \left(\gamma - L_\Sigma \right) \left(\frac{\mu_{\Sigma}}{ \ell_{\Sigma}^2} \frac{ 2 \rho}  { C_2^2   c_m^2 }\right)  } - 1 \right)^{-1}.
	\end{align*} $\hfill \square$
 
\section{Proof of Lemma~\ref{lem:rate_condition_1}}\label{sec:proof_lem:rate_condition_1}
	We substitute $G_{{\rm net}}(z) = \frac{\rho}{z - \rho}$ into~\eqref{eq:z_ub_2} and solve the equation
	\begin{align}\label{eq:def_z_0}
		z = \frac{A + \frac{\rho}{z - \rho} }{D - \frac{\rho}{z - \rho}}.
	\end{align}
	Let $\Delta$ denote the discriminant of the second order equation in $z $ \eqref{eq:z_ub_2}. 
Since 
	\begin{align*}
		\begin{split}
			\Delta & = (D\rho  + \rho + A )^2 - 4D(A- 1)\rho
			 = (D\rho  + \rho - A )^2 + 4 (A + D) \rho \geq 0 ,
		\end{split}
	\end{align*}
	Eq.~\eqref{eq:def_z_0} has two real roots ($z_1 \leq z_2$)
	\begin{align*}
		z_{1,2} = \frac{(D +1)\rho  + A  \pm \sqrt{(D\rho  + \rho + A )^2 - 4D(A- 1)\rho}}{2D }.
	\end{align*}
	By inspection, we have that $z_1 \geq 0.$ Hence, \eqref{eq:z_ub_2} is satisfied for all $z \in (0,z_1] \cup [z_2, 1)$.
	
	Next we  show  $z_1 \leq \rho$ and $z_2 \geq \rho$. As $z > \rho $ must be satisfied,  $z \in [z_2, 1)$. Suppose on the contrary that $z_1 > \rho$, which implies that
	\begin{align*}
		\begin{split}
			& \rho + A - D \rho > \sqrt{(D \rho + \rho + A)^2 - 4D(A  - 1)\rho}\\
			\Leftrightarrow & (\rho+A-D\rho)^2 > (\rho+A+D\rho)^2 - 4D(A - 1)\rho \quad \text{and} \quad \rho +  A - D \rho >0.
		\end{split}
	\end{align*}
	After simplification, the first inequality becomes
	$
		4 D \rho^2 + 4 D \rho<0,
	$
	which is impossible. Therefore we have proved that $z_1 \leq \rho$.
	
	As both $z_1$ and $z_2$ are nonnegative, we can simply upperbound $z_2$ as $z_2 \leq z_1 + z_2 $. This completes the proof. $\hfill \square$

\section{Proof of Proposition~\ref{prop:z-upbdd}}\label{pf:prop:rate_bound_1}

Substituting the expression of $\gamma$ given by~\eqref{eq:gamma_expression} into~\eqref{eq:z_rho} we obtain:
\begin{align}\label{eq:rate_rho}
	\begin{split}
		z _\rho & = \rho +  \sqrt{\rho}(1 - \rho)^3 \frac{ C_2^2   } {  C_3  } \left(  \sqrt{1 +   \frac{ \rho \sqrt{\rho} }{ (1 - \rho)^4}  \left( \frac{ 2 C_3  }  { C_2^2   }\right)  } + 1 \right)\\
		& \leq \rho +  \sqrt{\rho}(1 - \rho)^3 \frac{  C_2^2   } {  C_3  } \left(  \frac{ \sqrt{\rho \sqrt{\rho}}}{ (1 - \rho)^2} \cdot  \frac{\sqrt{2   C_3 }}{C_2} + 2 \right)\\
		& \leq  \rho +  \frac{ \sqrt{2} C_2   } {  \sqrt{C_3}  }  \rho (1 - \rho) + \frac{ 2 C_2^2   } {  C_3  } \sqrt{\rho} (1 - \rho)^3 \triangleq \bar{z}_\rho.
	\end{split}
\end{align}

The rest of the proof bounds $z_\sigma$ based on Lemma~\ref{lem:rate_condition_1}.
Substituting  $\gamma$ into the expression of $A$ and $D$ we get
	\begin{align}
		\begin{split}
			\frac{A}{D} 
			& = \frac{L_\Sigma  -  \mu_\Sigma+ C_1 s (\tau_\mu + \tau_g) +  C_3 \frac{  \ell_{\Sigma}^2 }{\mu_{\Sigma}}   \frac{c_m^2 \sqrt{\rho}}{ (1 - \rho)^4} }
			{L_\Sigma  -  C_1 s \tau_g  - \frac{\rho}{2} \mu_{\Sigma} - \frac{\rho}{2} \frac{\mu_{\Sigma}}{ \ell_{\Sigma}}    s \tau_\ell   +  C_3 \frac{  \ell_{\Sigma}^2 }{\mu_{\Sigma}}   \frac{c_m^2 \sqrt{\rho}}{ (1 - \rho)^4} }\\
			& = \frac{1  -  \kappa^{-1}+ C_1 s (\tau_\mu + \tau_g)/L_\Sigma +  C_3 \frac{  \ell_{\Sigma}^2 }{\mu_{\Sigma} L_\Sigma}   \frac{c_m^2 \sqrt{\rho}}{ (1 - \rho)^4} }
			{1  -  C_1 s \tau_g/L_\Sigma  - \frac{\rho}{2} \kappa^{-1} - \frac{\rho}{2} \frac{\mu_{\Sigma}}{ \ell_{\Sigma} L_\Sigma}   s \tau_\ell   +  C_3 \frac{  \ell_{\Sigma}^2 }{\mu_{\Sigma}L_\Sigma}   \frac{c_m^2 \sqrt{\rho}}{ (1 - \rho)^4} }.
		\end{split}
	\end{align}
Dividing  the numerator and denominator by $1- C_1 s \tau_g/L_\Sigma$ and by the definition of $\sigma_0$:
	\begin{align}\label{eq:A/D}
			\frac{A}{D}= & \frac{\sigma_{0}  + C_3 \frac{  \ell_{\Sigma}^2 }{\mu_{\Sigma}L_\Sigma}   \frac{c_m^2 \sqrt{\rho}}{ (1 - \rho)^4} (1 - C_1 s \tau_g /L_\Sigma)^{-1}}{1 + \left(  - \frac{\rho}{2} \kappa^{-1} - \frac{\rho}{2} \frac{\mu_{\Sigma}}{ \ell_{\Sigma} L_\Sigma}   s \tau_\ell  +  C_3 \frac{  \ell_{\Sigma}^2 }{\mu_{\Sigma}L_\Sigma}   \frac{c_m^2 \sqrt{\rho}}{ (1 - \rho)^4} \right) (1 - C_1 s \tau_g /L_\Sigma)^{-1} } \notag\\
			\stackrel{A\leq D}{\leq}  & \,\,\frac{\sigma_{0}  + C_3 \frac{  \ell_{\Sigma}^2 }{\mu_{\Sigma}L_\Sigma}   \frac{c_m^2 \sqrt{\rho}}{ (1 - \rho)^4} (1 - C_1 s \tau_g /L_\Sigma)^{-1} + \left(   \frac{\rho}{2} \kappa^{-1} + \frac{\rho}{2} \frac{\mu_{\Sigma}}{ \ell_{\Sigma} L_\Sigma}    s \tau_\ell     \right) (1 - C_1 s \tau_g /L_\Sigma)^{-1}  }
			{1 + \left(    C_3 \frac{  \ell_{\Sigma}^2 }{\mu_{\Sigma}L_\Sigma}   \frac{c_m^2 \sqrt{\rho}}{ (1 - \rho)^4} \right) (1 - C_1 s \tau_g /L_\Sigma)^{-1} } \notag\\
			\leq & \sigma_{0}  + C_3 \frac{  \ell_{\Sigma}^2 }{\mu_{\Sigma}L_\Sigma}   \frac{c_m^2 \sqrt{\rho}}{ (1 - \rho)^4} (1 - C_1 s \tau_g /L_\Sigma)^{-1} + \frac{\rho}{2} \left(   \kappa^{-1} +  \frac{\mu_{\Sigma}}{ \ell_{\Sigma} L_\Sigma}    s \tau_\ell     \right) (1 - C_1 s \tau_g /L_\Sigma)^{-1} \notag \\
			= & \sigma_{0} + \sqrt{\rho} \cdot \frac{C_3 \frac{  \ell_{\Sigma}^2 }{\mu_{\Sigma}L_\Sigma}   \frac{c_m^2 }{ (1 - \rho)^4}  + \frac{1}{2} \left(   \kappa^{-1} +  \frac{\mu_{\Sigma}}{ \ell_{\Sigma} L_\Sigma}    s \tau_\ell     \right) }{1 - C_1 s \tau_g /L_\Sigma} \notag\\
			\stackrel{\rho \leq 1/2}{\leq }  & \sigma_{0} + \sqrt{\rho} \cdot \frac{16 C_3  c_m^2 \frac{  \ell_{\Sigma}^2 }{\mu_{\Sigma}L_\Sigma}     + \frac{1}{2} \left(   \kappa^{-1} +  \frac{\mu_{\Sigma}}{ \ell_{\Sigma} L_\Sigma}    s \tau_\ell     \right) }{1 - C_1 s \tau_g /L_\Sigma},
	\end{align}
	where  for $A \leq D$ to hold, we require 
	\begin{align}\label{S.1}
		\left( 1-  \frac{\rho}{2} \right)\mu_{\Sigma}  \geq C_1 s \tau_g   + C_1 s (\tau_\mu + \tau_g) + \frac{\rho}{2} \frac{\mu_{\Sigma}}{ \ell_{\Sigma} }    s \tau_\ell .   
	\end{align}
	
	Similarly, we can bound $ \rho + \rho/D$ as 
	\begin{align}\label{eq:bound_rho_B}
			\rho + \frac{\rho}{D} 
			& =  \rho + \frac{ \rho\ \frac{ \ell_{\Sigma}}{\mu_{\Sigma} L_{\Sigma}} \frac{2 C_2^2   c_m^2 }{1 - \rho}   s \tau_\ell }{1 - C_1 s \tau_g/L_{\Sigma} +  C_3 \frac{  \ell_{\Sigma}^2 }{\mu_{\Sigma}L_{\Sigma}}   \frac{c_m^2 \sqrt{\rho}}{ (1 - \rho)^4} -   \left\{  \frac{\rho}{2} \kappa^{-1} +   \frac{\rho}{2} \frac{\mu_{\Sigma}}{ \ell_{\Sigma} L_{\Sigma}}     s \tau_\ell  \right\}} \notag\\
			& 	\stackrel{\rho \leq D}{\leq} \rho + \frac{ \rho\ \frac{ \ell_{\Sigma}}{\mu_{\Sigma} L_{\Sigma}} \frac{2 C_2^2   c_m^2 }{1 - \rho}   s \tau_\ell   +   \frac{\rho}{2} \frac{\mu_{\Sigma}}{ \ell_{\Sigma} L_{\Sigma}}     s \tau_\ell }{1 - C_1 s \tau_g/L_{\Sigma} -    \frac{\rho}{2} \kappa^{-1} +  C_3 \frac{  \ell_{\Sigma}^2 }{\mu_{\Sigma}L_{\Sigma}}   \frac{c_m^2 \sqrt{\rho}}{ (1 - \rho)^4} }\\
			& \stackrel{\rho \kappa^{-1}\leq 1}\leq \rho + \frac{ \rho\ \frac{ \ell_{\Sigma}}{\mu_{\Sigma} L_{\Sigma}} \frac{2 C_2^2   c_m^2 }{1 - \rho}   s \tau_\ell   +   \frac{\rho}{2} \frac{\mu_{\Sigma}}{ \ell_{\Sigma} L_{\Sigma}}     s \tau_\ell }{\frac{1}{2} - C_1 s \tau_g/L_{\Sigma}  +  C_3 \frac{  \ell_{\Sigma}^2 }{\mu_{\Sigma}L_{\Sigma}}   \frac{c_m^2 \sqrt{\rho}}{ (1 - \rho)^4}},\notag
	\end{align}
	where   $\rho \leq D $ requires 	\begin{align*}
		L_{\Sigma}  +  C_3 \frac{  \ell_{\Sigma}^2 }{\mu_{\Sigma}}   \frac{c_m^2 \sqrt{\rho}}{ (1 - \rho)^4} \geq  \rho\ \frac{ \ell_{\Sigma}}{\mu_{\Sigma} } \frac{2 C_2^2   c_m^2 }{1 - \rho}   s \tau_\ell  + C_1 s \tau_g +  \frac{\rho}{2} \mu_{\Sigma} +   \frac{\rho}{2} \frac{\mu_{\Sigma}}{ \ell_{\Sigma} }     s \tau_\ell .
	\end{align*}
	Using the fact that $\rho \leq 1$ and $L_{\Sigma} \geq \mu_{\Sigma}$, it suffices  
	\begin{align}\label{S.2}
		\mu_{\Sigma}   \geq  \rho\ \frac{ \ell_{\Sigma}}{\mu_{\Sigma} } \frac{4 C_2^2   c_m^2 }{1 - \rho}   s \tau_\ell  +2 C_1 s \tau_g  +   \rho \frac{\mu_{\Sigma}}{ \ell_{\Sigma} }    s \tau_\ell.
	\end{align}
	Note that  \eqref{eq:sample_condition} implies both (\ref{S.1}) and (\ref{S.2}). 
	
	Finally, we further simply the bound  in~\eqref{eq:bound_rho_B} as follows. Multiplying both the numerator and denominator by $(1 - \rho)^4$ we get
	\begin{align*}
		\begin{split}
			\rho + \frac{\rho}{D} \leq \rho + \frac{ \rho\ \frac{ \ell_{\Sigma}}{\mu_{\Sigma} L_{\Sigma}} \cdot 2 C_2^2   c_m^2    s \tau_\ell  (1 - \rho)^3 +   \frac{\rho}{2} \frac{\mu_{\Sigma}}{ \ell_{\Sigma} L_{\Sigma}}     s \tau_\ell (1 - \rho)^4}{(\frac{1}{2} - C_1 s \tau_g/L_{\Sigma} ) (1 - \rho)^4 +  C_3 c_m^2 \sqrt{\rho} \frac{  \ell_{\Sigma}^2 }{\mu_{\Sigma}L_{\Sigma}}  }.
		\end{split}
	\end{align*}
	Define function 
	\begin{align*}
		\begin{split}
			f(\rho) = \left(\frac{1}{2} - C_1 s \tau_g/L_{\Sigma} \right) (1 - \rho)^4 +  C_3 c_m^2 \sqrt{\rho} \frac{  \ell_{\Sigma}^2 }{\mu_{\Sigma}L_{\Sigma}}.
		\end{split}
	\end{align*}
	Differentiating $f$ with respect to $\rho$ yields
	\begin{align*}
		\begin{split}
			f'(\rho) = - 4 \left(\frac{1}{2} - C_1 s \tau_g/L_{\Sigma} \right) (1 - \rho)^3 +  C_3 c_m^2  \frac{  \ell_{\Sigma}^2 }{\mu_{\Sigma}L_{\Sigma}} \frac{1}{2 \sqrt{\rho}} >0,
		\end{split}
	\end{align*}
	for $C_3 \geq 4$, due to $c_m \geq 1$, $\ell_{\Sigma} \geq \mu_{\Sigma}$, and $\ell_{\Sigma} \geq L_{\Sigma}$. Therefore, $f(\rho)$ is monotonically increasing in $(0,1)$ and it holds
	\begin{align}\label{eq:bound_rho_B-1}
		\begin{split}
			\rho + \frac{\rho}{D} & \leq  \rho +    \rho\cdot\frac{  \frac{ \ell_{\Sigma}}{\mu_{\Sigma} L_{\Sigma}} \cdot 4 C_2^2   c_m^2    s \tau_\ell   +   \frac{\mu_{\Sigma}}{ \ell_{\Sigma} L_{\Sigma}}     s \tau_\ell }{1 - 2 C_1 s \tau_g/L_{\Sigma}}.
		\end{split}
	\end{align}

 Combining~\eqref{eq:A/D} and~\eqref{eq:bound_rho_B-1}, 
 $z_\sigma$   can be readily bounded as 
\begin{align}\label{eq:z_sigma_bound}
		\ z_\sigma 
		\leq & \sigma_{0} + \sqrt{\rho} \cdot \frac{16 C_3  c_m^2 \frac{  \ell_{\Sigma}^2 }{\mu_{\Sigma}L_\Sigma}     + \frac{1}{2} \left(   \kappa^{-1} +  \frac{\mu_{\Sigma}}{ \ell_{\Sigma} L_\Sigma}    s \tau_\ell     \right) }{1 - C_1 s \tau_g /L_\Sigma} + \rho +  \rho \cdot \frac{  \frac{ \ell_{\Sigma}}{\mu_{\Sigma} L_{\Sigma}} \cdot 4 C_2^2   c_m^2    s \tau_\ell   +   \frac{\mu_{\Sigma}}{ \ell_{\Sigma} L_{\Sigma}}     s \tau_\ell }{1 - 2 C_1 s \tau_g/L_{\Sigma}  } \notag\\
		\leq \ & \sigma_0 + \sqrt{\rho} \cdot \frac{18 C_3  c_m^2 \frac{  \ell_{\Sigma}^2 }{\mu_{\Sigma}L_\Sigma} + \frac{ \ell_{\Sigma}}{\mu_{\Sigma} L_{\Sigma}} \cdot 4 C_2^2   c_m^2    s \tau_\ell   +   \frac{2 \mu_{\Sigma}}{ \ell_{\Sigma} L_{\Sigma}}     s \tau_\ell  }{1 - 2 C_1 s \tau_g/L_{\Sigma} },
\end{align}
where in the last inequality we have used the fact that $C_3  c_m^2 \frac{  \ell_{\Sigma}^2 }{\mu_{\Sigma}L_\Sigma}  \geq 1$, $\rho \leq 1$ and $\kappa^{-1} \leq 1$. To make the bound not vacuous, we need to  assume  $1 - 2 C_1 s \tau_g /L_\Sigma>0$ [cf.~\eqref{eq:cond_denom_z_sigma}].
Combining  \eqref{eq:rate_rho} and \eqref{eq:z_sigma_bound} completes the proof. $\hfill \square$

 \section{Proof of Corollary~\ref{cor:z-upbdd} }\label{pf:cor-z-upbdd}
 
Recall the definition of $\sigma_{0}$ as in \eqref{eq:sigma_0}:
 \begin{align*}
     	\sigma_{0} \triangleq  \frac{1 - \kappa^{-1} + C_1 s (\tau_\mu + \tau_g)/ L_\Sigma }{1 - C_1 s \tau_g /L_\Sigma},\quad \kappa \triangleq \frac{L_{\Sigma}}{\mu_\Sigma}.
 \end{align*}
 By requiring 
 \begin{equation}\label{eq:bound_rho}
	\sqrt{\rho} \leq \left\{2 \kappa \left(18 C_3  c_m^2 \frac{  \ell_{\Sigma}^2 }{\mu_{\Sigma}L_\Sigma} + \frac{ \ell_{\Sigma}}{\mu_{\Sigma} L_{\Sigma}} \cdot 4 C_2^2   c_m^2    s \tau_\ell   +   \frac{2 \mu_{\Sigma}}{ \ell_{\Sigma} L_{\Sigma}}     s \tau_\ell \right) \right\}^{-1}
\end{equation}
we have 
\begin{align*}
    \sigma_0 + \sqrt{\rho} \cdot \frac{18 C_3  c_m^2 \frac{  \ell_{\Sigma}^2 }{\mu_{\Sigma}L_\Sigma} + \frac{ \ell_{\Sigma}}{\mu_{\Sigma} L_{\Sigma}} \cdot 4 C_2^2   c_m^2    s \tau_\ell   +   \frac{2 \mu_{\Sigma}}{ \ell_{\Sigma} L_{\Sigma}}     s \tau_\ell  }{1 - 2 C_1 s \tau_g/L_{\Sigma} }  \leq \frac{1 - (2\kappa)^{-1} + C_1 s (\tau_\mu + \tau_g)/L_\Sigma}{1 - 2 C_1 s \tau_g /L_\Sigma}.
\end{align*}
Choosing $C_3 = 4$ we obtain an upper bound of RHS of \eqref{eq:rate_expression} given by~\eqref{eq:def_z_bar}, under network connectivity condition~\label{eq:bound_rho-1}.

Thus far,  we can conclude that (\ref{eq:rate_condition_1})-(\ref{eq:rate_condition_2}) are satisfied by  any $z$ and $\gamma$ such that  $z>\bar{z}_{\text{up}}$, with $\gamma$ and $\bar{z}_{\text{up}}$ defined in~\eqref{eq:gamma_expression_final} and (\ref{eq:def_z_bar}), respectively,
as long as the following conditions hold [note that  conditions  \eqref{eq:cond_1} and (\ref{eq:cond_denom_z_sigma}) are already implied by \eqref{eq:gamma_expression_final} and  \eqref{eq:sample_condition}, respectively]: 
\begin{itemize}
    \item step size lower bound \eqref{eq:cond_2},  with $\epsilon$ chosen according to \eqref{eq:epsilon};
    \item  RSC/RSM tolerance parameters condition \eqref{eq:sample_condition}; 
    \item  and network connectivity condition~\eqref{eq:bound_rho}. 
\end{itemize}

Finally, we  simplify these conditions. First,  \eqref{eq:cond_2},   with $\epsilon$ given by \eqref{eq:epsilon}, is a consequence of   \eqref{eq:sample_condition} and \eqref{eq:gamma_expression_final}. In fact,  
substituting the expression of $\gamma$ and $\epsilon$ into~\eqref{eq:cond_2},  \eqref{eq:cond_2} reads
\begin{align*}
	& L_\Sigma  +  4 \frac{  \ell_{\Sigma}^2 }{\mu_{\Sigma}}   \frac{c_m^2 \sqrt{\rho}}{ (1 - \rho)^4} > C_1 s \tau_g + \frac{\rho}{2} \frac{\mu_\Sigma}{\ell_\Sigma}(\ell_{\Sigma} + s\tau_l) = C_1 s \tau_g + \frac{\rho}{2} \mu_\Sigma + \frac{\rho}{2}  \frac{\mu_\Sigma}{\ell_\Sigma}( s\tau_l)\\
	\Leftrightarrow \qquad & 2L_\Sigma -\rho \mu_\Sigma +  8 \frac{  \ell_{\Sigma}^2 }{\mu_{\Sigma}}   \frac{c_m^2 \sqrt{\rho}}{ (1 - \rho)^4} > 2 C_1 s \tau_g  + \rho \frac{\mu_\Sigma}{\ell_\Sigma}( s\tau_l).
\end{align*}
Since $\rho \leq 1$ and $\mu_\Sigma \leq L_\Sigma$, clearly\eqref{eq:sample_condition} implies~\eqref{eq:cond_2}.  
 Second, we claim that \eqref{eq:sample_condition} is implied by~\eqref{eq:bound_rho} and~\eqref{eq:sample_condition_1}.
Using the facts that $c_m \geq 1$ and $\ell_\Sigma/\mu_\Sigma \geq 1$, \eqref{eq:bound_rho} implies the following:   $\rho \leq 1/4$ and  
\begin{align}\label{eq:whatever}
    2 \sqrt{\rho}   \left( \frac{ \ell_{\Sigma}}{\mu_{\Sigma}^2 } \cdot 4 C_2^2   c_m^2    s \tau_\ell   +   \frac{2 }{ \ell_{\Sigma} }     s \tau_\ell \right)  \leq 1.
\end{align}
Consequently, we can upperbound the second and third  summands on the RHS of~\eqref{eq:sample_condition} as 
\begin{align*}
    & \frac{ \ell_{\Sigma}} {\mu_{\Sigma}}\frac{8 C_2^2   c_m^2 \rho }{(1 - \rho)^2}   s \tau_\ell +  \rho  \cdot \frac{\mu_{\Sigma}}{ \ell_{\Sigma}} \frac{}  {  }    s \tau_\ell \\
    \leq \ & 
    \frac{ \ell_{\Sigma}} {\mu_{\Sigma}}\frac{8 C_2^2   c_m^2 \rho }{9/16}   s \tau_\ell +  \rho  \cdot \frac{\mu_{\Sigma}}{ \ell_{\Sigma}}     s \tau_\ell \\
    \leq \ & \frac{16}{9} \cdot  \rho \left( \frac{\ell_{\Sigma}}{\mu_{\Sigma}} 8 C_2^2   c_m^2    s \tau_\ell +  \frac{ 4 \mu_{\Sigma}}{\ell_{\Sigma}} s \tau_{\ell}\right) \leq \frac{16}{9} \sqrt{\rho} \mu_\Sigma \leq \frac{8}{9} \mu_\Sigma.
\end{align*}
This together with~\eqref{eq:sample_condition_1} implies~\eqref{eq:sample_condition}. $\hfill \square$

\section{Proof of Proposition~\ref{prop:bound-B-STAT}} \label{pf:bound-B-STAT}

      We first bound $B(z)$. Substituting the  $\gamma$ [cf.~\eqref{eq:gamma_expression_final}] and $\epsilon$ [cf.~\eqref{eq:epsilon}] into into~\eqref{eq:BZ}:
\begin{align*}
		B(z) 
		\leq & 
		\frac{ C_2 s \tau_\ell }{L_\Sigma }   
		\frac{ \ell_{\Sigma}}{\mu_{\Sigma}} \frac{2 C_2   c_m^2 }{1 - \rho}  
		\frac{\rho}{z - \rho}  \|\bTheta^{0}-\hTheta\|^2 \notag\\
		& 
		+\frac{ C_2}{L_\Sigma  } m c_m \left( A_3 \cdot  \frac{ \ell_{\Sigma}}{\mu_{\Sigma}} \frac{2 C_2   c_m }{1 - \rho}    + s c_g^2  \frac{ 1}{\mu_{\Sigma}} \frac{2 C_2   c_m }{1 - \rho}  + c_g \nu\right) \frac{\rho}{z - \rho}
		+   \|\bTheta^{1/2} - \hTheta\|^2 \notag\\
	\overset{(\ref{eq:lower-bound-z-step-2})}{\leq}	  & 
		\frac{  s \tau_\ell }{L_\Sigma }   
		\frac{ \ell_{\Sigma}}{\mu_{\Sigma}} \frac{2 C_3   c_m^2 \sqrt{\rho}}{(1 - \rho)^4}  
	 \|\bTheta^{0}-\hTheta\|^2 \notag\\
		& 
		+\frac{ 1}{L_\Sigma  } m c_m \sqrt{\rho} \left( A_3 \cdot  \frac{ \ell_{\Sigma}}{\mu_{\Sigma}} \frac{2 C_2   c_m }{1 - \rho}    + s c_g^2  \frac{ 1}{\mu_{\Sigma}} \frac{2 C_2   c_m }{1 - \rho}  + c_g \nu\right) \frac{ C_3}{(1 - \rho)^3 C_2}
		+   \|\bTheta^{1/2} - \hTheta\|^2.
\end{align*}
Substituting the expression of $A_3$ [cf. Prop.~\ref{prop:track_error}] in the above bound  yields
\begin{align*}
    B(z) \leq &	 \sqrt{\rho}	\frac{  s \tau_\ell }{L_\Sigma }   
		\frac{ \ell_{\Sigma}}{\mu_{\Sigma}} \frac{2 C_3   c_m^2 }{(1 - \rho)^4}  
	 \|\bTheta^{0}-\hTheta\|^2 \notag\\
	 & 
		+ m  \sqrt{\rho}  \frac{ \ell_{\Sigma}}{\mu_{\Sigma}L_\Sigma }   \frac{2  C_3 c_m^2 }{(1 - \rho)^4 } \left( \ell_\Sigma \| \btheta^* \|^2 + \tau_\ell\|\btheta^* \|_1^2 \right) \notag \\
		& 
		+\frac{ 1}{L_\Sigma  } m c_m \sqrt{\rho} \left( s c_g^2  \frac{ 1}{\mu_{\Sigma}} \frac{2 C_2   c_m }{1 - \rho}  + c_g \nu\right) \frac{ C_3}{(1 - \rho)^3 C_2}
		+   \|\bTheta^{1/2} - \hTheta\|^2\\
		\leq &	\sqrt{\rho}	\frac{  s \tau_\ell }{L_\Sigma }   
		\frac{ \ell_{\Sigma}}{\mu_{\Sigma}} \frac{2 C_3   c_m^2 }{(1 - \rho)^4}  
	 \|\bTheta^{0}-\hTheta\|^2 \notag\\
	 & 
		+ m  \sqrt{\rho}  \frac{ \ell_{\Sigma}}{\mu_{\Sigma}L_\Sigma }   \frac{2 C_3 c_m^2  }{(1 - \rho)^4 } \left( \ell_\Sigma \| \btheta^* \|^2 + s\tau_\ell\|\btheta^* \|^2 \right)  \quad \text{($\btheta^*$ is $s$-sparse)} \notag\\
		& 
		+\frac{ 1}{L_\Sigma  } m c_m \sqrt{\rho} \left( s c_g^2  \frac{ 1}{\mu_{\Sigma}} \frac{2 C_2   c_m }{1 - \rho}  + c_g \nu\right) \frac{ C_3}{(1 - \rho)^3 C_2}
		+   \|\bTheta^{1/2} - \hTheta\|^2.
\end{align*}
Since~\eqref{eq:bound_rho_1} is equivalent to
\begin{align*}
    	2\sqrt{\rho} \kappa \left(18 C_3  c_m^2 \frac{  \ell_{\Sigma}^2 }{\mu_{\Sigma}L_\Sigma} + \frac{ \ell_{\Sigma}}{\mu_{\Sigma} L_{\Sigma}} \cdot 4 C_2^2   c_m^2    s \tau_\ell   +   \frac{2 \mu_{\Sigma}}{ \ell_{\Sigma} L_{\Sigma}}     s \tau_\ell \right) \leq 1,
\end{align*}
with $C_3 = 4$, 
together with the fact that $\kappa \geq 1$, we have
\begin{align*}
\sqrt{\rho}  \frac{  \ell_{\Sigma}^2 }{\mu_{\Sigma} L_\Sigma} C_3  c_m^2  \leq \frac{1}{36}
    \quad \text{and} \quad 
     \sqrt{\rho}   \frac{ \ell_{\Sigma}}{\mu_{\Sigma} L_{\Sigma}} \cdot  C_2^2   c_m^2    s \tau_\ell  \leq \frac{1}{8}.
\end{align*}
In addition, from~\eqref{eq:bound_rho_1} we get $\rho \leq 1/4$, and thus we obtain~\eqref{eq:B_final}.

We bound now $\overline{\Delta}_{\rm stat}$.  Using the expression of $\gamma$ [cf.~\eqref{eq:gamma_expression_final}] and $\epsilon$ [cf.~\eqref{eq:epsilon}], we have
\begin{align*}
\begin{split}
        \overline{\Delta}_{\rm stat} & =\frac{ C_2}{\gamma}\frac{\rho \cdot  c_m }{1 - \rho}  \tau_\ell \nu^2 (\epsilon + \epsilon^{-1})  + \frac{C_1 (\tau_\mu + \tau_g )}{\gamma}     \nu^2\\
& \leq \frac{ C_2}{L_\Sigma}\frac{\rho \cdot  c_m }{1 - \rho}  \tau_\ell \nu^2 \left(\frac{\mu_{\Sigma}}{ \ell_{\Sigma}} \frac{1 - \rho}  {2 C_2   c_m } + \frac{ \ell_{\Sigma}}{\mu_{\Sigma}} \frac{2 C_2   c_m }{1 - \rho}  \right)  + \frac{C_1 (\tau_\mu + \tau_g )}{L_\Sigma}     \nu^2\\
& =  \left(\frac{\rho    }{2}   \frac{\mu_{\Sigma}}{ \ell_{\Sigma}L_\Sigma}  +  \rho  \frac{ \ell_{\Sigma}}{\mu_{\Sigma}L_\Sigma} \frac{2 C_2^2   c_m^2 }{(1 - \rho)^2}  \right) \tau_\ell \nu^2 + \frac{C_1 (\tau_\mu + \tau_g )}{L_\Sigma}     \nu^2\\
& \leq \frac{\rho}{2 L_\Sigma}\left( 1  +    \frac{ \ell_{\Sigma}}{\mu_{\Sigma}} \frac{4 C_2^2   c_m^2 }{(1 - \rho)^2}  \right) \tau_\ell \nu^2  + \frac{C_1 (\tau_\mu + \tau_g )}{L_\Sigma}     \nu^2
 \leq {\Delta}_{{\rm stat}}     \quad \text{(given $C_2 \geq 1$).}
\end{split}
\end{align*} $\hfill \square$

  \section{Proof of Lemma \ref{lem:gradient_bound}}  \label{proof_lem:gradient_bound} 
  We first bound $\max_{j \in [m]}\|\bX_j^\top \mathbf{n}_j/n\|_\infty$.  Each column of $\bX_j$ is an $n$-dimensional Gaussian random vector with independent entries. The maximum variance of these entries is bounded by $ \zeta =\max_{j \in [d]} \Sigma_{jj} $. 
Since $\bX_j$ is independent of $\mathbf{n}_j$ and the elements of $\mathbf{n}_j$ are i.i.d. $\sigma^2$-sub-Gaussian,   each element of $\bX_j^\top \mathbf{n}_j$ is the sum of $n$ independent  sub-exponential random variables with sub-exponential norm no larger than $\sqrt{\zeta} \sigma$. Applying the Bernstein's inequality and the union bound we get
\begin{align*}
    \begin{split}
        \mathbb{P} \left( \max_{j \in [m]}\left\|\frac{\bX_j^\top \mathbf{n}_j}{n} \right\|_\infty \geq t\right) \leq  2 \exp \left( - c_3 \min \left\{ \frac{t^2}{\zeta \sigma^2}, \frac{t}{ \sqrt{\zeta} \sigma}\right\} n + \log md\right),
    \end{split}
\end{align*}
for some $c_3 >0$. 
When $\log md \geq (c_3/2 ) \cdot n$, we choose $t = (2/c_3) \sqrt{\zeta} \sigma  \cdot \log md/n \geq \sqrt{\zeta} \sigma$, which leads to 
\begin{align*}
    \begin{split}
        \mathbb{P} \left( \max_{j \in [m]}\left\|\frac{\bX_j^\top \mathbf{n}_j}{n} \right\|_\infty \geq  \sqrt{\zeta} \sigma  \cdot \frac{2}{c_3}\frac{\log md}{n}\right) \leq  2 \exp \left( -  \log md\right).
    \end{split}
\end{align*}
On the other hand,  when $\log md < (c_3/2 ) \cdot n$, we choose $t =  \sqrt{\zeta} \sigma  \cdot \sqrt{(2/c_3)\log md/n}<  \sqrt{\zeta} \sigma$, which leads to 
\begin{align*}
    \begin{split}
    \mathbb{P} \left( \max_{j \in [m]}\left\|\frac{\bX_j^\top \mathbf{n}_j}{n} \right\|_\infty \geq  \sqrt{\zeta} \sigma  \cdot \sqrt{\frac{2}{c_3}\frac{ \log md}{n}}\right) \leq  2 \exp \left( -  \log md\right).
    \end{split}
\end{align*}
Combining the two cases we obtain~\eqref{eq:bound_loc_gradient}.   Applying the same procedure to $\|\nabla \cL (\btheta^*)\|_\infty$, proves~\eqref{eq:bound_g_gradient}. $\hfill \square$

\section{Proof of Proposition \ref{prop:LRSS}}\label{App:proof_prop:LRSS}
		The proof is a  modification of \citep{raskutti2010restricted}. Next, we only highlight the key differences, for completeness. Define the set $$V(r) = \{ r \in \mathbb{R}^d \,|\, \|\Sigma^{\frac{1}{2}} v \|_2 = 1, \|v\|_1 \leq r \},$$ for   fixed radius $r > 0$. Accordingly, define the random variable 
		$$
		M (r, \bX_i) = \sup_{v \in V(r) }\frac{\| \bX_i v\|}{\sqrt{n}}.$$
		
		Following similar steps as in \citep[Lemma 1]{raskutti2010restricted}, and using the Sudakov-Fernique inequality for Gaussian processes, we can bound the expectation of $M (r, \bX_i) $ as
		\begin{align}\label{eq:mean-ub}
			\begin{split}
				\mathbb{E}M (r, \bX_i) \leq 1 + 3 \rho(\Sigma) \sqrt{\frac{\log d}{n}} r  \triangleq t(r),
			\end{split}
		\end{align}
		where $\rho^2(\Sigma) = \max_j \Sigma_{jj}$. 
		Then applying the concentration inequality in the same way as in \citep[Lemma 2]{raskutti2010restricted} we obtain
		\begin{align}\label{eq:M-hpb}
			\mathbb{P} (M (r, \bX_i) > t(r) + \sqrt{m} t(r) ) \leq 2 \exp(- Nt(r)^2/2) \leq 2 \exp(-n ((\sqrt{m} + 1) t(r))^2/8).
		\end{align}
		Finally,   define the event 
		$$\mathcal{T}_i = \{ \exists \ v \in \mathbb{R}^d \text{ s.t.} \| \Sigma^{1/2} v \|_2 = 1 \text{ and } \|\bX_i v\|_2/\sqrt{n}> 2( \sqrt{m} + 1) t(\|v\|_1)\}.$$
		Applying \citep[Lemma 3]{raskutti2010restricted} with $$f(v, \bX_i) = \|\bX_i v\|_2/\sqrt{n}, \quad h(v) = \|v\|_1,\quad  g(r) = ( \sqrt{m} + 1) t(r),\quad  a_n = n/8  \quad \text{and} \quad  \mu = \sqrt{m},$$ we have $$\mathbb{P}(\mathcal{T}_i^c) \geq 1- c_2'\exp(-c_0' N),$$ for some $c_0' \geq1/2$.
		
		Let us use now the union bound: 
		\begin{align*}
			\mathbb{P}(\cup_i \mathcal{T}_i) \leq \sum_{i = 1}^{m} \mathbb{P} (\mathcal{T}_i) \leq c_2' m \exp (-c_0' N) = c_2' \exp (-c_0' N + \log m).
		\end{align*}
		Therefore 
		\begin{align}~\label{eq:RSM_loc_i}
			\frac{\|\bX_i v\|}{\sqrt{n}} \leq 2 (\sqrt{m} + 1) \|\Sigma^{1/2}v\| + 6 (\sqrt{m} + 1) \rho(\Sigma) \sqrt{\frac{\log d}{n}} \|v\|_1, \quad \forall v \in \mathbb{R}^d,\quad \forall i \in [m],
		\end{align}
		with probability larger than $1- c_2'\exp(-c_0' N + \log m)$. Since $N/2  \geq m/2 > \log m$, it holds $1- \exp(-c_0' N + \log m) \geq 1- \exp(-c_0'' N)$, for some $c_0'' >0$. $\hfill \square$
 
 \section{Supplement to the proof of Theorem~\ref{theorem:main}}\label{pf:thm-hp}
 \noindent\textbf{(i)} We show~\eqref{eq:cond_rho_hig-prob-setting}, i.e.,
\begin{align*}
    s \log d/N < c_5 \cdot \frac{\mu_\Sigma}{\zeta},\quad \text{and}\quad \rho < (c_6 m^8 \kappa^4)^{-1}
\end{align*}
implies \eqref{cond_mu_sigma}-\eqref{eq:rho_final_condition}. 
Substituting the expression of $\tau_\mu$ and $\tau_g$ it is  not hard to verify~\eqref{eq:cond_rho_hig-prob-setting} implies ~\eqref{cond_mu_sigma}: $\mu_\Sigma > 36 C_1 s (\tau_\mu + \tau_g)$,
with $c_5$ sufficiently small. Substituting $\tau_\ell$ into~\eqref{eq:rho_final_condition} we can see 
a sufficient condition for~\eqref{eq:rho_final_condition}   is  
  \begin{align}\label{eq:condtion_rho_hp}
     \rho 
      \leq C_6 \left(    m^3 \kappa^2 +  m^4 \kappa^2  \frac{ \zeta}{L_\Sigma}   \frac{s \log d}{N}   \right)^{-2},
  \end{align}
for some $C_6>0$; \eqref{eq:condtion_rho_hp}  is implied by~\eqref{eq:cond_rho_hig-prob-setting} for sufficiently large   $c_6>0$. 


\noindent\textbf{(ii)} We  use \eqref{eq:global_RSC-RSM_high-prob}-\eqref{eq:locsl_RSM_high-prob} to further simplify    the expressions of the  rate expression~\eqref{eq:rate_expression_final}, the   statistical error~\eqref{eq:Delta_stat_final}, and $\B$ in (\ref{eq:B_final}).  

The rate $\rate$ in ~\eqref{eq:rate_expression_final} can be bounded as 
   \begin{align*}
      \rate 
      & \leq \left(1 - (2 \kappa)^{-1} + 2 C_1  c_1 \frac{\zeta}{L_\Sigma} \frac{s \log d}{N} \right)\left(1 - 2 C_1 c_1 \frac{\zeta}{L_\Sigma} \frac{s \log d}{N}\right)^{-1}.
  \end{align*} 
The statistical error can be bounded as \begin{align}\label{eq:stat_err_hp}
    \begin{split}
     \Delta_{\rm stat} & = \frac{\rho}{2 L_\Sigma}\left(   \frac{ \ell_{\Sigma}}{\mu_{\Sigma}} \frac{5 C_2^2   c_m^2 }{(1 - \rho)^2}  \right) \tau_\ell \nu^2  + \frac{C_1 (\tau_\mu + \tau_g )}{L_\Sigma} \nu^2 \\
     & = {\rho}\left(   \frac{ 4}{\mu_{\Sigma}} \frac{5 C_2^2   m^2 }{(1 - \rho)^2}  \right) c_1 \zeta \frac{m^2 \log d}{N} \nu^2 + \frac{C_1 }{L_\Sigma} \left(2 c_1 \zeta  \frac{\log d}{N} \right) \nu^2\\
     & \leq C_5 \left(\frac{m^4 \rho \,\kappa }{(1 - \rho)^2} \frac{\zeta}{L_\Sigma} \frac{\log d}{N}   +  \frac{\zeta}{L_\Sigma} \frac{\log d}{N} \right)\nu^2\\
     & \leq  c_7\left( \frac{\zeta}{L_\Sigma}\frac{s \log d}{N}  \right) \|\htheta - \btheta^*\|^2,
     \end{split}
  \end{align}
for some $C_5, c_7 >0$, where in the last inequality we used $c_6 \rho m^8 \kappa^4 < 1$, $\nu^2=\mathcal{O}(s \|\htheta - \btheta^*\|^2)$ (due to   $
    \| \htheta - \btheta^* \|_1 \leq 2\sqrt{s} \|\htheta - \btheta^*\| 
$  \citep[Lemma 5]{agarwal2012fast}), and  $\rho < (3C_2^2)^{-2}$ implied by~\eqref{eq:rho_final_condition}.

   Finally,  for $B$ in  (\ref{eq:B_final}), we have 
   \begin{align}\label{eq:stat_bound_B}
\begin{split}
        \B  = &  \frac{C_4}{m} \left( \|\bTheta^0 - \hTheta\|^2 + \|\bTheta^*\|^2 + \|\bTheta^{1/2} - \hTheta\|^2 + \frac{m c_m \sqrt{\rho}}{L_\Sigma}\left( \frac{c_m s c_g^2}{\mu_\Sigma} + c_g \nu\right)\right)\\
     \leq & \frac{C_4}{m} \left(\|\bTheta^0 - \hTheta\|^2 + \|\bTheta^*\|^2 + \|\bTheta^{1/2} - \hTheta\|^2 
    \right)\\
    & + C_4 \frac{ 4 \zeta \sigma^2 }{\mu_\Sigma L_\Sigma}
     m s\sqrt{\rho}  \left( \frac{4}{c_3} \frac{m\log md}{N}   + \left( \frac{4}{c_3} \frac{m\log md}{N}  \right)^2 \right)\\
     & + C_4 \frac{ 12\sqrt{\zeta} \sigma }{L_\Sigma} \sqrt{m \rho s}  \left(     \sqrt{\frac{4}{c_3}\frac{ m \log md}{N}}+  
    \frac{4}{c_3}\frac{ m \log md}{N}\right)    \|\htheta - \btheta^*\|,
\end{split}
\end{align}
  where   we used  again $\|\htheta - \btheta^*\|_1 \leq 2 \sqrt{s} \|\htheta - \btheta^*\|$.  $\hfill \square$

  \section{Properties of The $z$-Transform}\label{app:z-transform}
\subsection*{Proof of Lemma~\ref{lem:z_trans}}  
{\bf (i):} For the time-shifted sequence $\{a(t)\}_{t=2}^{K+1}$, it holds
	\begin{align*}
	\sum_{t=1}^K a(t+1) z^{-t} & = z\sum_{t=1}^{K} a(t+1)z^{-t-1}\\
	& = z \left(\sum_{t=1}^K a(t)z^{-t} + a(K+1)z^{-K-1} - a(1)z^{-1}\right)\\
	& = z A^K(z) -  a(1) + a(K+1)z^{-K} \geq z A^K(z) -  a(1).
	\end{align*}

	\noindent {\bf (ii):} For the  sequence $\{\sum_{s=0}^{t-1}\rho^{t-s} a(s)\}_{t=1}^{K}$, we have  $\forall z \in (\rho,1)$:
	\begin{align*}
	\sum_{t=1}^K\sum_{s=0}^{t-1} z^{s-t} \rho^{t-s} a(s) z^{-s}
	& = \sum_{s=0}^{K-1} a(s) z^{-s} \sum_{t=s+1}^K z^{s-t} \rho^{t-s} \leq \frac{\rho}{z - \rho} \cdot \left( A^K (z) + a(0) \right).
	\end{align*}
	
	\noindent{\bf (iii):} Similarly,  the  sequence $\{\sum_{s=0}^{t-1}\rho^{t-s} a(s+1)\}_{t=1}^{K}$, we have
	\begin{align*}
	\sum_{t=1}^K  \sum_{s=0}^{t-1}\rho^{t-s} a(s+1) z^{-t} = \sum_{s=0}^{K-1}  a (s+1) z^{-s -1} \sum_{t=s+1}^K  \rho^{t-s} z^{-t + s + 1} \leq z \cdot \frac{\rho}{z - \rho} A^K(z).
	\end{align*}$\hfill \square$

\subsection*{Proof of Lemma~\ref{lem:linear_rate}} 
	Since $a(t) \geq 0$, \eqref{eq:condition} implies that for all $K \geq 1$ and $z \in (\bar{z},1)$, 
	$$
	\label{eq:maynothold}
	a(K) z^{-K} \leq B + c \cdot \sum_{t=1}^K z^{-t}.$$	Therefore, for all $z \in (\bar{z},1),$
	 $$a(K)  \leq B \cdot z^{K} + c \cdot \sum_{t=1}^K z^{K-t} \leq  B \cdot z^{K} +  \frac{c}{1-z}.$$ 
	Letting $z \to \bar{z}^+$ concludes the proof. $\hfill \square$

\vskip 0.2in
\bibliographystyle{plain}
\bibliography{note.bib,Network_learning.bib}

\end{document}